\PassOptionsToPackage{table,xcdraw,usenames,dvipsnames}{xcolor}
\documentclass[9.5pt]{fairmeta}

\usepackage{url}
\usepackage{amsmath,amssymb,amsthm,mathtools,mathrsfs,bbm}
\usepackage{booktabs,multirow,tabularx,longtable,makecell,colortbl,arydshln}
\usepackage{graphicx,float}
\usepackage{wrapfig}
\usepackage{subfigure}
\usepackage{enumitem}
\usepackage[ruled,vlined]{algorithm2e}
\usepackage{pifont}
\usepackage{listings}
\usepackage{xspace}

\definecolor{dynamicrow}{HTML}{DCEFE3}
\definecolor{gaincolor}{HTML}{22863A}
\definecolor{losscolor}{HTML}{CB2431}
\definecolor{neutralcolor}{HTML}{6A737D}
\definecolor{promptgreen}{HTML}{2E7D57}
\definecolor{promptgreenbg}{HTML}{F1F8F4}
\definecolor{promptolive}{HTML}{577A35}
\definecolor{promptolivebg}{HTML}{F5F8EE}
\definecolor{promptteal}{HTML}{28736F}
\definecolor{prompttealbg}{HTML}{EFF8F7}
\definecolor{promptmint}{HTML}{3C8A64}
\definecolor{promptmintbg}{HTML}{F0FAF5}
\definecolor{promptforest}{HTML}{1F6B47}
\definecolor{promptforestbg}{HTML}{EEF7F1}
\definecolor{skillgreen}{HTML}{2D6A4F}
\definecolor{skillgreenbg}{HTML}{F0F8F3}
\newcommand{\gain}[1]{{\scriptsize\textcolor{gaincolor}{#1}}}

\newcommand{\nodiff}[1]{{\scriptsize\textcolor{neutralcolor}{#1}}}
\newcommand{\deltagain}[1]{\gain{#1}}
\newcommand{\deltaloss}[1]{\nodiff{#1}}
\newcommand{\deltasame}[1]{\nodiff{#1}}
\newcommand{\improvedcell}[1]{\cellcolor{dynamicrow}#1}

\newcommand{\ourmethod}{{\fontfamily{lmtt}\selectfont \textbf{Agentic ESOpt}}\xspace}
\newcommand{\llmname}[1]{{\fontfamily{pcr}\selectfont {#1}}\xspace}

\definecolor{bittersweet}{rgb}{1.0, 0.44, 0.37}
\definecolor{mygreen}{rgb}{0.29, 0.7, 0.48}

\definecolor{demphcolor}{RGB}{144,144,144}

\definecolor{mygray}{gray}{0.4}
\definecolor{autopurple}{HTML}{7030A0}
\newcounter{takeaway}
\newcommand{\takeaway}[1]{%
    \refstepcounter{takeaway}%
    \begin{tcolorbox}[enhanced, colframe=autopurple, colback=autopurple!4, boxrule=0.65pt, arc=1mm,
    title={Takeaway \thetakeaway}, fonttitle=\bfseries\scriptsize,
    attach boxed title to top left={xshift=5pt,yshift=-2pt},
    boxed title style={colframe=autopurple, colback=autopurple, boxrule=0.4pt, arc=0.8mm,
    left=3pt, right=3pt, top=1pt, bottom=1pt},
    top=7pt, bottom=3pt, left=4pt, right=4pt, boxsep=1pt]
        #1
    \end{tcolorbox}
}

\definecolor{dyna_yellow}{HTML}{BF9000}
\definecolor{adaptive_blue}{HTML}{0070C0}
\definecolor{darksalmon}{rgb}{0.91, 0.59, 0.48}
\definecolor{emerald}{rgb}{0.31, 0.78, 0.47}
\definecolor{green(pigment)}{rgb}{0.0, 0.65, 0.31}
\definecolor{amaranth}{rgb}{0.9, 0.17, 0.31}
\definecolor{iris}{rgb}{0.35, 0.31, 0.81}
\definecolor{uu}{rgb}{0.95, 0.51, 0.51}
\definecolor{spirodiscoball}{rgb}{0.06, 0.75, 0.99}
\definecolor{mygrey}{gray}{0.4}
\definecolor{QuestionColor}{rgb}{0.7, 0.1, 0.1}
\definecolor{AnswerColor}{rgb}{0.1, 0.5, 0.1}
\definecolor{ReasoningColor}{rgb}{0.1, 0.1, 0.7}

\SetKwInOut{Input}{Input}\SetKwInOut{Output}{Output}
\newtheorem{lemma}{Lemma}

\hypersetup{
    colorlinks=true,
    linkcolor=metaaccent,
    citecolor=metaaccent,
    filecolor=magenta,
    urlcolor=metaaccent,
}

\titlehead{\textcolor{nusblue}{NUS}\hspace{0.4em}\textcolor{nusorange}{|}\hspace{0.4em}National University of Singapore}

\title{\ourmethod: Fine-Tuning Long-Horizon LLM Agents \\with Minimal GPU Requirements}

\author[1]{Zhi Zheng}
\author[2]{Rongsheng Chen}
\author[2]{Yunpeng Ba}
\author[2]{Zhenkun Wang}
\author[3]{Yee Whye Teh}
\author[1]{Wee Sun Lee}

\affiliation[1]{National University of Singapore}
\affiliation[2]{Southern University of Science and Technology}
\affiliation[3]{Oxford}

\metadata[GitHub \protect\githubemoji]{\url{https://github.com/zz1358m/Agentic-ESOpt}}
\correspondence{\email{zhi.zheng@u.nus.edu}}

\abstract{

Reinforcement Learning (RL) has been promising in single-turn LLM fine-tuning. However, long-horizon agentic reasoning introduces increasingly branching interactions and sparse rewards, exposing several limitations of RL: its heavyweight backpropagation makes it impractical to fine-tune larger LLMs, and longer-horizon trajectories make the credit assignment substantially harder. This paper argues that evolution strategies (ES) can be a better choice for fine-tuning long-horizon LLM agents. Compared with agentic RL, ES offers three key advantages: \textbf{1) Model Scalability:} ES enables full-parameter optimization requiring only minimal, inference-level GPU memory, making it possible to fine-tune large LLMs. \textbf{2) Flexibility:} its lightweight, black-box feedback interface makes ES fine-tuning easy to compose with prompt-space evolution (e.g., skill optimization \& test-time compute); and \textbf{3) Long-Horizon Scalability:} ES performs trajectory-level parameter attribution without decomposing rewards across horizons, yielding better scalability than Agentic RL as the horizon length grows.

Based on this insight, we propose \ourmethod, a full-parameter agentic fine-tuning framework tailored to flexible parameter--context co-evolution. At each step, Agentic ESOpt samples perturbations around the current LLM parameters, evaluates the resulting agents with rewards, and applies an online reward-weighted update. To improve the exploration--adaptation trade-off, Agentic ESOpt further introduces a cosine decay schedule of the perturbation scale $\sigma$. We evaluate Agentic ESOpt across both train-time fine-tuning and agentic test-time compute settings. On long-horizon Sudoku, Agentic ESOpt outperforms RL methods by 12.50\% with Qwen3.5-4B. On WebArena-Lite, full-parameter optimization of \textbf{\textit{Qwen3.5-27B}} improves the No Skill baseline by 6.69\%, and combining Agentic ESOpt with Trace2Skill further improves the Trace2Skill baseline by 2.42\%. In test-time automatic heuristic design, Agentic ESOpt performs online prompt--parameter co-evolution, improving its matched baseline in 28 of 36 settings.
}
\begin{document}

\maketitle

\begin{figure}[htbp]
\centering\vspace{-12pt}
\subfigure[Challenges of Agentic Reasoning]{
  \includegraphics[width=0.29\linewidth]{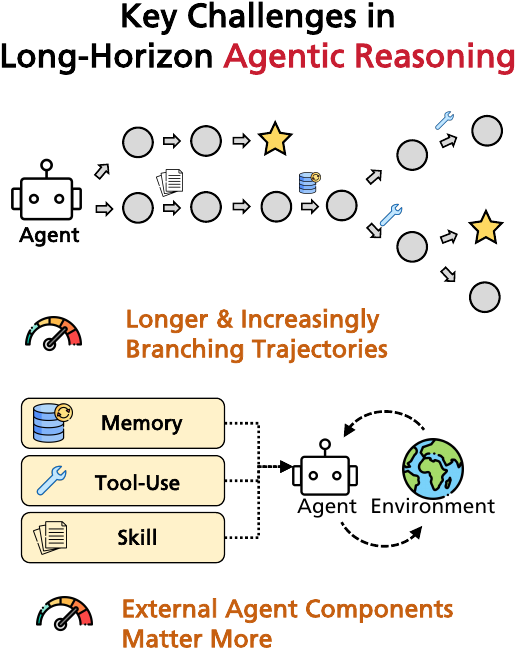}
}
\subfigure[Bottlenecks of Agentic RL]{
  \includegraphics[width=0.293\linewidth]{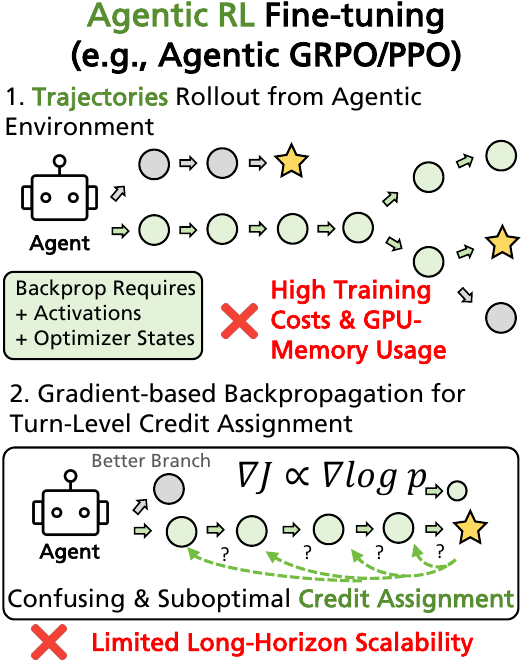}
}
\subfigure[Advantages of Agentic ESOpt \textbf{(Ours)}]{
  \includegraphics[width=0.323\linewidth]{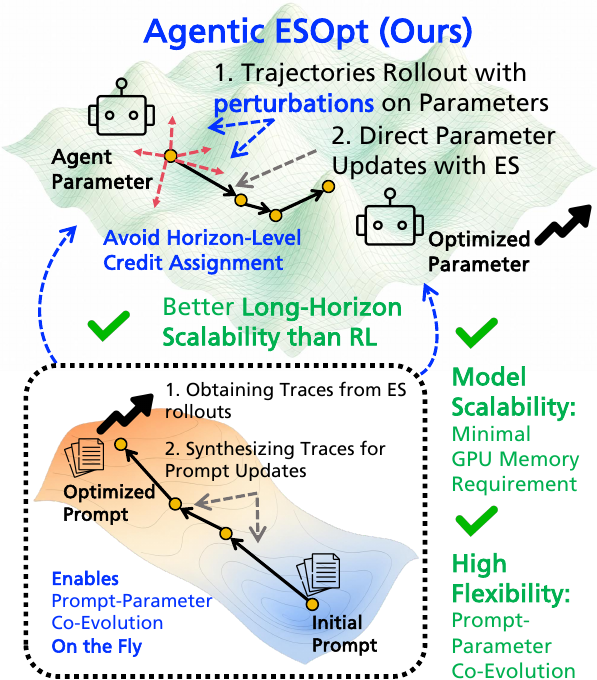}
}\vspace{-5pt}
\caption{(a) Long-horizon agentic reasoning introduces new challenges for Agentic RL (b), including high GPU memory requirements in training and difficult credit assignment across horizons. Agentic ESOpt (c) addresses these issues through inference-level GPU memory, flexible black-box feedback, and better long-horizon scalability.}\vspace{-12pt}
\label{fig:overview}
\end{figure}

\section{Introduction}
Advanced Large Language Models (LLMs) such as Qwen3, DeepSeek-R1, and Gemini 2.5 have demonstrated strong capabilities as general-purpose agents \citep{yang2025qwen3,guo2025deepseek,geminiteam2025gemini25}. With their capabilities in tool use, long-context processing, and multimodal interaction, these models can navigate websites \citep{zhou2024webarena}, edit repositories \citep{yang2024sweagent}, and coordinate multi-step software workflows \citep{wang2024openhands}. However, general-purpose agents can still perform poorly on uncommon tool APIs \citep{ma2024spreadsheetbench} and specialized scientific or algorithmic tasks \citep{liu2024evolution}. Therefore, efficiently fine-tuning advanced LLM agents to adapt task-specific expertise remains important \citep{du2026survey}.

Reinforcement learning (RL) has demonstrated remarkable effectiveness in single-turn LLM fine-tuning \citep{shao2024deepseekmath,liu2025understanding,zheng2025soft,tajwar2026maximum}. However, in long-horizon agentic reasoning, which introduces increasingly branching interactions and provides only sparse feedback, several limitations of Agentic RL are exposed. As illustrated in Figure \ref{fig:overview},  first, Agentic RL requires storing heavyweight activations, optimizer states, and performing backpropagation through trajectories, making full-parameter fine-tuning increasingly impractical for larger LLMs. Moreover, as trajectories become longer and more branching, assigning sparse trajectory-level rewards back to individual decisions becomes substantially harder \citep{kim2026longhorizon}.

This paper argues that \textbf{evolution strategies (ES) \citep{salimans2017evolution} can be a better choice for fine-tuning long-horizon LLM agents}. Instead of performing backpropagation, ES samples perturbations around the current LLM parameters, evaluates the perturbed agents with environment rewards, and applies a reward-weighted parameter update. Compared with Agentic RL, ES offers three key advantages:
\begin{itemize}[leftmargin=*]
    \item \textbf{1) Model Scalability:} ES enables full-parameter optimization with only inference-level GPU memory, which is the minimal amount, substantially reducing the memory barrier to fine-tuning larger LLM agents.
    \item \textbf{2) Flexibility:} Its lightweight, black-box feedback interface makes ES fine-tuning easy to compose with skill--space evolution \citep{ni2026trace2skill} and test-time compute \citep{liu2024evolution}.
    \item \textbf{3) Long-Horizon Scalability:} Unlike RL estimators that usually lead to poor long-horizon credit assignment, ES performs trajectory-level parameter attribution without decomposing rewards across turns, yielding better scalability than Agentic RL as the horizon length grows.
\end{itemize}

Recent work has explored ES for single-turn LLM reasoning \citep{qiu2026esscale,sun2026essam,sarkar2025evolution}, where ES achieves higher GPU memory efficiency but slightly lower performance than RL methods. However, we argue that the structural advantages of ES are \textbf{particularly pronounced in fine-tuning long-horizon agents}, where ES can be significantly preferable to RL, rather than merely a cheaper alternative. To implement this advantage, we propose \ourmethod, a full-parameter ES framework for both train-time agent fine-tuning and agentic test-time compute. At each generation, Agentic ESOpt samples full-parameter perturbations, evaluates the resulting agents with environment rewards, and applies an online reward-weighted update. Its lightweight black-box update enables on-the-fly parameter adaptation within prompt-space optimization loops, allowing parameter updates to complement both skill-space optimization and test-time search. To improve the exploration--adaptation trade-off, Agentic ESOpt further introduces a cosine decay mechanism for the perturbation scale $\sigma$, which retains a nonzero terminal $\sigma_T$ for mild smoothing regularization in train-time optimization, and decays $\sigma_T$ to 0 for progressively finer adaptation in test-time optimization.

We evaluate \ourmethod for agentic fine-tuning in both train-time and test-time compute with LLMs from 4B to 27B. In the \textbf{train-time fine-tuning}, we study long-horizon reasoning, ReAct-style tool use, and web agents. On Sudoku, RL- and ES-based methods with matched FLOPs are competitive at smaller minimum successful horizons (5- and 10-horizon), but their relative ordering changes as the horizon grows: at 15 turns, Agentic ESOpt reaches +12.5\% compared to the strongest GRPO baseline. Across ReAct-style Math and DocVQA, Agentic ESOpt achieves an average improvement of 13.7\% over the \textit{Qwen3.5-4B} base model and 8.3\% over Agentic GRPO. On WebArena-Lite, full-parameter optimization of \textit{Qwen3.5-27B} improves the No Skill baseline from 29.47\% to 36.16\%, while combining Agentic ESOpt with Trace2Skill raises it from 33.94\% to 36.36\%. In the \textbf{test-time} agentic heuristic design, Agentic ESOpt improves the matched baselines in 28 of 36 comparisons. Moreover, a preliminary population-sensitivity study further suggests that stronger LLM backbones can obtain useful ES updates with smaller populations. Our contributions are summarized as follows:
\begin{itemize}[leftmargin=*]
    \item We identify that, on long-horizon agentic reasoning, ES becomes preferable to Agentic RL. We attribute this shift to three key properties of ES: \textbf{model scalability} through inference-level GPU memory, \textbf{flexibility} through black-box trajectory feedback, and \textbf{long-horizon scalability} through trajectory-level parameter attribution.


    \item We introduce \ourmethod, a backpropagation-free ES framework require only minimal GPU memory. It supports flexible agentic fine-tuning in both \textbf{train-time adaptation} and \textbf{test-time compute}. It allows parameter optimization to compose with prompt-space evolution, while a cosine schedule of perturbation improves the exploration--adaptation trade-off.

    \item We validate \ourmethod across long-horizon Sudoku, ReAct-style tool use, web agents, and automatic heuristic design with models from 4B to 27B. Agentic ESOpt outperforms RL on long-horizon Sudoku and ReAct-style Math/DocVQA, enables full-parameter adaptation of a 27B WebArena agent, improves Trace2Skill, and enhances existing test-time evolutionary search in 28 of 36 settings.
\end{itemize}

\section{Preliminaries: Agentic LLM Reasoning}

We define a multi-turn LLM agent that repeatedly observes the environment and produces an action as follows:
\[
    a_t \sim \pi_{\theta}(a_t \mid \boldsymbol{o}_{\leq t}, c_t),
\]
where $\theta$ denotes model parameters, $\boldsymbol{o}_{\leq t}$ is the interaction history, and $c_t$ is an external prompt, memory, skill, or tool instruction. An episode induces a multi-turn trajectory $\boldsymbol{\tau}=(o_1,a_1,\ldots,o_H,a_H)$, where the \textbf{horizon} $H$ is the number of agent--environment turns before termination. Its return is $R(\boldsymbol{\tau})=\sum_{t=1}^{H}\gamma^{t-1}r_t$. In many agentic reasoning tasks, rewards are sparse: intermediate rewards are typically zero, $r_t=0$ for $t<H$, and only the completed trajectory receives a task score. In some problems (e.g., MATH \& DocVQA ReAct-style Tool Usage), the agent receives the complete multi-turn trajectory $\boldsymbol{\tau}$ as input, but in some problems with (partial) Markov properties (Sudoku, WebArena), the agent only needs to receive partial local input (e.g., $o_t$ only).

The optimization of the Task-specific agent can act on $c_t$ or $\theta$. Prompt-space optimization methods optimize $c_t$ while freezing the LLM. They are lightweight, but can only elicit behaviors already accessible to the frozen policy \citep{yang2026skillopt,ni2026trace2skill,liu2024evolution,zheng2025mctsahd}. Policy fine-tuning methods optimize $\theta$ for better agentic ability with Supervised Fine-Tuning (SFT), On-Policy Distillation (OPD), Group Relative Policy Optimization (GRPO), or Proximal Policy Optimization (PPO).

\paragraph{Agentic SFT \& OPD.}
Agentic SFT learns from expert actions \citep{chu2025sft}, while Agentic OPD learns from token distributions of higher-ability LLMs \citep{song2026survey}. Both methods require labels beyond the scalar environment reward, which requires additional cost, so we exclude them from our comparison scope.

\paragraph{Agentic PPO \& GRPO.}
Agentic PPO and Agentic GRPO require a scalar environment reward, making them useful in many application scenarios \citep{qi2025webrl}. Agentic GRPO \citep{shao2024deepseekmath} samples $G$ trajectories for the same task and assigns trajectory $i$ the group-relative advantage as follows:
\[
    \widehat A_i=\frac{R_i-\operatorname{mean}_{j}(R_j)}{\operatorname{std}_{j}(R_j)+\varepsilon}.
\]
Let $q_{i,t}(\theta)=\pi_\theta(a_{i,t}\mid h_{i,t})/\pi_{\theta_{\mathrm{old}}}(a_{i,t}\mid h_{i,t})$, where $h_{i,t}=(\boldsymbol{o}_{\leq t},c_t)$. Agentic GRPO minimizes the loss as follows:
\begin{equation}
    \mathcal L_{\mathrm{GRPO}}
    =-\frac{1}{G}\sum_{i=1}^{G}\sum_{t=1}^{H_i}
    \min\!\left(q_{i,t}\widehat A_i,
    \operatorname{clip}(q_{i,t},1-\epsilon,1+\epsilon)\widehat A_i\right)
    +\beta D_{\mathrm{KL}}(\pi_\theta\|\pi_{\mathrm{ref}}),   \label{grpo}
\end{equation}
However, recent studies \citep{feng2026group,he2026hierarchy,zhang2026reasoning} point out that although GRPO's group-relative advantage is representative in single-turn scenarios, it cannot cover multi-turn scenarios. To achieve better credit assignment, PPO has recently been adopted in multi-turn agentic RL \citep{li2026turn,hou2026single}, which uses the same clipped surrogate but replaces $\widehat A_i$ with a turn-level advantage estimated by a critic LLM \citep{schulman2017proximal}:
\begin{equation}
    \mathcal L_{\mathrm{PPO}}
    =-\sum_{t=1}^{H}\min\!\left(q_t\widehat A_t,
    \operatorname{clip}(q_t,1-\epsilon,1+\epsilon)\widehat A_t\right),
    \quad
    \widehat A_t=\sum_{l=0}^{H-t}(\gamma\lambda)^l
    \bigl(r_{t+l}+\gamma V_\phi(h_{t+l+1})-V_\phi(h_{t+l})\bigr).\label{ppo}
\end{equation}
We argue that PPO cannot fully remove the long-horizon difficulty. It requires a critic warm-up phase to learn meaningful values, while sparse terminal rewards make its early advantages unreliable. Even with a well-trained critic, the policy gradient still sums $H$ action-level score terms, so its variance remains dependent on the horizon length. These limitations motivate the trajectory-level estimator introduced next and analyzed in Appendix \ref{app:horizon-variance}.  These drawbacks in credit assignment motivate the parameter-space ES formulation of \ourmethod.

\section{Methodology: Agentic ESOpt}

As illustrated in Figure \ref{fig:main-process}, \ourmethod performs full-parameter ES optimization by sampling parameter perturbations around the current LLM, evaluating the perturbed agents with scalar environment rewards, and applying a reward-weighted parameter update. This update is forward-only, which only requires storing the noise seed and using in-place addition and subtraction \citep{qiu2026esscale}, so the GPU memory requirement of Agentic ESOpt remains minimal, the same as the inference requirements. Moreover, the same black-box trajectory feedback can be reused by skill-space optimizers or test-time compute.

\begin{figure}[t]
\centering \vspace{-10pt}
\includegraphics[width=\linewidth]{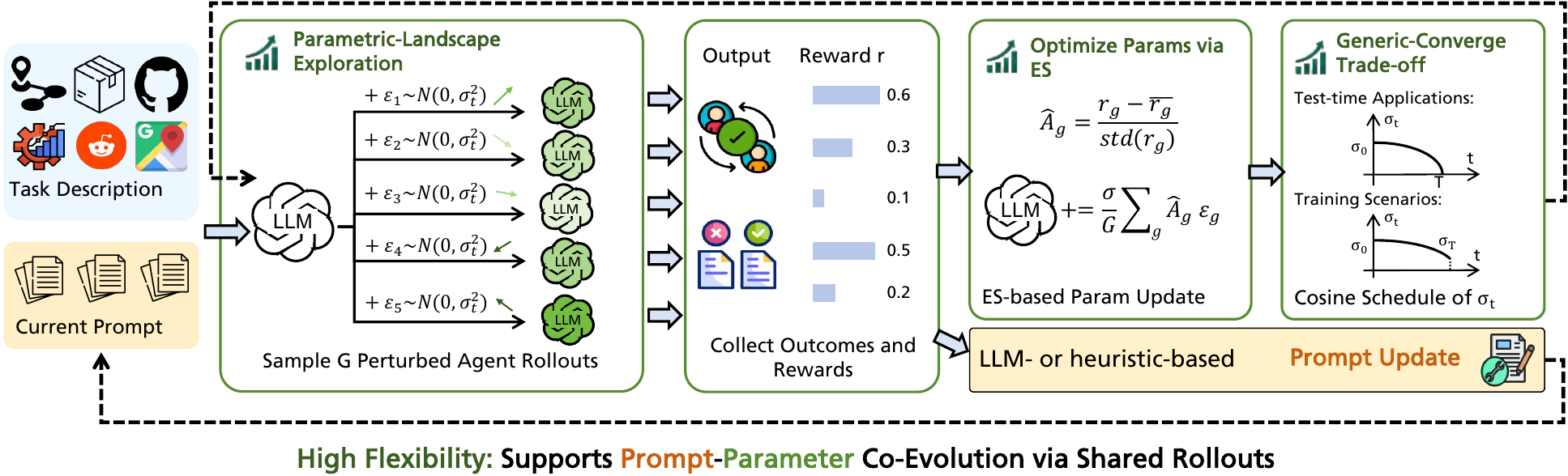}
\caption{\textbf{Detailed workflow of Agentic ESOpt.} Starting from the current LLM, Agentic ESOpt samples parameter perturbations, evaluates the perturbed agents in the environment, normalizes their scalar rewards, and applies a reward-weighted ES update. Compared with Agentic RL, Agentic ESOpt provides \textbf{model scalability}, \textbf{optimization flexibility}, and \textbf{long-horizon scalability}. Its lightweight black-box interface also allows easy composition with prompt-space optimization methods such as Trace2Skill (LLM-based) and EoH (heuristic-based), enabling on-the-fly parameter adaptation within existing test-time compute procedures.}
\label{fig:main-process}\vspace{-10pt}
\end{figure}

Formally, let $\boldsymbol{\tau}=(o_0,a_0,\ldots,o_H,a_H)$ denote an interaction trajectory induced by policy $\pi_\theta$, and let $R(\boldsymbol{\tau})$ denote its scalar trajectory return. For a fixed external agent state $c$, the objective is as follows:
\[
J(\theta;c)
=
\mathbb{E}_{\boldsymbol{\tau}\sim\pi_\theta(\cdot\mid c)}
\left[R(\boldsymbol{\tau})\right].
\]

Agentic ESOpt optimizes the objective by searching over the parameter space around $\theta$, where $\boldsymbol{\epsilon}\in\mathbb{R}^{d}$ is a full-parameter perturbation on the $d$-dimensional parameters, and the Gaussian-smoothed objective is as:
\begin{equation}
J_\sigma(\theta;c)
=
\mathbb{E}_{\boldsymbol{\epsilon}\sim\mathcal{N}(0,I)}
\left[
J(\theta+\sigma\boldsymbol{\epsilon};c)
\right].\label{smoothed-objective}
\end{equation}
Then we can derive the ES pseudo-gradient as follows \citep{salimans2017evolution,qiu2026esscale}:
\begin{equation}
\nabla_\theta J_\sigma(\theta;c)
=
\frac{1}{\sigma}
\mathbb{E}_{\boldsymbol{\epsilon}}
\left[
J(\theta+\sigma\boldsymbol{\epsilon};c)\boldsymbol{\epsilon}
\right].\label{gradient}
\end{equation}
The derivation is detailed in Appendix \ref{app:es-scalar-gradient}. The ES gradients are estimated from scalar scores without differentiating through the agent--environment interaction.

In implementation, to estimate Eq.\eqref{gradient}, Agentic ESOpt samples $G$ perturbations $\boldsymbol{\epsilon}_1,\ldots,\boldsymbol{\epsilon}_G$, evaluates the corresponding perturbed agents, and obtains rewards $R_i = R(\boldsymbol{\tau}_i)$. To reduce variance, we normalize the rewards within the population with a z-score as follows:
\[
\hat{R}_i
=
\frac{R_i-\mu_R}{s_R+\varepsilon},
\quad \text{where}\quad
\mu_R=\frac{1}{G}\sum_{j=1}^G R_j,
\quad
s_R^2=\frac{1}{G}\sum_{j=1}^G(R_j-\mu_R)^2.
\]
In practice, unlike the canonical ES estimator in Eq.~\eqref{gradient}, our normalized update omits the explicit $1/\sigma$ factor, with $\alpha$ serving as the effective update scale. The implemented update is therefore:
\[
\theta_{t+1}
=
\theta_t
+
\frac{\alpha}{G}
\sum_{i=1}^{G}\hat{R}_i\boldsymbol{\epsilon}_i.
\]
We follow \citet{qiu2026esscale} in storing only the noise seed of each perturbation and using in-place addition and subtraction for each perturbed agent, so Agentic ESOpt requires only the same amount of GPU memory for inference.

\paragraph{Prompt-Space Composition and Prompt-Parameter Co-Evolution.}
Test-time compute and prompt-space optimization methods typically keep the LLM parameters fixed throughout the search process. As a result, the search can only reweight or recombine behaviors already accessible to the frozen policy, which may limit optimization when solving the task requires changing the underlying policy itself. In contrast, the lightweight black-box updates of \ourmethod allow parameter adaptation to be performed on the fly alongside prompt-space search, \textbf{enabling prompt--parameter co-evolution}.

Let $\mathcal{D}_t$ denote the trajectories and scores collected at iteration $t$, $\mathcal{U}_{\mathrm{ES}}$ the Agentic ESOpt parameter update, and $\mathcal{U}_c$ an external update rule for prompt $c_t$. A general alternating outer loop can then update the two spaces as
\[
\theta_{t+1}
=
\mathcal{U}_{\mathrm{ES}}(\theta_t;c_t,\mathcal{D}_t),
\quad
c_{t+1}
=
\mathcal{U}_c(c_t;\mathcal{D}_t).
\]

\subsection{Cosine Decay of the Perturbation Radius $\sigma$}

With perturbations $\boldsymbol{\epsilon}\sim\mathcal{N}(0,I)$, Agentic ESOpt effectively optimizes the Gaussian-smoothed objective in Eq \eqref{smoothed-objective}. However, a nonzero perturbation radius introduces a smoothing bias, whose leading term is characterized as:

\begin{lemma}[Gaussian-smoothing bias]
\label{lem:es-bias}
Assume $J(\theta;c)$ is sufficiently smooth in a neighborhood of $\theta$. Then
\[
J_\sigma(\theta;c)
=
J(\theta;c)
+
\frac{\sigma^2}{2}
\mathrm{Tr}\!\left(
\nabla_\theta^2 J(\theta;c)
\right)
+
O(\sigma^4).
\]
\end{lemma}
The proof is in Appendix \ref{app:es-smoothing-bias}. The second-order term
$\mathrm{Tr}(\nabla_\theta^2 J)$
can be viewed as a \textbf{regularization term}. For a maximization objective, it penalizes sharp local optima while favoring flatter parameter neighborhoods. Therefore, a larger $\sigma$ introduces stronger regularization but also a larger bias from the original objective.

Existing ES methods for single-turn LLM fine-tuning typically use a fixed perturbation radius throughout optimization \citep{qiu2026esscale}, without explicitly adapting this trade-off. Agentic ESOpt instead gradually decreases $\sigma$ over $T$ update steps, using a larger radius early for broader exploration and stronger regularization, and a smaller radius later for increased exploitation and reduced objective bias as follows:
\[
\sigma_t
=
\sigma_T
+
(\sigma_0-\sigma_T)
\frac{1+\cos(\pi t/T)}{2},
\qquad
t=0,\ldots,T.
\]
For \textbf{train-time Agentic ESOpt}, we retain a nonzero $\sigma_T$ to balance exploitation with exploration and regularization. In contrast, \textbf{test-time compute} focuses on the unbiased outcome of the current task instead of the generalization of the agent, so we decay $\sigma_T$ to zero to minimize the objective bias toward the end of optimization.

\section{Agentic Sudoku: Controlled Experiments on Long-Horizon Scalability}
\label{sec:sudoku-agentic}

To demonstrate the advantage of Agentic ESOpt on \textbf{long-horizon scalability} compared to Agentic RL, this section first analyzes their characteristics theoretically. Then we evaluate the proposed Agentic ESOpt in a controlled multi-turn Sudoku environment, where the minimal possible horizon length $H$ is strictly defined by the number of masks on the Sudoku board.

\paragraph{Theoretical Reason for Scalability.}
Consider a trajectory of $H$ actions, $\boldsymbol{a}=(a_1,\ldots,a_H)$, with terminal return $R(\boldsymbol{a})$. A simple Agentic RL estimator (e.g., Eq \eqref{grpo}, Eq \eqref{ppo}) with baseline $b$ has the form as follows:
\[
    \widehat{g}_{\mathrm{PG}}
    =
    (R(\boldsymbol{a})-b)
    \sum_{t=1}^{H}
    \nabla_\theta \log \pi_\theta(a_t\mid s_t).
\]
Following the analysis of \citet[Sec.~3.1]{salimans2017evolution}, suppose the return has weak correlation with any individual action, the per-step score terms are approximately uncorrelated, and ES and policy gradients induce comparable return variation. The estimator variance will then grow nearly linearly with the horizon as follows:
\[
    \operatorname{Var}[\widehat{g}_{\mathrm{PG}}]
    \approx
    \operatorname{Var}[R(\boldsymbol{a})]
    \operatorname{Var}\!\left[
        \sum_{t=1}^{H}\nabla_\theta\log\pi_\theta(a_t\mid s_t)
    \right]
    \propto H.
\]
Agentic ESOpt samples one parameter perturbation $\boldsymbol{\epsilon}$ for the complete rollout. Its corresponding estimator is as follows:
\[
    \widehat{g}_{\mathrm{ES}}
    =
    (R(a(\theta+\sigma\boldsymbol{\epsilon}))-b)
    \frac{\boldsymbol{\epsilon}}{\sigma},
    \quad
    \operatorname{Var}[\widehat{g}_{\mathrm{ES}}]
    \approx
    \operatorname{Var}[R(\boldsymbol{a})]
    \operatorname{Var}\!\left[\frac{\boldsymbol{\epsilon}}{\sigma}\right].
\]
So, the parameter-score term $\boldsymbol{\epsilon}/\sigma$ does not sum over $H$. Agentic ESOpt therefore assigns the terminal return directly to one coherent policy variation, without asking the same scalar outcome to distinguish among $H$ horizons. This scaling argument predicts a \textbf{relative} advantage for Agentic ESOpt as the effective horizon grows.

Importantly, this comparison isolates the horizon-dependent structure of the two estimators. Agentic RL accumulates action-score terms across turns, whereas Agentic ESOpt avoids this with direct parameter-space search whose variance does not explicitly grow with $H$. This predicts an increasing relative advantage for parameter-space attribution as the effective horizon grows. A complete covariance expansion and discussion of other sources of estimator difficulty are provided in Appendix \ref{app:horizon-variance}. Figure~\ref{fig:sudoku-landscapes} further visualizes local parameter-space neighborhoods as $H^*$ increases. Although the reward contrast decreases for harder settings, the ES parameter score itself does not introduce an additional sum over turns.

\textbf{Controlled Long-Horizon Sudoku.} To empirically test the predicted long-horizon scaling behavior, we design a multi-turn Sudoku environment with a controllable \textbf{minimum successful horizon}. The environment provides \textbf{only a terminal reward}, and each valid action fills at most one cell. We define the shortest successful horizon for task $x$ as $H^*(x)=\min_{\boldsymbol{\tau}:R(\boldsymbol{\tau})=1}|\boldsymbol{\tau}|$. Masking 5, 10, or 15 cells therefore gives $H^*\in\{5,10,15\}$. The realized horizon $H=|\boldsymbol{\tau}|$ may exceed $H^*$ because of invalid or unproductive actions. Our theoretical analysis concerns the realized horizon $H$, while the experiments are grouped by $H^*$.

\textbf{Experimental Setup.} Sudoku experiments are conducted on 4$\times$NVIDIA H100 80GB GPUs. We compare Agentic ESOpt with two 8-rollout Agentic GRPO runs with different sampling configurations, Agentic PPO \citep{li2026turn}, and Vanilla Agentic ES on Qwen3.5-4B. We create a 32-instance training dataset and a 32-instance evaluation dataset for each $H^*$. Following the recommendation in the model card, the temperature is 0.7, top-p is 0.8, and top-k is 20 across evaluations. We use task success rate as the primary metric, where an episode is successful only if the Sudoku is fully solved within the interaction budget, and report the mean and standard deviation over three evaluation runs. Detailed training and decoding configurations are in the Appendix \ref{app:sudoku-diagnostics}.

\begin{figure}[H]
\centering\vspace{-5pt}
\subfigure[Final evaluation success]{
  \includegraphics[width=0.34\linewidth]{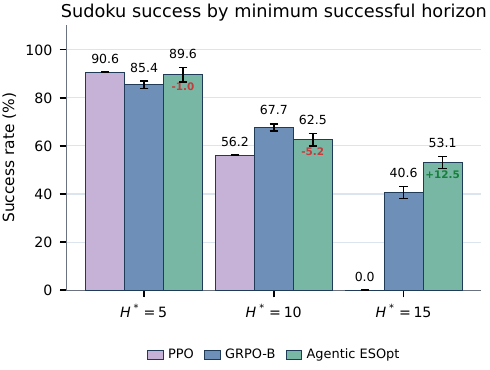}
}%
\subfigure[$H^*=5$ evaluation curve]{
  \includegraphics[width=0.31\linewidth]{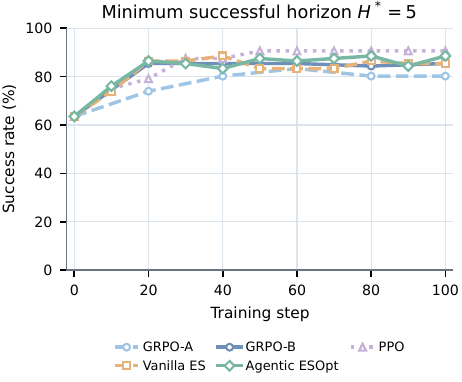}
}%
\subfigure[$H^*=15$ evaluation curve]{
  \includegraphics[width=0.31\linewidth]{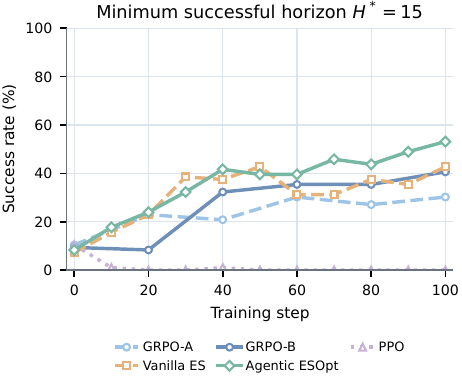}
}\vspace{-10pt}
\caption{Agentic Sudoku performance grouped by minimum successful horizon $H^*$. (a) reports final success rate averaged over 3 runs with standard-deviation error bars for PPO, the stronger GRPO-B configuration, and Agentic ESOpt. Red/green annotations below the Agentic ESOpt values report its difference from the stronger Agentic RL result. (b) and (c) show evaluation curves for the methods. Both Vanilla ES and Agentic ESOpt use $G=32$.}\vspace{-10pt}
\label{fig:sudoku-agentic}
\end{figure}

\begin{table}[H]
\centering
\small\vspace{-10pt}
\setlength{\tabcolsep}{12pt}
\caption{Agentic Sudoku final evaluation success rate ($\times$100) and GPU memory requirement, grouped by minimum successful horizon $H^*$. Values are reported as mean $\pm$ standard deviation.}
\label{tab:sudoku-agentic-summary}\vspace{-5pt}
\resizebox{\linewidth}{!}{%
\begin{tabular}{lcccc}
\toprule[0.5mm]
Method & GPU Mem Req. & $H^*=5$ & $H^*=10$ & $H^*=15$ \\
\midrule[0.3mm]
Qwen3.5-27B & 51.75GB & $86.46 \pm 3.90$ & $50.00 \pm 2.55$ & $28.13 \pm 2.55$ \\
Qwen3.5-4B & \textbf{8.41GB} & $63.54 \pm 7.80$ & $31.25 \pm 4.42$ & $10.42 \pm 1.47$ \\
\quad + Agentic PPO\textsuperscript{$*$} & 89.40GB & $\mathbf{90.63 \pm 0.00}$ & $56.25 \pm 0.00$ & $0.00 \pm 0.00$ \\
\quad + Agentic GRPO\textsuperscript{$\dagger$} & 58.88GB & $80.21 \pm 1.47$ & $44.79 \pm 2.95$ & $30.21 \pm 2.95$ \\
\quad + Agentic GRPO\textsuperscript{$\ddagger$} & 58.88GB & $85.42 \pm 1.47$ & $\mathbf{67.71 \pm 1.47}$ & $40.63 \pm 2.55$ \\\midrule
\rowcolor{dynamicrow}
\quad + Agentic ESOpt (G=32) & \textbf{8.41GB} & $89.58 \pm 2.95$ & $62.50 \pm 2.55$ & $\mathbf{53.13 \pm 2.55}$ \\
\quad w/o $\sigma$ decay (Vanilla ES) & \textbf{8.41GB} & $85.42 \pm 3.90$ & $55.21 \pm 5.89$ & $42.71 \pm 3.90$ \\
\quad w/o $\sigma_T$ (i.e., $\sigma_t=0$) & \textbf{8.41GB} & $85.42 \pm 3.90$ & $54.17 \pm 3.90$ & $28.13 \pm 2.55$ \\
\bottomrule[0.5mm]
\end{tabular}
}\vspace{2pt}

\begin{minipage}{\linewidth}
\footnotesize
\textsuperscript{$*$}PPO training use temperature $1$, top-$p=1$, and top-$k=-1$.
\textsuperscript{$\dagger$}GRPO training uses temperature $0.7$, top-$p=0.8$, and top-$k=20$.
\textsuperscript{$\ddagger$}GRPO training uses temperature $1$, top-$p=1$, and top-$k=-1$.
All GRPO and PPO evaluations, as well as Vanilla ES and Agentic ESOpt training and evaluation rollouts, use the recommended temperature-$0.7$, top-$p=0.8$, top-$k=20$.
\end{minipage}\vspace{-5pt}
\end{table}

\Cref{tab:sudoku-agentic-summary} and \Cref{fig:sudoku-agentic} (a) show a clear tendency that changes with the minimum successful horizon. At $H^*=5$, PPO leads with $90.63\%$, followed by Agentic ESOpt at $89.58\%$ and the stronger GRPO configuration at $85.42\%$. At $H^*=10$, GRPO leads with $67.71\%$, followed by Agentic ESOpt at $62.50\%$ and PPO at $56.25\%$. At $H^*=15$, Agentic ESOpt becomes strongest at $53.13\%$, 12.50 percentage points above GRPO at $40.63\%$, while PPO becomes ineffective with only terminal reward. Under sparse terminal rewards, its critic fails to learn a reliable value signal, making the resulting advantage estimates uninformative for credit assignment. Agentic ESOpt does not dominate with small $H^*$; its relative advantage appears as $H^*$ grows, where the credit assignment of Agentic RL becomes harder. The ordering reversal---PPO at $H^*=5$, GRPO at $H^*=10$, and Agentic ESOpt at $H^*=15$---is therefore more informative than a uniform win: it is consistent with a horizon-dependent advantage regime rather than a globally stronger optimizer or a single favorable decoding configuration.

The evaluation curves \Cref{fig:sudoku-agentic} (b,c) clarify how this separation develops. At $H^*=5$, methods have a similar learning curve. However, at $H^*=15$, GRPO variants fail to make fast improvements in the initial steps of training, while PPO quickly collapses. Their performance is related to the poor long-horizon credit assignment, where Agentic RL usually reaches the limit of the horizon (set to 45 for $H^*=15$). In \Cref{fig:sudoku-turn-diagnostics} (b), the $H^*=15$ setting has a 45-turn interaction budget; the GRPO turn count improves at around step 60, whereas Agentic ESOpt stays close to the minimum successful horizon and ends at 15.41 turns.

Compared to Vanilla ES, Agentic ESOpt can improve late in the longer run, from $39.58\%$ at step 60 to $53.13\%$ at step 100. This means that earlier, more extensive exploration given by the sigma schedule avoids the algorithm getting trapped in local optima. As shown in Table \ref{tab:sudoku-agentic-summary}, removing cosine decay over the sigma in Vanilla Agentic ES leads to worse results, while setting the final sigma to 0 results in overfitting and poor evaluation performance.

\takeaway{Under the \textit{weak-correlation scaling assumptions} used in the OpenAI ES analysis \citep[Sec.~3.1]{salimans2017evolution}, parameter-space ES avoids the explicit horizon-wise score accumulation of action-space policy gradients. The Sudoku experiment exhibits the predicted \textbf{horizon-dependent crossover}: Agentic ESOpt is not uniformly strongest at short horizons, but becomes strongest at the largest controlled $H^*$.}

\begin{table}[H]
\centering
\small\vspace{-3pt}
\setlength{\tabcolsep}{6pt}
\caption{Training compute and wall-clock time on agentic Sudoku. Both Agentic GRPO and Agentic ESOpt are executed on the same four NVIDIA H100 GPUs and fully utilize the available hardware.}
\label{tab:sudoku-efficiency}
\resizebox{\linewidth}{!}{%
\begin{tabular}{lcccccc}
\toprule[0.5mm]
& \multicolumn{2}{c}{$H^*=5$} & \multicolumn{2}{c}{$H^*=10$} & \multicolumn{2}{c}{$H^*=15$} \\
\cmidrule(lr){2-3}\cmidrule(lr){4-5}\cmidrule(lr){6-7}
Method & FLOPs & Time & FLOPs & Time & FLOPs & Time \\
\midrule
Qwen3.5-4B + Agentic GRPO & 3.2 EFLOPs & 5.4 h & 7.6 EFLOPs & 13.1 h & 10.9 EFLOPs & 19.0 h \\
\rowcolor{dynamicrow}
Qwen3.5-4B + Agentic ESOpt ($G$=32) & 3.1 EFLOPs & 3.1 h & 6.3 EFLOPs & 5.8 h & 9.4 EFLOPs & 9.4 h \\
\bottomrule[0.5mm]
\end{tabular}%
}\vspace{-5pt}
\end{table}

\paragraph{Compute and wall-clock efficiency.}
The memory and compute results characterize two complementary efficiency properties of Agentic ESOpt. The GPU-memory results in \Cref{tab:sudoku-agentic-summary} measure its minimal training-side requirement: Agentic ESOpt requires only 8.41GB, equal to the inference memory of the Qwen3.5-4B backbone and 85.7\% below GRPO's 58.88GB requirement.

ES requires a larger population ($G=32$ versus 8-rollout in GRPO); however, it does not introduce a corresponding model-compute disadvantage. As analyzed in Appendix \ref{app:flops-accounting}, Agentic ESOpt with $G=32$ has a comparable model-FLOPs budget to eight-rollout GRPO. As shown in \Cref{tab:sudoku-efficiency}, in the actual four-H100 implementation, the measured wall-clock results indicate that the larger ES population does not translate into additional end-to-end overhead in this environment. This comparison should be interpreted along separate efficiency axes: Agentic ESOpt is substantially more memory efficient and competitive in model compute, while it intentionally spends more independent environment evaluations to replace reference-model evaluation and backpropagation.

\takeaway{
Agentic ESOpt demonstrates \textbf{high model-side efficiency}, requiring only inference-level GPU memory. Although it evaluates more trajectories than Agentic GRPO ($G=32$ versus eight rollouts), the saved reference-model and backward-pass compute keeps model FLOPs and measured wall-clock time comparable in Sudoku, making full-parameter adaptation of the 27B agent in \Cref{tab:webarena} practical.
}

\section{Agentic ESOpt for Train-Time Fine-tuning}
\label{sec:train-time-co-optimize}

Besides the hand-crafted Sudoku, this section evaluates the ability of Agentic ESOpt in fine-tuning LLMs towards specific-purpose agents. We fine-tune agents for Python computation tool usage in Agentic Math reasoning and OCR and document-image analysis tool in DocVQA, and towards a web agent in WebArena.

\subsection{Agentic ReAct-Style Tool Usage on Math and DocVQA}
\label{sec:react-experiment}

We first fine-tune LLMs for the tool-use ability in ReAct-style interaction \citep{yao2022react}. This task is a relatively long-horizon scenario with a usual horizon $H>10$. We follow the setting and hyperparameters in \citet{ni2026trace2skill}, training Math on 400 DAPO problems and evaluating it on 100 held-out DAPO problems and out-of-distribution 30-problem AIME 2026. For DocVQA, we fine-tune agents on a 50-question validation subset and evaluate them on held-out 100 questions. In evaluating all methods, we set the horizon limit to 50 turns for both MATH and DocVQA, confining the max\_token in each turn to 4096 for MATH and 512 for DocVQA, respectively. This setting will not cause significant truncation. We implement two fine-tuning methods, Agentic GRPO and Agentic ESOpt, on Qwen3.5-4B. Besides, we also implement a prompt-space optimization method, Trace2Skill, which collects trajectories by sampling and synthesizes trajectories for skills at the end. We also implement a sequential combination of fine-tuning methods and Trace2Skill, where Agentic GRPO + Trace2Skill and Agentic ESOpt + Trace2Skill represent using the trajectories collected during training over No Skill prompts for skill distillation. GRPO is 8-rollout as well, and we have $G=16$ for Agentic ESOpt and 16 samples for Trace2Skill. As discussed in Appendix \ref{app:flops-accounting}, in fine-tuning agents, Agentic ESOpt consumes only approximately half the FLOPs of GRPO. Detailed settings are given in \Cref{app:math-docvqa-implementation,app:docvqa-implementation}.

\begin{table}[H]
\centering
\small\vspace{-2pt}
\setlength{\tabcolsep}{3.3pt}
\caption{Math-reasoning and DocVQA results under No Skill and Trace2Skill contexts. Qwen3.5-27B No Skill is an evaluation-only baseline; all optimized rows use Qwen3.5-4B. DAPO, AIME 2026, and DocVQA accuracy are reported in \%. Mean@4 averages four sampled Pass@1 accuracies; Pass@4 or Max@4 represent the best of four samples.}\vspace{-3pt}
\label{tab:math-docvqa}
\resizebox{\linewidth}{!}{%
\begin{tabular}{llcccccccc}
\toprule[0.5mm]
& & \multicolumn{4}{c}{Math reasoning} & \multicolumn{4}{c}{DocVQA} \\
\cmidrule(lr){3-6}\cmidrule(lr){7-10}
& & \multicolumn{2}{c}{DAPO} & \multicolumn{2}{c}{AIME 2026} & \multicolumn{2}{c}{ANLS} & \multicolumn{2}{c}{Accuracy} \\
\cmidrule(lr){3-4}\cmidrule(lr){5-6}\cmidrule(lr){7-8}\cmidrule(lr){9-10}
Model & Method & Mean@4 & Pass@4 & Mean@4 & Pass@4 & Mean@4 & Max@4 & Mean@4 & Pass@4 \\
\midrule[0.35mm]
Qwen3.5-27B & No Skill & 65.8 & 87.0 & 76.7 & 93.3 & 0.5036 & 0.7843 & 51.8 & 69.0 \\
\midrule[0.35mm]
Qwen3.5-4B & No Skill & 63.0 & \textbf{86.0} & 55.8 & 86.7 & 0.3875 & 0.5981 & 40.3 & 53.0 \\
Qwen3.5-4B & Agentic GRPO + No Skill & 68.8 & 83.0 & 58.3 & 76.7 & 0.4627 & 0.5398 & 48.0 & 56.0 \\
Qwen3.5-4B & \improvedcell{Agentic ESOpt + No Skill} &
\improvedcell{76.8} &
\textbf{86.0} &
\improvedcell{70.8} &
\improvedcell{\textbf{96.7}} &
\improvedcell{0.5043} &
\improvedcell{0.6507} &
\improvedcell{52.5} &
\improvedcell{\textbf{61.0}} \\
& {\scriptsize $\Delta$ vs No Skill} &
\improvedcell{\deltagain{$\uparrow$13.8}} &
\deltasame{0.0} &
\improvedcell{\deltagain{$\uparrow$15.0}} &
\improvedcell{\deltagain{$\uparrow$10.0}} &
\improvedcell{\deltagain{$\uparrow$0.1168}} &
\improvedcell{\deltagain{$\uparrow$0.0526}} &
\improvedcell{\deltagain{$\uparrow$12.3}} &
\improvedcell{\deltagain{$\uparrow$8.0}} \\
\midrule[0.35mm]
Qwen3.5-4B & Trace2Skill & 64.8 & 82.0 & 50.8 & 83.3 & 0.4612 & \textbf{0.6772} & 47.3 & \textbf{69.0} \\
Qwen3.5-4B & Agentic GRPO + Trace2Skill & 67.8 & 85.0 & 50.0 & 80.0 & 0.4743 & 0.5692 & 49.5 & 60.0 \\
Qwen3.5-4B & \improvedcell{Agentic ESOpt + Trace2Skill} &
\improvedcell{\textbf{77.3}} &
\improvedcell{\textbf{86.0}} &
\improvedcell{\textbf{71.7}} &
\improvedcell{\textbf{96.7}} &
\improvedcell{\textbf{0.5086}} &
0.6654 &
\improvedcell{\textbf{52.8}} &
61.0 \\
& {\scriptsize $\Delta$ vs Trace2Skill} &
\improvedcell{\deltagain{$\uparrow$12.5}} &
\improvedcell{\deltagain{$\uparrow$4.0}} &
\improvedcell{\deltagain{$\uparrow$20.8}} &
\improvedcell{\deltagain{$\uparrow$13.3}} &
\improvedcell{\deltagain{$\uparrow$0.0474}} &
\deltaloss{$\downarrow$0.0118} &
\improvedcell{\deltagain{$\uparrow$5.5}} &
\deltaloss{$\downarrow$8.0} \\
\bottomrule[0.5mm]
\end{tabular}%
}\vspace{-10pt}
\end{table}

As shown in Table \ref{tab:math-docvqa}, Agentic ESOpt consistently outperforms the matched Agentic GRPO baselines across ReAct-style Math and DocVQA. Without evolved skills, Agentic ESOpt improves the \textit{Qwen3.5-4B} base model by 13.8 and 15.0 percentage points on DAPO and AIME 2026 Mean@4, respectively, and improves DocVQA Mean@4 accuracy by 12.3 points. Averaged across these three metrics, Agentic ESOpt improves the base model by 13.7 points and Agentic GRPO by 8.3 points. Agentic ESOpt also composes effectively with Trace2Skill: the combined method achieves the strongest Qwen3.5-4B Mean@4 results on DAPO, AIME 2026, and DocVQA, showing the flexibility of parameter-space optimization methods to compose with external skill optimization.

Beyond the average-performance gains, Agentic ESOpt also preserves strong Pass@4 performance. We show the full Pass@K performance in Appendix \ref{app:math-docvqa-pass-at-k} up to $k=32$. Across datasets, Agentic ESOpt variants improve every reported Pass@K metric over their matched GRPO baselines \citep{yue2025does}. In AIME 2026 and DocVQA, Agentic ESOpt improves the Pass@K performance compared to the base LLM. This indicates that Agentic ESOpt performs a different and more favorable optimization paradigm compared to Agentic GRPO.

\takeaway{\textbf{Agentic ESOpt improves average performance without sacrificing Pass@K coverage}: across Math and DocVQA, both Agentic ESOpt variants outperform their matched Agentic GRPO baselines on every reported Pass@4 metric, suggesting broader coverage of successful trajectories after parameter optimization.}

\subsection{Agentic Reasoning for WebArena: Experiments on Model Scalability}
We then evaluate \ourmethod on WebArena-Lite, a 165-task browser benchmark derived from WebArena \citep{liu2025visualagentbench}. Each task provides a natural-language goal and an interactive website state. The agent observes a WebRL-style textual browser representation, acts in an id-based action space, and receives task-success feedback after completing a trajectory \citep{zhou2024webarena,qi2025webrl}. The evaluation set contains 21 Reddit, 32 GitLab, 35 CMS, 28 Map, 46 OSS, and 3 Wikipedia tasks; the five major categories are shown separately in \Cref{tab:webarena}, while Wikipedia is only included in the dataset average.

This setting scales Agentic ESOpt to Qwen3.5-27B on four NVIDIA H100 80GB GPUs. At this model scale, full-parameter Agentic RL is no longer practical on four H100 80GB GPUs. In contrast, Agentic ESOpt retains inference-level memory requirements, providing feasibility for us to perform full-parameter adaptation of a 27B web agent. This experiment therefore focuses on large-model feasibility and real-world agent adaptation, complementing the controlled ES-versus-RL comparison in Section~\ref{sec:sudoku-agentic} and \ref{sec:react-experiment}. We compare two paired settings using the same Qwen3.5-27B backbone: No Skill versus Agentic ESOpt + No Skill, and Trace2Skill versus Agentic ESOpt + Trace2Skill. Both Agentic ESOpt variants share the same parameter updates learned without skills; the combined variant additionally distills a skill from the resulting trajectories, without a second skill-conditioned ES stage or alternating joint optimization. We also report closed-source GPT-5.4, GPT-5.4-mini, and GPT-5.4-nano as reference points. Complete setup details and evaluation curves are provided in Appendix~\ref{app:webarena-details} and \ref{app:webarena-eval-curve}, while task prompts and distilled skills are provided in \ref{app:prompt-webarena}, \ref{app:webarena-trace2skill-skill}, and \ref{app:webarena-es-trace2skill-skill}.

All the Qwen3.5-27B rollouts for Agentic ESOpt, Trace2skill, and evaluation use temperature $0.7$, a 2048-token generation budget, WebRL-style observations, and WebRL id-based actions. Agentic ESOpt sets the population size $G=8$, uses task success on sampled training cases as the reward. Trace2skill also samples each question for 8 runs. Final results are evaluated on the official 165-task WebArena-Lite evaluation set \citep{liu2025visualagentbench}.

\begin{table}[H]
\centering
\small\vspace{-5pt}
\setlength{\tabcolsep}{5pt}
\caption{WebArena-Lite success rates (\%). We report the average over 3 runs and the standard deviations. Improved Agentic ESOpt cells relative to their paired baseline are shaded green; bold denotes the best result in each column. The five displayed site categories contain 162 tasks; the remaining Wikipedia tasks are included in Dataset Avg.}\vspace{-5pt}
\label{tab:webarena}
\resizebox{\linewidth}{!}{%
\begin{tabular}{llcccccc}
\toprule[0.5mm]
Model & Method & Reddit (21) & GitLab (32) & CMS (35) & Map (28) & OSS (46) & Dataset Avg. \\
\midrule[0.35mm]
\multicolumn{8}{l}{\textit{Strong frozen baseline (evaluation only)}} \\
GPT-5.4 & No Skill & 47.62 & \textbf{46.88} & 46.67 & \textbf{19.05} & 21.01 & 34.14 {\scriptsize $\pm$ 0.76} \\
GPT-5.4-mini & No Skill & 39.68 & 29.17 & 30.48 & 13.10 & 13.77 & 23.23 {\scriptsize $\pm$ 1.14} \\
GPT-5.4-nano & No Skill & 39.68 & 27.08 & 19.05 & 11.90 & 8.70 & 18.79 {\scriptsize $\pm$ 0.99} \\
\midrule[0.35mm]
Qwen3.5-27B & No Skill & 50.79 & 35.42 & 41.90 & 8.33 & 21.01 & 29.47 {\scriptsize $\pm$ 1.14} \\
Qwen3.5-27B & \improvedcell{Agentic ESOpt + No Skill} &
49.21 &
\improvedcell{43.75} &
\improvedcell{49.52} &
\improvedcell{14.29} &
\improvedcell{30.43} &
\improvedcell{36.16 {\scriptsize $\pm$ 0.70}} \\
& {\scriptsize $\Delta$ vs No Skill} &
\deltaloss{$\downarrow$1.58} &
\improvedcell{\deltagain{$\uparrow$8.33}} &
\improvedcell{\deltagain{$\uparrow$7.62}} &
\improvedcell{\deltagain{$\uparrow$5.96}} &
\improvedcell{\deltagain{$\uparrow$9.42}} &
\improvedcell{\deltagain{$\uparrow$6.69}} \\
\midrule[0.25mm]
Qwen3.5-27B & Trace2Skill & 49.21 & 39.58 & 46.67 & 13.10 & 28.26 & 33.94 {\scriptsize $\pm$ 3.37} \\
Qwen3.5-27B & \improvedcell{Agentic ESOpt + Trace2Skill} &
\improvedcell{\textbf{52.80}} &
\improvedcell{41.67} &
\improvedcell{\textbf{50.48}} &
10.71 &
\improvedcell{\textbf{32.61}} &
\improvedcell{\textbf{36.36 {\scriptsize $\pm$ 0.86}}} \\
& {\scriptsize $\Delta$ vs Trace2Skill} &
\improvedcell{\deltagain{$\uparrow$3.59}} &
\improvedcell{\deltagain{$\uparrow$2.09}} &
\improvedcell{\deltagain{$\uparrow$3.81}} &
\deltaloss{$\downarrow$2.39} &
\improvedcell{\deltagain{$\uparrow$4.35}} &
\improvedcell{\deltagain{$\uparrow$2.42}} \\
\bottomrule[0.5mm]
\end{tabular}
}\vspace{-10pt}
\end{table}

As shown in \Cref{tab:webarena}, Agentic ESOpt improves the Qwen3.5-27B No Skill baseline from 29.47\% to 36.16\%, a gain of 6.69 percentage points. When combined with Trace2Skill, it further improves the Trace2Skill baseline from 33.94\% to 36.36\%, demonstrating the flexibility of Agentic ESOpt in composing with skill-space optimization on Larger LLMs.

\section{Agentic ESOpt on Test-Time Compute: Automatic Heuristic Design}
\label{sec:test-time-heuristic-design}

Traditional test-time compute generally fixes the LLM parameters throughout the search process, which may limit optimization when solving the task requires changing the underlying policy itself. Due to the high flexibility of Agentic ESOpt, it can incorporate test-time compute to improve its performance \citep{snell2024scaling}.

In this section, we evaluate the ability of Agentic ESOpt to boost a typical test-time compute problem--automatic heuristic design (AHD) \citep{liu2024evolution,zheng2025mctsahd,huang2026calm,liu2025fine}. Among them, the representative work EoH \citep{liu2024evolution} improves the performance of these algorithms on a fixed training dataset by maintaining a population of heuristic algorithms and applying crossover and mutation operators. In generating each new heuristic, it provides an LLM agent with a problem description and a Python function signature; the LLM generates heuristic code, and the evaluator returns a scalar objective after executing the code within a solver. Agentic ESOpt can be inserted into existing EoH and pure sample procedures without changing their outer search scaffold. Agentic ESOpt + Sample or EoH therefore supports both heuristic-space search and parameter-space adaptation. We consider two AHD scenarios and compare Sample and EoH with their corresponding Agentic ESOpt variants under the same evaluation budgets. In constructive AHD, the heuristic makes local decisions for NP-hard combinatorial optimization problems (defined in Appendix \ref{app:ahd-details}): TSP, KP, and ASP; In ACO-style AHD, the algorithm provides heuristic information to an ant-colony optimizer for TSP, CVRP, and BPP. TSP, CVRP, and BPP are minimized, whereas KP and ASP are maximized.

For constructive tasks, the \textit{Optimal} row reports the known optimum $x^\star$ only as a reference. We define the normalized optimality gap as $g(x)=\frac{|x-x^\star|}{|x^\star|}$, and report the gain of an Agentic ESOpt result $m$ over its matched baseline $b$ as $\Delta=\frac{g(b)}{g(m)}-1.$
The formal heuristic-space view, complete search settings, ACO-style results, ablations, and exact generation prompt are indexed in Appendix \ref{app:ahd-details}, \ref{app:ahd-setting}, \ref{app:ahd-additional-results}, and \ref{app:prompt-ahd}.

For Agentic ESOpt + EoH, all objectives are represented internally as minimization costs, and the reward is the selected parent's cost minus the perturbed child's cost. For Agentic ESOpt + Sample, the reward is the negative child cost. Rewards are z-score normalized within each parameter-update batch. Agentic ESOpt preserves the original EoH operators and attaches parameter updates only to the mutation operators $m1$ and $m2$. We report the best candidate from the final generation and average over repeated runs. A 20-run significance analysis on representative TSP and KP settings is provided in \Cref{tab:ahd-significance}.

\begin{table}[H]
\centering
\small\vspace{-5pt}
\setlength{\tabcolsep}{6pt}
\caption{Design constructive heuristics at total evaluation $T\in\{1000,2000\}$. TSP is minimized; KP and ASP are maximized. The Optimal row shows optimum objective values, not an experimental baseline. Sample and EoH are each paired with their Agentic ESOpt counterpart; the $\Delta$ rows report the gap-ratio gain defined in the text. Baselines and Agentic ESOpt use \textbf{\textit{LLaMA-3.1-8B-Instruct}} for all runs.}\vspace{-5pt}
\label{tab:ahd-construct}
\resizebox{\linewidth}{!}{%
\begin{tabular}{lcccccc}
\toprule[0.5mm]
Method & \makecell{TSP\\$N=20$} & \makecell{TSP\\$N=50$} & \makecell{KP\\$N=50,W=12.5$} & \makecell{KP\\$N=100,W=25$} & \makecell{ASP\\$N=15,W=10$} & \makecell{ASP\\$N=21,W=15$} \\
\midrule
Optimal & 3.8199 & 5.6750 & 20.0370 & 40.2710 & 3{,}003 & 43{,}596 \\
Greedy Construct & 4.4797 & 6.9590 & 19.9850 & 40.2250 & 1{,}530 & 15{,}050 \\
\midrule[0.35mm]
\multicolumn{7}{c}{\textit{Total Evaluations: $T=1000$}} \\
\midrule[0.35mm]
Sample & 4.3286 & 6.7110 & 19.9896 & 40.2297 & 2{,}753 & 30{,}336.67 \\
\improvedcell{Agentic ESOpt + Sample} &
\improvedcell{4.2336} &
\improvedcell{6.5488} &
\improvedcell{19.9899} &
\improvedcell{40.2314} &
2{,}729 &
\improvedcell{31{,}082.67} \\
{\scriptsize $\Delta$ vs Sample} &
\improvedcell{\deltagain{$\downarrow$22.96\%}} &
\improvedcell{\deltagain{$\downarrow$18.56\%}} &
\improvedcell{\deltagain{$\uparrow$0.79\%}} &
\improvedcell{\deltagain{$\uparrow$4.42\%}} &
\deltaloss{$\downarrow$8.76\%} &
\improvedcell{\deltagain{$\uparrow$5.96\%}} \\
\midrule
EoH & 4.2481 & 6.5450 & 19.9958 & 40.2320 & 2{,}760 & 28{,}465.67 \\
\improvedcell{Agentic ESOpt + EoH} &
\improvedcell{4.2170} &
\improvedcell{6.4631} &
\improvedcell{20.0007} &
\improvedcell{40.2351} &
\improvedcell{2{,}770} &
\improvedcell{30{,}887.33} \\
{\scriptsize $\Delta$ vs EoH} &
\improvedcell{\deltagain{$\downarrow$7.83\%}} &
\improvedcell{\deltagain{$\downarrow$10.39\%}} &
\improvedcell{\deltagain{$\uparrow$13.59\%}} &
\improvedcell{\deltagain{$\uparrow$8.65\%}} &
\improvedcell{\deltagain{$\uparrow$4.67\%}} &
\improvedcell{\deltagain{$\uparrow$29.18\%}} \\
\midrule[0.35mm]
\multicolumn{7}{c}{\textit{Total Evaluations: $T=2000$}} \\
\midrule[0.35mm]
Sample & 4.2585 & 6.6008 & 19.99118 & 40.2320 & 2{,}759 & 30{,}123.33 \\
\improvedcell{Agentic ESOpt + Sample} &
\improvedcell{4.2098} &
\improvedcell{6.5332} &
\improvedcell{19.99120} &
40.2320 &
2{,}751 &
\improvedcell{30{,}269.33} \\
{\scriptsize $\Delta$ vs Sample} &
\improvedcell{\deltagain{$\downarrow$12.51\%}} &
\improvedcell{\deltagain{$\downarrow$7.88\%}} &
\improvedcell{\deltagain{$\uparrow$0.04\%}} &
\deltasame{0.00\%} &
\deltaloss{$\downarrow$3.17\%} &
\improvedcell{\deltagain{$\uparrow$1.10\%}} \\
\midrule
EoH & 4.2165 & 6.4706 & 19.9984 & 40.2356 & 2{,}764 & 28{,}016.00 \\
\improvedcell{Agentic ESOpt + EoH} &
\improvedcell{4.1799} &
\improvedcell{6.4442} &
\improvedcell{19.9987} &
\improvedcell{40.2358} &
\improvedcell{2{,}784} &
\improvedcell{30{,}124.67} \\
{\scriptsize $\Delta$ vs EoH} &
\improvedcell{\deltagain{$\downarrow$10.19\%}} &
\improvedcell{\deltagain{$\downarrow$3.43\%}} &
\improvedcell{\deltagain{$\uparrow$0.81\%}} &
\improvedcell{\deltagain{$\uparrow$0.48\%}} &
\improvedcell{\deltagain{$\uparrow$9.92\%}} &
\improvedcell{\deltagain{$\uparrow$24.36\%}} \\
\bottomrule[0.5mm]
\end{tabular}
}\vspace{-10pt}
\end{table}

As shown in \Cref{tab:ahd-construct}, Agentic ESOpt + EoH improves all six constructive test sets at both evaluation budgets. Agentic ESOpt + Sample improves nine of twelve comparisons, with one tie and two regressions. Across both constructive baselines, Agentic ESOpt therefore improves 21 of 24 matched comparisons. The ACO-style results and additional ablations are reported in \Cref{app:ahd-additional-results}. Combining the constructive and ACO-style settings, Agentic ESOpt improves 28 of 36 matched method--budget comparisons across 12 test sets and six scenarios.

\takeaway{
AHD demonstrates that Agentic ESOpt can extend test-time compute from search over a frozen LLM policy to \textbf{coupled heuristic--parameter optimization}. Without modifying the outer Sample or EoH search scaffold and under matched evaluation budgets, Agentic ESOpt improves \textbf{28 of 36} comparisons, showing that lightweight black-box parameter adaptation can systematically enhance existing test-time search.
}

\section{Discussion}

We use this section to clarify the regime in which parameter-space ES is most attractive, the current scope of the evidence, and the main directions for extending Agentic ESOpt.

\subsection{Population Scaling with Model Size}

The population size $G$ is a key hyperparameter of ES. Our current results provide preliminary evidence that stronger backbones may be less sensitive to small populations. As shown in \Cref{tab:population-scaling}, for the 4B model, increasing the population from $G=8$ to $G=16$ on the same 15-turn Sudoku setting increases the best test accuracy from $5.10\%$ to $35.42\%$, and final test success increases from $2.95\%$ to $22.92\%$. However, the 9B model is substantially less dependent on a large population. Increasing $G$ from 8 to 16 changes best test success from $30.21\%$ to $37.50\%$, a $24.1\%$ relative improvement, while final test success remains $30.21\%$, giving no relative improvement.

\begin{figure}[H]
\centering
\begin{minipage}[t]{0.52\linewidth}
\vspace{0pt}
\centering
\captionsetup{type=table}
\caption{Vanilla-ES population sensitivity on 15-turn Sudoku. Relative changes compare $G=16$ with $G=8$.}
\label{tab:population-scaling}
\vspace{5pt}
\scriptsize
\setlength{\tabcolsep}{3pt}
\resizebox{\linewidth}{!}{%
\begin{tabular}{lccccc}
\toprule
Backbone & $G$ & \makecell{Best\\test} & \makecell{Final\\test} & \makecell{$\Delta$ best\\test} & \makecell{$\Delta$ final\\test} \\
\midrule
Qwen3.5-4B & 8  & 5.10 & 2.95 & -- & -- \\
Qwen3.5-4B & 16 & 35.42 & 22.92 & $+594.5\%$ & $+677.0\%$ \\
\midrule
Qwen3.5-9B & 8  & 30.21 & 30.21 & -- & -- \\
Qwen3.5-9B & 16 & 37.50 & 30.21 & $+24.1\%$ & $0.0\%$ \\
\bottomrule
\end{tabular}%
}
\end{minipage}\hfill
\begin{minipage}[t]{0.44\linewidth}
\vspace{0pt}
\centering
\includegraphics[width=\linewidth]{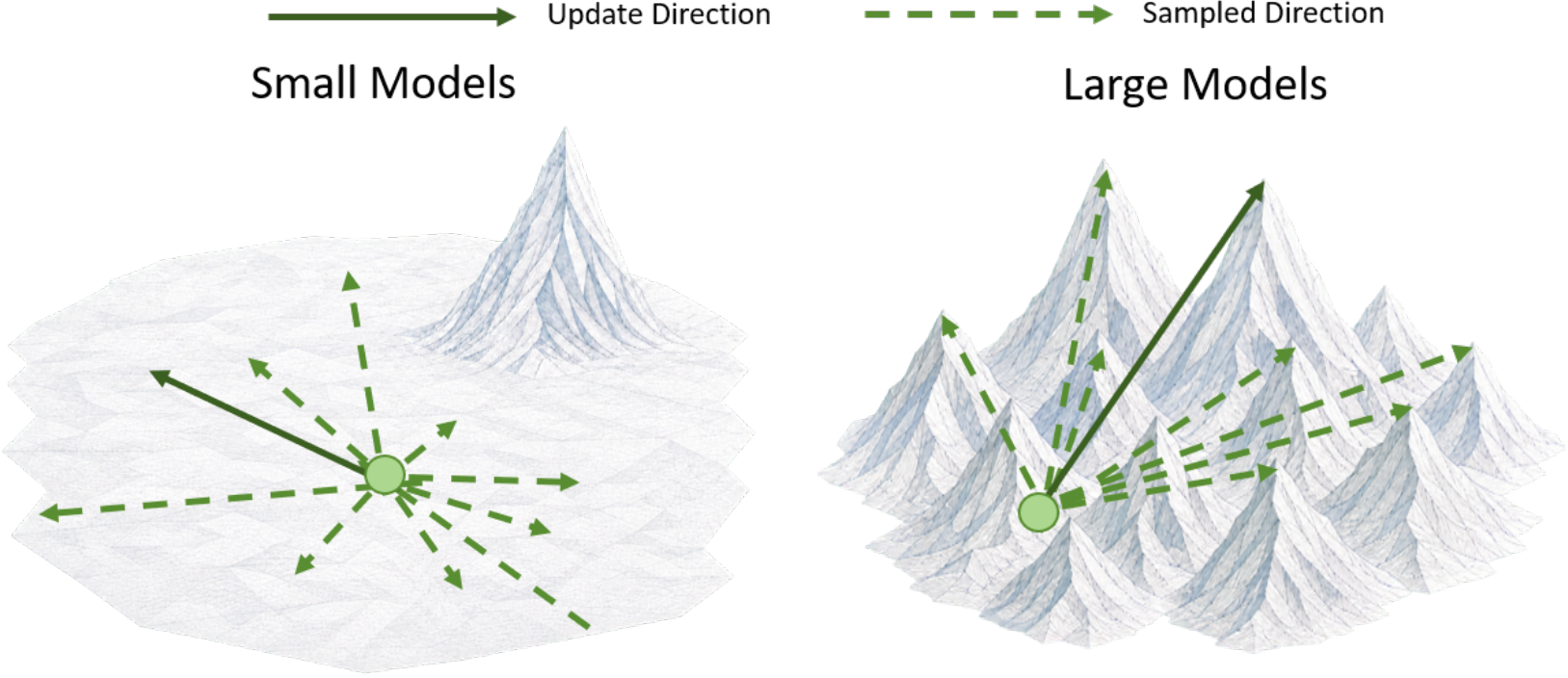}
\vspace{-5pt}
\caption{Intuition for population scaling: sampled directions around a stronger backbone are more likely to align with a useful Agentic ESOpt direction.}
\label{fig:population-scaling-intuition}
\end{minipage}
\end{figure}

As shown in \Cref{fig:population-scaling-intuition}, one plausible interpretation is that a pre-trained stronger backbone will lead to a more competent local region, so a larger fraction of nearby perturbations produce behavior informative for the update. This interpretation is consistent with the finding of \citet{gan2026neural} that useful behavioral diversity becomes denser around stronger pre-trained models. This supports the prospect of using Agentic ESOpt to fine-tune larger advanced LLMs with fewer FLOPs; we will take establishing a universal scaling law and implementing Agentic ESOpt on frontier LLMs as future work.

\takeaway{The experiment provides \textbf{initial population-sensitivity evidence}: doubling $G$ changes final-test success by $+677.0\%$ for 4B but $0.0\%$ for 9B. That indicates stronger backbones may need fewer sampled directions.}

\section{Conclusion}

This paper presents \ourmethod, a full-parameter evolution-strategy framework for fine-tuning long-horizon LLM agents requiring only minimal inference-level GPU memory. We argue that \textbf{for long-horizon agents, Agentic ESOpt is a more natural optimization paradigm than conventional Agentic RL}. As interaction horizons grow and feedback becomes sparse, policy gradients must assign a trajectory-level outcome across an expanding sequence of actions. Agentic ESOpt instead attributes the outcome to a coherent policy perturbation, avoiding explicit horizon-wise action-score accumulation.

Our experiments support this shift. On controlled Sudoku, Agentic PPO and Agentic GRPO remain competitive at shorter horizons, whereas \textbf{Agentic ESOpt becomes strongest as the minimum successful horizon grows}; the advantage further extends to ReAct-style Math and DocVQA. Agentic ESOpt also provides a model-scaling advantage through forward-only optimization. Its \textbf{inference-level memory} footprint enables full-parameter adaptation of a Qwen3.5-27B WebArena agent on four H100 GPUs.

Beyond train-time fine-tuning, its \textbf{flexibility} naturally integrates parameter adaptation into broader optimization loops of test-time compute. Agentic ESOpt improves skill optimization and test-time heuristic search, including 28 of 36 matched AHD comparisons under fixed evaluation budgets.

Taken together, these results position ES not as a cheaper substitute for RL, but as \textbf{a better-matched optimization mechanism for long-horizon, sparse-feedback LLM agents}. As agentic systems move toward longer interactions, larger models, and increasingly complex prompt-space interactive components, Agentic ESOpt may become a central paradigm for agent fine-tuning and online self-improvement.

Limitations of the current study are discussed in Appendix \ref{app:limitation}, and future work is outlined in Appendix \ref{app:futurework}.

\section*{Acknowledgement}
We would like to sincerely thank Jiaying Wu, Penghui Qi, Zichen Liu, and Ziqiao Meng from the National University of Singapore, as well as Ziang Li from humans\& for their important comments on the methodology and paper-writing.

\bibliographystyle{assets/plainnat}
\bibliography{paper}

@article{snell2024scaling,
  title={Scaling llm test-time compute optimally can be more effective than scaling model parameters},
  author={Snell, Charlie and Lee, Jaehoon and Xu, Kelvin and Kumar, Aviral},
  journal={arXiv preprint arXiv:2408.03314},
  year={2024}
}

@article{salimans2017evolution,
  title={Evolution strategies as a scalable alternative to reinforcement learning},
  author={Salimans, Tim and Ho, Jonathan and Chen, Xi and Sidor, Szymon and Sutskever, Ilya},
  journal={arXiv preprint arXiv:1703.03864},
  year={2017}
}

@inproceedings{he2026hierarchy,
  title={Hierarchy-of-groups policy optimization for long-horizon agentic tasks},
  author={He, Shuo and Feng, Lang and Cheng, Xin and Feng, Lei and An, Bo and others},
  booktitle={International Conference on Learning Representations},
  volume={2026},
  pages={27572--27593},
  year={2026}
}

@article{hoy2026matching,
  title={Matching Accuracy, Different Geometry: Evolution Strategies vs GRPO in LLM Post-Training},
  author={Hoy, William and Wang, Binxu and Pan, Xu},
  journal={arXiv preprint arXiv:2604.01499},
  year={2026}
}

@article{wang2026milestone,
  title={Milestone-guided policy learning for long-horizon language agents},
  author={Wang, Zixuan and Yan, Yuchen and Li, Hongxing and Pan, Teng and Li, Dingming and Zhang, Ruiqing and Lu, Weiming and Xiao, Jun and Zhuang, Yueting and Shen, Yongliang},
  journal={arXiv preprint arXiv:2605.06078},
  year={2026}
}

@article{shen2026skillopt,
  title={SkillOpt-Lite: Better and Faster Agent Self-evolution via One Line of Vibe},
  author={Shen, Yifei and Li, Bo and Zhang, Xinjie},
  journal={arXiv preprint arXiv:2607.03451},
  year={2026}
}

@article{team2026kimi,
  title={Kimi k3: Open frontier intelligence},
  author={Team, Kimi and Bai, Tongtong and Bai, Yifan and Bao, Yiping and Cai, Jianfeng and Cai, Xinyuan and Cao, Peizhou and Cao, Yuxuan and Chai, Ziwei and Charles, Y and others},
  journal={arXiv preprint arXiv:2607.24653},
  year={2026}
}

@inproceedings{huang2026calm,
  title={Calm: Co-evolution of algorithms and language model for automatic heuristic design},
  author={Huang, Ziyao and Wu, Weiwei and Wu, Kui and Lee, Wei-Bin and Wang, Jianping},
  booktitle={International Conference on Learning Representations},
  volume={2026},
  pages={72468--72510},
  year={2026}
}

@article{liu2025fine,
  title={Fine-tuning large language model for automated algorithm design},
  author={Liu, Fei and Zhang, Rui and Lin, Xi and Lu, Zhichao and Zhang, Qingfu},
  journal={arXiv preprint arXiv:2507.10614},
  year={2025}
}

@inproceedings{liu2025visualagentbench,
  title={Visualagentbench: Towards large multimodal models as visual foundation agents},
  author={Liu, Xiao and Zhang, Tianjie and Gu, Yu and Iong, Iat Long and XiXuan, Song and Xu, Yifan and Zhang, Shudan and Lai, Hanyu and Sun, Jiadai and Yang, Xinyue and others},
  booktitle={International Conference on Learning Representations},
  volume={2025},
  pages={95650--95707},
  year={2025}
}

@article{zheng2025soft,
  title={Soft-grpo: Surpassing discrete-token llm reinforcement learning via gumbel-reparameterized soft-thinking policy optimization},
  author={Zheng, Zhi and Gu, Yu and Liu, Wei and Teh, Yee Whye and Lee, Wee Sun},
  journal={arXiv preprint arXiv:2511.06411},
  year={2025}
}

@article{gan2026neural,
  title={Neural thickets: Diverse task experts are dense around pretrained weights},
  author={Gan, Yulu and Isola, Phillip},
  journal={arXiv preprint arXiv:2603.12228},
  year={2026}
}

@article{liu2025gem,
  title={Gem: A gym for agentic llms},
  author={Liu, Zichen and Sims, Anya and Duan, Keyu and Chen, Changyu and Yu, Simon and Zhou, Xiangxin and Xu, Haotian and Xiong, Shaopan and Liu, Bo and Tan, Chenmien and others},
  journal={arXiv preprint arXiv:2510.01051},
  year={2025}
}

@article{madaan2023self,
  title={Self-refine: Iterative refinement with self-feedback},
  author={Madaan, Aman and Tandon, Niket and Gupta, Prakhar and Hallinan, Skyler and Gao, Luyu and Wiegreffe, Sarah and Alon, Uri and Dziri, Nouha and Prabhumoye, Shrimai and Yang, Yiming and others},
  journal={Advances in neural information processing systems},
  volume={36},
  pages={46534--46594},
  year={2023}
}

@article{zhu2025scaling,
  title={Scaling test-time compute for llm agents},
  author={Zhu, King and Li, Hanhao and Wu, Siwei and Xing, Tianshun and Ma, Dehua and Tang, Xiangru and Liu, Minghao and Yang, Jian and Liu, Jiaheng and Jiang, Yuchen Eleanor and others},
  journal={arXiv preprint arXiv:2506.12928},
  year={2025}
}

@article{yao2023tree,
  title={Tree of thoughts: Deliberate problem solving with large language models},
  author={Yao, Shunyu and Yu, Dian and Zhao, Jeffrey and Shafran, Izhak and Griffiths, Tom and Cao, Yuan and Narasimhan, Karthik},
  journal={Advances in neural information processing systems},
  volume={36},
  pages={11809--11822},
  year={2023}
}

@article{liu2025ea4llm,
  title={EA4LLM: A Gradient-Free Approach to Large Language Model Optimization via Evolutionary Algorithms},
  author={Liu, WenTao and Song, Siyu and Hao, Hao and Zhou, Aimin},
  journal={arXiv preprint arXiv:2510.10603},
  year={2025}
}

@article{korotyshova2025essa,
  title={ESSA: Evolutionary Strategies for Scalable Alignment},
  author={Korotyshova, Daria and Shaposhnikov, Boris and Malakhov, Alexey and Khokhulin, Alexey and Surnachev, Nikita and Ovcharenko, Kirill and Bredis, George and Gorbatovski, Alexey and Sinii, Viacheslav and Gavrilov, Daniil},
  journal={arXiv preprint arXiv:2507.04453},
  year={2025}
}

@article{fu2026reasoning,
  title={Reasoning Resides in Layers: Restoring Temporal Reasoning in Video-Language Models with Layer-Selective Merging},
  author={Fu, Zihang and Wang, Haonan and Kang, Jian and Kawaguchi, Kenji and Wu, Jiaying},
  journal={arXiv preprint arXiv:2604.11399},
  year={2026}
}

@inproceedings{zhao2025second,
  title={Second-order fine-tuning without pain for llms: A hessian informed zeroth-order optimizer},
  author={Zhao, Yanjun and Dang, Sizhe and Ye, Haishan and Dai, Guang and Qian, Yi and Tsang, Ivor},
  booktitle={International Conference on Learning Representations},
  volume={2025},
  pages={43496--43520},
  year={2025}
}

@article{malladi2023fine,
  title={Fine-tuning language models with just forward passes},
  author={Malladi, Sadhika and Gao, Tianyu and Nichani, Eshaan and Damian, Alex and Lee, Jason D and Chen, Danqi and Arora, Sanjeev},
  journal={Advances in Neural Information Processing Systems},
  volume={36},
  pages={53038--53075},
  year={2023}
}

@article{gu2026hyper,
  title={Hyper-ES: Effective Evolution Strategies for LLM Reasoning via Descent Direction Merging},
  author={Gu, Yu and Zheng, Zhi and Ba, Yunpeng and Tong, Xialiang and Yuan, Mingxuan and Wang, Zhenkun},
  journal={arXiv preprint arXiv:2608.05541},
  year={2026}
}

@article{zhang2026reasoning,
  title={From reasoning to agentic: Credit assignment in reinforcement learning for large language models},
  author={Zhang, Chenchen},
  journal={arXiv preprint arXiv:2604.09459},
  year={2026}
}

@article{tajwar2026maximum,
  title={Maximum likelihood reinforcement learning},
  author={Tajwar, Fahim and Zeng, Guanning and Zhou, Yueer and Song, Yuda and Arora, Daman and Jiang, Yiding and Schneider, Jeff and Salakhutdinov, Ruslan and Feng, Haiwen and Zanette, Andrea},
  journal={arXiv preprint arXiv:2602.02710},
  year={2026}
}

@article{shao2024deepseekmath,
  title={Deepseekmath: Pushing the limits of mathematical reasoning in open language models},
  author={Shao, Zhihong and Wang, Peiyi and Zhu, Qihao and Xu, Runxin and Song, Junxiao and Bi, Xiao and Zhang, Haowei and Zhang, Mingchuan and Li, YK and Wu, Y and others},
  journal={arXiv preprint arXiv:2402.03300},
  year={2024}
}

@article{liu2025understanding,
  title={Understanding r1-zero-like training: A critical perspective},
  author={Liu, Zichen and Chen, Changyu and Li, Wenjun and Qi, Penghui and Pang, Tianyu and Du, Chao and Lee, Wee Sun and Lin, Min},
  journal={arXiv preprint arXiv:2503.20783},
  year={2025}
}

@article{chu2025sft,
  title={Sft memorizes, rl generalizes: A comparative study of foundation model post-training},
  author={Chu, Tianzhe and Zhai, Yuexiang and Yang, Jihan and Tong, Shengbang and Xie, Saining and Schuurmans, Dale and Le, Quoc V and Levine, Sergey and Ma, Yi},
  journal={arXiv preprint arXiv:2501.17161},
  year={2025}
}

@article{schulman2017proximal,
  title={Proximal policy optimization algorithms},
  author={Schulman, John and Wolski, Filip and Dhariwal, Prafulla and Radford, Alec and Klimov, Oleg},
  journal={arXiv preprint arXiv:1707.06347},
  year={2017}
}

@article{guo2025deepseek,
  title={Deepseek-r1: Incentivizing reasoning capability in llms via reinforcement learning},
  author={Guo, Daya and Yang, Dejian and Zhang, Haowei and Song, Junxiao and Zhang, Ruoyu and Xu, Runxin and Zhu, Qihao and Ma, Shirong and Wang, Peiyi and Bi, Xiao and others},
  journal={arXiv preprint arXiv:2501.12948},
  year={2025}
}

@article{yue2025does,
  title={Does reinforcement learning really incentivize reasoning capacity in llms beyond the base model?},
  author={Yue, Yang and Chen, Zhiqi and Lu, Rui and Zhao, Andrew and Wang, Zhaokai and Song, Shiji and Huang, Gao},
  journal={arXiv preprint arXiv:2504.13837},
  year={2025}
}

@article{yu2025dapo,
  title={Dapo: An open-source llm reinforcement learning system at scale},
  author={Yu, Qiying and Zhang, Zheng and Zhu, Ruofei and Yuan, Yufeng and Zuo, Xiaochen and Yue, Yu and Dai, Weinan and Fan, Tiantian and Liu, Gaohong and Liu, Lingjun and others},
  journal={arXiv preprint arXiv:2503.14476},
  year={2025}
}

@article{xu2025mem,
  title={A-mem: Agentic memory for llm agents},
  author={Xu, Wujiang and Liang, Zujie and Mei, Kai and Gao, Hang and Tan, Juntao and Zhang, Yongfeng},
  journal={arXiv preprint arXiv:2502.12110},
  year={2025}
}

@article{wu2025human,
  title={From human memory to ai memory: A survey on memory mechanisms in the era of llms},
  author={Wu, Yaxiong and Liang, Sheng and Zhang, Chen and Wang, Yichao and Zhang, Yongyue and Guo, Huifeng and Tang, Ruiming and Liu, Yong},
  journal={arXiv preprint arXiv:2504.15965},
  year={2025}
}

@article{zheng2025mctsahd,
  title = {Monte Carlo Tree Search for Comprehensive Exploration in LLM-Based Automatic Heuristic Design},
  author = {Zheng, Zhi and Xie, Zhuoliang and Wang, Zhenkun and Hooi, Bryan},
  journal = {arXiv preprint arXiv:2501.08603},
  year = {2025}
}

@inproceedings{qi2025webrl,
  title={Webrl: Training llm web agents via self-evolving online curriculum reinforcement learning},
  author={Qi, Zehan and Liu, Xiao and Iong, Iat Long and Lai, Hanyu and Sun, Xueqiao and Sun, Jiadai and Yang, Xinyue and Yang, Yu and Yao, Shuntian and Xu, Wei and others},
  booktitle={International Conference on Learning Representations},
  volume={2025},
  pages={79791--79821},
  year={2025}
}

@inproceedings{zhou2024webarena,
  title={Webarena: A realistic web environment for building autonomous agents},
  author={Zhou, Shuyan and Xu, Frank F and Zhu, Hao and Zhou, Xuhui and Lo, Robert and Sridhar, Abishek and Cheng, Xianyi and Ou, Tianyue and Bisk, Yonatan and Fried, Daniel and others},
  booktitle={International Conference on Learning Representations},
  volume={2024},
  pages={15585--15606},
  year={2024}
}

@article{liu2024evolution,
  title={Evolution of heuristics: Towards efficient automatic algorithm design using large language model},
  author={Liu, Fei and Tong, Xialiang and Yuan, Mingxuan and Lin, Xi and Luo, Fu and Wang, Zhenkun and Lu, Zhichao and Zhang, Qingfu},
  journal={arXiv preprint arXiv:2401.02051},
  year={2024}
}

@article{yao2022react,
  title={React: Synergizing reasoning and acting in language models},
  author={Yao, Shunyu and Zhao, Jeffrey and Yu, Dian and Du, Nan and Shafran, Izhak and Narasimhan, Karthik and Cao, Yuan},
  journal={arXiv preprint arXiv:2210.03629},
  year={2022}
}

@article{shinn2023reflexion,
  title={Reflexion: Language agents with verbal reinforcement learning},
  author={Shinn, Noah and Cassano, Federico and Gopinath, Ashwin and Narasimhan, Karthik and Yao, Shunyu},
  journal={Advances in neural information processing systems},
  volume={36},
  pages={8634--8652},
  year={2023}
}

@article{yang2026skillopt,
  title={Skillopt: Executive strategy for self-evolving agent skills},
  author={Yang, Yifan and Gong, Ziyang and Huang, Weiquan and Yang, Qihao and Zhou, Ziwei and Huang, Zisu and Li, Yan and Gao, Xuemei and Dai, Qi and Liu, Bei and others},
  journal={arXiv preprint arXiv:2605.23904},
  year={2026}
}

@article{ni2026trace2skill,
  title={Trace2skill: Distill trajectory-local lessons into transferable agent skills},
  author={Ni, Jingwei and Liu, Yihao and Liu, Xinpeng and Sun, Yutao and Zhou, Mengyu and Cheng, Pengyu and Wang, Dexin and Zhao, Erchao and Jiang, Xiaoxi and Jiang, Guanjun},
  journal={arXiv preprint arXiv:2603.25158},
  year={2026}
}

@article{wang2023voyager,
  title={Voyager: An open-ended embodied agent with large language models},
  author={Wang, Guanzhi and Xie, Yuqi and Jiang, Yunfan and Mandlekar, Ajay and Xiao, Chaowei and Zhu, Yuke and Fan, Linxi and Anandkumar, Anima},
  journal={arXiv preprint arXiv:2305.16291},
  year={2023}
}

@article{yang2025qwen3,
  title={Qwen3 Technical Report},
  author={Yang, An and Li, Anfeng and Yang, Baosong and Zhang, Beichen and Hui, Binyuan and Zheng, Bo and Yu, Bowen and others},
  journal={arXiv preprint arXiv:2505.09388},
  year={2025}
}

@article{geminiteam2025gemini25,
  title={Gemini 2.5: Pushing the Frontier with Advanced Reasoning, Multimodality, Long Context, and Next Generation Agentic Capabilities},
  author={{Gemini Team}},
  journal={arXiv preprint arXiv:2507.06261},
  year={2025}
}

@article{yang2024sweagent,
  title={{SWE-agent}: Agent--Computer Interfaces Enable Automated Software Engineering},
  author={Yang, John and Jimenez, Carlos E. and Wettig, Alexander and Lieret, Kilian and Yao, Shunyu and Narasimhan, Karthik and Press, Ofir},
  journal={arXiv preprint arXiv:2405.15793},
  year={2024}
}

@article{feng2026group,
  title={Group-in-group policy optimization for llm agent training},
  author={Feng, Lang and Xue, Zhenghai and Liu, Tingcong and An, Bo},
  journal={Advances in Neural Information Processing Systems},
  volume={38},
  pages={46375--46408},
  year={2026}
}

@inproceedings{li2026turn,
  title={Turn-ppo: Turn-level advantage estimation with ppo for improved multi-turn rl in agentic llms},
  author={Li, Junbo and Zhou, Peng and Meng, Rui and Vadera, Meet P and Li, Lihong and Li, Yang},
  booktitle={Findings of the Association for Computational Linguistics: EACL 2026},
  pages={6227--6243},
  year={2026}
}

@article{hou2026single,
  title={Single-Rollout Asynchronous Optimization for Agentic Reinforcement Learning},
  author={Hou, Zhenyu and Li, Yujiang and Tang, Jie and Dong, Yuxiao},
  journal={arXiv preprint arXiv:2607.07508},
  year={2026}
}

@article{sun2026essam,
  title={ESSAM: A Novel Competitive Evolution Strategies Approach to Reinforcement Learning for Memory Efficient LLMs Fine-Tuning},
  author={Sun, Zhishen and Dang, Sizhe and Dai, Guang and Ye, Haishan},
  journal={arXiv preprint arXiv:2602.01003},
  year={2026}
}

@article{song2026survey,
  title={A survey of on-policy distillation for large language models},
  author={Song, Mingyang and Zheng, Mao},
  journal={arXiv preprint arXiv:2604.00626},
  year={2026}
}

@article{sarkar2025evolution,
  title={Evolution strategies at the hyperscale},
  author={Sarkar, Bidipta and Fellows, Mattie and Duque, Juan Agustin and Letcher, Alistair and Villares, Antonio Le{\'o}n and Sims, Anya and Wibault, Clarisse and Samsonov, Dmitry and Cope, Dylan and Liesen, Jarek and others},
  journal={arXiv preprint arXiv:2511.16652},
  year={2025}
}

@article{kim2026longhorizon,
  title={On Training Large Language Models for Long-Horizon Tasks: An Empirical Study of Horizon Length},
  author={Kim, Sunghwan and Cho, Junhee and Kwak, Beong-woo and Kwon, Taeyoon and Wang, Liang and Yang, Nan and Zhang, Xingxing and Wei, Furu and Yeo, Jinyoung},
  journal={arXiv preprint arXiv:2605.02572},
  year={2026}
}

@article{du2026survey,
  title={A survey on the optimization of large language model-based agents},
  author={Du, Shangheng and Zhao, Jiabao and Shi, Jinxin and Xie, Zhentao and Jiang, Xin and Bai, Yanhong and He, Liang},
  journal={ACM Computing Surveys},
  volume={58},
  number={9},
  pages={1--37},
  year={2026},
  publisher={ACM New York, NY}
}

@article{ma2024spreadsheetbench,
  title={Spreadsheetbench: Towards challenging real world spreadsheet manipulation},
  author={Ma, Zeyao and Zhang, Bohan and Zhang, Jing and Yu, Jifan and Zhang, Xiaokang and Zhang, Xiaohan and Luo, Sijia and Wang, Xi and Tang, Jie},
  journal={Advances in Neural Information Processing Systems},
  volume={37},
  pages={94871--94908},
  year={2024}
}

@article{wang2024openhands,
  title={{OpenHands}: An Open Platform for {AI} Software Developers as Generalist Agents},
  author={Wang, Xingyao and Li, Boxuan and Song, Yufan and Xu, Frank F. and Tang, Xiangru and Zhuge, Mingchen and Pan, Jiayi and others},
  journal={arXiv preprint arXiv:2407.16741},
  year={2024}
}

@inproceedings{qiu2026esscale,
  title={Evolution Strategies at Scale: {LLM} Fine-Tuning Beyond Reinforcement Learning},
  author={Qiu, Xin and Gan, Yulu and Hayes, Conor F. and Liang, Qiyao and Xu, Yinggan and Dailey, Roberto and Meyerson, Elliot and Hodjat, Babak and Miikkulainen, Risto},
  booktitle={International Conference on Machine Learning},
  year={2026}
}

\clearpage
\beginappendix

\section*{Appendix Contents}
\vspace{15pt}
\label{app:contents}

\setlength{\parindent}{0pt}
\setlength{\parskip}{2pt}

\newcommand{\appentry}[2]{%
    \noindent
    \hyperref[#1]{#2}
    \dotfill
    \pageref{#1}\par
}

\newcommand{\appsubentry}[2]{%
    \noindent\hspace*{1.5em}
    \hyperref[#1]{#2}
    \dotfill
    \pageref{#1}\par
}

\appentry{app:limitation-futurework}
    {\textbf{A.\ Limitations \& Future Work}}
\appsubentry{app:limitation}
    {A.1.\ Limitations}
\appsubentry{app:futurework}
    {A.2.\ Future Work}

\vspace{0.35em}

\appentry{app:related-work}
    {\textbf{B.\ Additional Related Work}}
\appsubentry{app:related-work-rl}
    {B.1.\ Agentic Reinforcement Learning}
\appsubentry{app:related-work-es}
    {B.2.\ Evolution Strategies}
\appsubentry{app:related-work-skills}
    {B.3.\ Skill and Memory-Based Self-Improvement}

\vspace{0.35em}

\appentry{app:theory}
    {\textbf{C.\ Theoretical Analysis and Computational Accounting}}
\appsubentry{app:es-scalar-gradient}
    {C.1.\ Derivation of the Scalar-Score ES Gradient}
\appsubentry{app:es-smoothing-bias}
    {C.2.\ Proof of the Gaussian-Smoothing Bias Expansion}
\appsubentry{app:horizon-variance}
    {C.3.\ Long-Horizon Credit Assignment}
\appsubentry{app:ahd-details}
    {C.4.\ AHD as Heuristic-Space Optimization}
\appsubentry{app:flops-accounting}
    {C.5.\ Training FLOPs Calculation}

\vspace{0.35em}

\appentry{app:settings}
    {\textbf{D.\ Experimental Settings and Additional Results}}
\appsubentry{app:sudoku-setting}
    {D.1.\ Sudoku}
\appsubentry{app:math-docvqa-implementation}
    {D.2.\ Math Reasoning}
\appsubentry{app:docvqa-implementation}
    {D.3.\ DocVQA}
\appsubentry{app:webarena-details}
    {D.4.\ WebArena-Lite}
\appsubentry{app:ahd-setting}
    {D.5.\ Automatic Heuristic Design}
\appsubentry{app:hyperparameters}
    {D.6.\ Cross-Setting Hyperparameter Summary}

\vspace{0.35em}

\appentry{app:environment-prompts}
    {\textbf{E.\ Prompts and Skills}}
\appsubentry{app:prompt-sudoku}
    {E.1.\ Sudoku}
\appsubentry{app:prompt-math}
    {E.2.\ Math Reasoning}
\appsubentry{app:prompt-docvqa}
    {E.3.\ DocVQA}
\appsubentry{app:prompt-webarena}
    {E.4.\ WebArena-Lite}
\appsubentry{app:prompt-ahd}
    {E.5.\ Automatic Heuristic Design}

\vspace{0.35em}

\appentry{app:license}
    {\textbf{F.\ Licensing and External Components}}
\appsubentry{app:license-scope}
    {F.1.\ Licensing Scope}
\appsubentry{app:upstream-sources}
    {F.2.\ Upstream Repositories Used by the Released Workflow}

\clearpage

\section{Limitations \& Future Work}
\label{app:limitation-futurework}

\subsection{Limitations}
\label{app:limitation}

\paragraph{Introducing new hyperparameters.}
Compared to Agentic RL, Agentic ESOpt introduces hyperparameters for the perturbation radius $\sigma$. Their optimal values can depend on LLMs, reward distribution, and environment, which may introduce additional tuning cost compared with using a fixed training recipe. However, as summarized in Appendix \ref{app:hyperparameters}, we use relatively consistent configurations across tasks and LLMs, taking $\sigma_0\approx1e-3$ and $\alpha\approx5e-4$ across 5 experiments. This shows that this set of parameters is highly generic. 
We will take the automatic schedule for $\sigma$ as future work.

\paragraph{Concerns in Expensive Evaluation Scenarios.}
Agentic ESOpt trades backpropagation cost for a larger number of independent environment evaluations: under matched model FLOPs, it can afford more trajectories because each trajectory requires only a forward pass. This trade-off may become less favorable when environment evaluation itself is \textit{extremely} expensive, in which case rollout cost rather than model computation can dominate the total budget.

\paragraph{Continual learning remains unclear.}
The current experiments establish in-setting adaptation. As mentioned in \citet{hoy2026matching}, unlike GRPO, which usually does sparse updates, ES will lead to a random walk in directions irrelevant to the optimization objective, raising concerns about continual learning. However, although the ES update is dense in the strict sense, its magnitude is still highly concentrated. As shown in Table \ref{tab:es_update_magnitude}, Agentic ESOpt updates can still maintain sparsity within a certain threshold: $96.26\%$ of the final parameter updates remain within the perturbation scale $\sigma$, and $99.42\%$ are smaller than $2.0\times10^{-3}$.

\begin{table}[H]
\centering
\caption{Distribution of parameter-update magnitudes after Agentic ESOpt for WebArena on Qwen3.5-27B with perturbation scale $\sigma_t=1.5\times10^{-3}$.}
\label{tab:es_update_magnitude}
\small
\begin{tabular}{lc}
\toprule
Update-magnitude threshold & Fraction of parameters \\
\midrule
$|\Delta\theta| \le 1.0\times10^{-3}$ & $83.68\%$ \\
$|\Delta\theta| \le 1.5\times10^{-3}\;(=\sigma)$ & $96.26\%$ \\
$|\Delta\theta| \le 2.0\times10^{-3}$ & $99.42\%$ \\
\bottomrule
\end{tabular}
\end{table}

\subsection{Future Work}
\label{app:futurework}

\paragraph{Scaling to advanced LLMs and population scaling laws.}
The inference-level memory requirement of Agentic ESOpt provides a direct path toward full-parameter adaptation of substantially larger LLM agents. Our 4B/9B results further suggest that stronger backbones may require fewer perturbation directions, motivating a systematic study of population scaling with model capability. Moreover, as future work, scaling to advanced LLMs \citep{team2026kimi,guo2025deepseek,geminiteam2025gemini25} may also enable online adaptation in industrial scenarios involving proprietary tools, APIs, and complex workflows.

\paragraph{Quantized ES optimization.}
Quantization could further extend the scalability of Agentic ESOpt, but applying dense parameter perturbations directly to quantized weights requires scale-aware noise generation and numerically stable updates \citep{sarkar2025evolution}. Developing ES-specific infra for quantization-compatible perturbation and seed-replay mechanisms is therefore an important systems direction for large-scale ES fine-tuning.


\paragraph{Tighter skill--parameter co-evolution.}
Our experiments demonstrate both shared-rollout skill distillation with Trace2Skill and online parameter adaptation inside EoH. A natural next step is fully coupled multi-step skill optimization \citep{yang2026skillopt,shen2026skillopt} in which the external context $c_t$ and model parameters $\theta_t$ evolve on comparable timescales, allowing skill- and parameter-space updates to continually reshape the data distribution seen by one another.

\newpage
\section{Additional Related Work}
\label{app:related-work}

\subsection{Agentic Reinforcement Learning}
\label{app:related-work-rl}
Reinforcement learning provides the standard route for optimizing model parameters from verifiable task feedback \citep{shao2024deepseekmath,liu2025understanding,yu2025dapo,zheng2025soft,tajwar2026maximum}. These methods have demonstrated strong performance on single-turn LLM reasoning.

The recent agentic scenarios bring new challenges about long-context and long-horizon credit assignment \citep{kim2026longhorizon,liu2025gem}, and there are studies that try to apply RL techniques to agentic scenarios \citep{qi2025webrl,feng2026group,wang2026milestone}. Some work highlights better rollout in a GRPO-style baseline \citep{feng2026group,he2026hierarchy}, and there are also works training critic network and use PPO for better credit assignment \citep{hou2026single}. These methods can improve the credit assignment ability of RL, but we believe that considering Agentic ESOpt should be seen as a more direct solution. Gradient-based agentic RL usually requires a training stack that stores rollouts, estimates token-level policy gradients, manages reference models or KL constraints, and consumes substantial memory. Agentic ESOpt maintains the same high-level goal of parameter adaptation, but uses scalar black-box fitness so that the environment trajectory does not need to be differentiable or retained for backpropagation.

\subsection{Evolution Strategies for single-turn LLM Fine-Tuning.}
\label{app:related-work-es}

As advanced LLMs are equipped with larger and larger parameters, fine-tuning them on devices with moderate scales becomes unaffordable, and ES have emerged as efficient gradient-free optimizers for fine-tuning single-turn LLMs. \citet{qiu2026esscale} demonstrated that ES can scale to full-parameter LLM fine-tuning, matching or exceeding GRPO in sample efficiency and training stability. This momentum has driven ES-based methods into pre-training \citep{liu2025ea4llm}, few-shot adaptation \citep{gan2026neural,korotyshova2025essa,gu2026hyper,fu2026reasoning}, and memory-efficient tuning via sharpness-aware mechanisms~\citep{sun2026essam}. In these methods, ES achieves higher GPU memory efficiency but lower performance than RL methods. However, we argue that the structural advantages of ES are actually pronounced in fine-tuning long-horizon agents, where Agentic ESOpt can become preferable to Agentic RL, instead of merely serving as a cheaper alternative.

As a similar category of gradient-free LLM fine-tuning, zeroth-order methods \citep{malladi2023fine,zhao2025second} focus on SFT, making them hard to apply to single-turn and multi-turn reasoning scenarios.

\subsection{Skill and Memory-Based Self-Improvement and Test-Time Compute}
\label{app:related-work-skills}

A broad class of self-improving agents improves behavior by optimizing external, non-parametric components while keeping the underlying LLM fixed. Reflexion converts failed trajectories into verbal feedback that conditions later trials \citep{shinn2023reflexion}; Voyager accumulates executable skills for open-ended embodied exploration \citep{wang2023voyager}; and recent memory systems organize and retrieve past experience over long-horizon interactions \citep{xu2025mem,wu2025human}. More recent methods such as SkillOpt and Trace2Skill make this optimization explicit by treating the skill document itself as an optimizable object \citep{yang2026skillopt,ni2026trace2skill}.

Test-time compute \citep{snell2024scaling,zhu2025scaling} methods follow a related principle: they improve task performance through additional search \citep{yao2023tree}, sampling, reflection \citep{madaan2023self}, or evolution \citep{liu2024evolution} at inference time, while typically leaving the model parameters unchanged. These approaches are lightweight and flexible, but their optimization remains constrained by the behaviors accessible to the frozen policy.

Agentic ESOpt complements these methods by introducing parameter adaptation as an additional optimization dimension. Its black-box update can reuse the same trajectory-level feedback already collected for skill, memory, or test-time search, allowing external-space optimization and parameter-space optimization to proceed together. This enables Agentic ESOpt to strengthen existing self-improvement and test-time compute pipelines without replacing their original search or skill-update mechanisms.

\newpage
\section{Theoretical Analysis and Computational Accounting}
\label{app:theory}

\subsection{Derivation of the Scalar-Score ES Gradient}
\label{app:es-scalar-gradient}

Let $d$ be the number of LLM parameters and let
\[
    q_\sigma(\vartheta\mid\theta)
    =
    \frac{1}{(2\pi\sigma^2)^{d/2}}
    \exp\!\left(
        -\frac{\lVert\vartheta-\theta\rVert_2^2}{2\sigma^2}
    \right)
\]
denotes the density of a Gaussian perturbation centered at $\theta$. Equivalently,
$\vartheta=\theta+\sigma\boldsymbol{\epsilon}$ with $\boldsymbol{\epsilon}\sim\mathcal N(0,I)$. The smoothed objective can then be written as
\[
    J_\sigma(\theta;c)
    =
    \int J(\vartheta;c)q_\sigma(\vartheta\mid\theta)\,d\vartheta.
\]
Assuming differentiation and integration can be interchanged, the log-derivative identity gives
\begin{align}
    \nabla_\theta J_\sigma(\theta;c)
    &=
    \int J(\vartheta;c)
    \nabla_\theta q_\sigma(\vartheta\mid\theta)\,d\vartheta \\
    &=
    \mathbb E_{\vartheta\sim q_\sigma(\cdot\mid\theta)}
    \left[
        J(\vartheta;c)
        \nabla_\theta\log q_\sigma(\vartheta\mid\theta)
    \right].
\end{align}
For the Gaussian density,
\[
    \nabla_\theta\log q_\sigma(\vartheta\mid\theta)
    =
    \frac{\vartheta-\theta}{\sigma^2}
    =
    \frac{\boldsymbol{\epsilon}}{\sigma},
\]
and hence
\[
    \nabla_\theta J_\sigma(\theta;c)
    =
    \frac{1}{\sigma}
    \mathbb E_{\boldsymbol{\epsilon}}
    \left[J(\theta+\sigma\boldsymbol{\epsilon};c)\boldsymbol{\epsilon}\right].
\]

Now expand the inner objective using its trajectory definition:
\[
    J(\theta+\sigma\boldsymbol{\epsilon};c)
    =
    \mathbb E_{\boldsymbol{\tau}\sim\pi_{\theta+\sigma\boldsymbol{\epsilon}}(\cdot\mid c)}
    \left[R(\boldsymbol{\tau})\right].
\]
The law of total expectation therefore yields
\[
    \nabla_\theta J_\sigma(\theta;c)
    =
    \frac{1}{\sigma}
    \mathbb E_{\substack{
        \boldsymbol{\epsilon}\sim\mathcal N(0,I)\\
        \boldsymbol{\tau}\sim\pi_{\theta+\sigma\boldsymbol{\epsilon}}(\cdot\mid c)
    }}
    \left[R(\boldsymbol{\tau})\boldsymbol{\epsilon}\right].
\]
Thus, for independent perturbations $\boldsymbol{\epsilon}_i$ and conditional rollouts
$\boldsymbol{\tau}_i\sim\pi_{\theta+\sigma\boldsymbol{\epsilon}_i}(\cdot\mid c)$, the canonical Monte Carlo estimator is
\[
    \widehat g_{\mathrm{ES}}
    =
    \frac{1}{G\sigma}
    \sum_{i=1}^{G}R(\boldsymbol{\tau}_i)\boldsymbol{\epsilon}_i,
    \qquad
    \mathbb E[\widehat g_{\mathrm{ES}}]
    =
    \nabla_\theta J_\sigma(\theta;c).
\]
Because $\mathbb E[\boldsymbol{\epsilon}]=0$, any perturbation-independent scalar baseline $b$ may replace $R(\boldsymbol{\tau}_i)$ by $R(\boldsymbol{\tau}_i)-b$ without changing this expectation. Crucially, the estimator uses the environment only to produce the scalar $R(\boldsymbol{\tau}_i)$; no derivative of the sampled actions, environment transitions, or reward function is required. The within-population reward standardization used by Agentic ESOpt is the practical finite-sample variant described in the main text.

\subsection{Proof of the Gaussian-Smoothing Bias Expansion}
\label{app:es-smoothing-bias}

This subsection proves \Cref{lem:es-bias}. We suppress the fixed external context $c$ for readability and write $J(\theta)$ for the corresponding task objective.

\begin{proof}[Proof of \Cref{lem:es-bias}]
Define the Gaussian-smoothed objective
\[
    J_\sigma(\theta)
    =
    \mathbb E_{\boldsymbol{\epsilon}\sim\mathcal N(0,I)}
    \left[J(\theta+\sigma\boldsymbol{\epsilon})\right].
\]
By the Gaussian score-function identity, or equivalently integration by parts,
\[
    \nabla_\theta J_\sigma(\theta)
    =
    \frac{1}{\sigma}
    \mathbb E_{\boldsymbol{\epsilon}}
    \left[J(\theta+\sigma\boldsymbol{\epsilon})\boldsymbol{\epsilon}\right].
\]
Thus, the population ES estimator is unbiased for $\nabla_\theta J_\sigma(\theta)$.

To compare the smoothed and original objectives, expand $J(\theta+\sigma\boldsymbol{\epsilon})$ around $\theta$. Odd moments of a centered Gaussian vanish, and $\mathbb E[\boldsymbol{\epsilon}\boldsymbol{\epsilon}^{\top}]=I$, giving
\[
    J_\sigma(\theta)
    =
    J(\theta)
    +
    \frac{\sigma^2}{2}
    \operatorname{Tr}\!\left(\nabla_\theta^2 J(\theta)\right)
    +
    O(\sigma^4)
    =
    J(\theta)
    +
    \frac{\sigma^2}{2}\Delta_\theta J(\theta)
    +
    O(\sigma^4).
\]
Differentiating the expansion yields
\[
    \nabla_\theta J_\sigma(\theta)
    =
    \nabla_\theta J(\theta)
    +
    \frac{\sigma^2}{2}
    \nabla_\theta\Delta_\theta J(\theta)
    +
    O(\sigma^4).
\]
Therefore, the leading smoothing bias relative to $\nabla_\theta J(\theta)$ is proportional to $\sigma^2$, as stated in \Cref{lem:es-bias}.
\end{proof}

\subsection{Long-Horizon Credit Assignment}
\label{app:horizon-variance}

Let $u_t=\nabla_\theta\log\pi_\theta(a_t\mid s_t)$ be the policy score at turn $t$, and let $\widetilde R=R-b$ denote the baseline-centered terminal return. The trajectory-level policy-gradient estimator is
\[
    \widehat g_{\mathrm{PG}}
    =
    \widetilde R\sum_{t=1}^{H}u_t.
\]
Assume, as a scaling approximation, that (i) the return has weak correlation with any individual action score, (ii) score terms at different turns are approximately uncorrelated, and (iii) their marginal covariance is comparable across turns, $\operatorname{Cov}(u_t)\approx\Sigma_u$. Then
\begin{align}
    \operatorname{Cov}(\widehat g_{\mathrm{PG}})
    &\approx
    \operatorname{Var}(\widetilde R)
    \operatorname{Cov}\!\left(\sum_{t=1}^{H}u_t\right) \\
    &=
    \operatorname{Var}(\widetilde R)
    \left[
        \sum_{t=1}^{H}\operatorname{Cov}(u_t)
        +
        \sum_{s\neq t}\operatorname{Cov}(u_s,u_t)
    \right] \\
    &\approx
    H\,\operatorname{Var}(\widetilde R)\Sigma_u.
\end{align}
Thus, under these assumptions, the policy-score contribution to the estimator variance grows approximately linearly with the realized horizon $H$.

For Agentic ESOpt, one perturbation $\boldsymbol{\epsilon}\sim\mathcal N(0,I)$ is sampled for the complete trajectory, giving
\[
    \widehat g_{\mathrm{ES}}
    =
    \widetilde R(\theta+\sigma\boldsymbol{\epsilon})
    \frac{\boldsymbol{\epsilon}}{\sigma}.
\]
Under the analogous weak-correlation approximation,
\[
    \operatorname{Cov}(\widehat g_{\mathrm{ES}})
    \approx
    \operatorname{Var}(\widetilde R)
    \operatorname{Cov}\!\left(
        \frac{\boldsymbol{\epsilon}}{\sigma}
    \right),
\]
whose parameter-score factor contains no sum over turns. The comparison therefore isolates an explicit horizon-dependent source of variance in action-space policy gradients that is absent from the ES parameter score.

\paragraph{Scope of the scaling comparison.}
The above argument concerns the \textbf{horizon-dependent score structure} of the two estimators. Agentic ESOpt can still become harder to optimize as the horizon grows because the return distribution may become sparser or less discriminative. Its estimator quality may also depend on the parameter dimension $d$, perturbation radius $\sigma$, population size $G$, and the local geometry of the smoothed objective. Accordingly, the analysis does not assert that the total variance of ES is independent of $H$, nor that ES universally dominates policy gradients. The comparative prediction is that, when other sources of difficulty are approximately comparable, increasing the effective horizon introduces an additional accumulation of action-score terms for policy gradients but not for the ES parameter score. This follows the scaling perspective of \citet[Sec.~3.1]{salimans2017evolution}.

\paragraph{From realized horizon $H$ to minimum successful horizon $H^*$.}
The theoretical quantity above is the realized trajectory length $H$. In the Sudoku experiments, we instead control
\[
H^*(x)
=
\min_{\boldsymbol{\tau}:R(\boldsymbol{\tau})=1}
|\boldsymbol{\tau}|,
\]
which gives a task-dependent lower bound on the horizon of every successful trajectory. Because each valid Sudoku action fills at most one masked cell, masking 5, 10, and 15 cells yields $H^*=5$, $10$, and $15$, respectively. Failed, invalid, or unproductive actions can make the realized horizon substantially larger than $H^*$.

Increasing $H^*$ in this natural task family also changes aspects of task difficulty, such as the amount of missing information and the probability of completing all required actions correctly. The Sudoku study is therefore intended to test the practical prediction that the relative behavior of parameter-space and action-space attribution changes as the minimum interaction requirement grows, instead of identifying a pure delay effect in an otherwise identical MDP. The executed-turn analysis in \Cref{fig:sudoku-turn-diagnostics} complements $H^*$ by directly measuring how the realized horizon evolves during optimization.

\subsection{AHD as Heuristic-Space Optimization}
\label{app:ahd-details}

Automatic heuristic design is not a language-generation task in the usual sense: the generated text matters only through the algorithmic behavior it induces. We therefore view AHD as optimization over a heuristic space \citep{zheng2025mctsahd}. Let $\mathcal{X}$ be a distribution over problem instances and let $h \in \mathcal{H}$ denote a heuristic sampled from the LLM-induced proposal distribution. A solver $A_h$ uses $h$ to construct a solution $y=A_h(x)$ for instance $x$. The idealized objective is
\[
    \min_{h \in \mathcal{H}}
    \mathbb{E}_{x \sim \mathcal{X}}
    \left[
        S(y,x)
    \right],
    \qquad
    y=A_h(x),
\]
where $S$ is a signed score chosen so that larger is better. This paper covers 5 NP-hard combinatorial optimization problems, including the Traveling Salesman Problem (TSP), which seeks the shortest tour visiting all nodes exactly once; the Capacitated Vehicle Routing Problem (CVRP), which minimizes the total routing cost for capacity-constrained vehicles serving customer demands; the 0-1 Knapsack Problem (KP), which maximizes the total value of selected items under a capacity constraint; the Bin-Packing Problem (BPP), which minimizes the number of fixed-capacity bins required to pack all items; and the Admissible Set Problem (ASP), which seeks the largest subset satisfying a predefined set of feasibility constraints.

For minimization problems such as TSP, CVRP, and BPP, $S$ is the negative objective value; for maximization problems such as KP and ASP, $S$ is the original objective value. The relevant object is therefore not the textual form of $h$, but the induced mapping from problem states to algorithmic decisions.

In constructive AHD, $A_h$ is a sequential decision rule. At a partial solution state $s_t$, the heuristic induces a preference over feasible actions,
\[
    a_t
    \in
    \arg\max_{a \in \mathcal{A}(s_t)}
    h(s_t,a),
    \qquad
    s_{t+1}=T(s_t,a_t),
\]
until a complete solution is obtained. TSP evaluates the length of the constructed tour, KP evaluates the value of the selected item set under a capacity constraint, and ASP evaluates the cardinality of the constructed admissible set. These tasks test whether Agentic ESOpt can alter the LLM proposal distribution toward heuristics whose local decisions accumulate into better global solutions.

In ACO-style AHD, $A_h$ is stochastic. The heuristic defines an energy or desirability function over solution components, which biases a sampling distribution of the form
\[
    p_h(a_t \mid s_t)
    \propto
    \exp\left(\beta h(s_t,a_t)\right)
    \cdot
    \Phi_t(s_t,a_t),
\]
where $\Phi_t$ denotes the non-learned search state, such as pheromone or feasibility terms, and $\beta$ controls the strength of the heuristic bias. The objective is the expected quality of solutions sampled by this stochastic solver. Agentic ESOpt + EoH preserves this outer optimization problem and changes only the model-induced distribution over $h$.

\newpage
\subsection{Training FLOPs Calculation}
\label{app:flops-accounting}

We follow the FLOPs accounting of \citet[Appendix~E.1]{gan2026neural}.
For a model with $P$ parameters and a trajectory containing $\bar L$ model-processed tokens,
a forward pass requires approximately $2P\bar L$ FLOPs, while a backward pass requires approximately $4P\bar L$ FLOPs.

Let $T$, $B$, and $G$ denote the number of training iterations, prompt batch size, and number of rollouts or perturbation directions per prompt, respectively. The method-specific training costs are

\begin{align}
\mathrm{FLOPs}_{\mathrm{ES}}
&=
T_{\mathrm{ES}}B_{\mathrm{ES}}G_{\mathrm{ES}}
\underbrace{(2)}_{\text{policy forward}}
P\bar L_{\mathrm{ES}}
\nonumber\\
&=
2T_{\mathrm{ES}}B_{\mathrm{ES}}G_{\mathrm{ES}}P\bar L_{\mathrm{ES}},
\label{eq:flops-es}
\\[3pt]
\mathrm{FLOPs}_{\mathrm{GRPO}}
&=
T_{\mathrm{GRPO}}B_{\mathrm{GRPO}}G_{\mathrm{GRPO}}
\underbrace{(2+2+4)}_{\substack{\text{policy forward}\\+\;\text{reference forward}\\+\;\text{policy backward}}}
P\bar L_{\mathrm{GRPO}}
\nonumber\\
&=
8T_{\mathrm{GRPO}}B_{\mathrm{GRPO}}G_{\mathrm{GRPO}}P\bar L_{\mathrm{GRPO}},
\label{eq:flops-grpo}
\\[3pt]
\mathrm{FLOPs}_{\mathrm{PPO}}
&=
T_{\mathrm{PPO}}B_{\mathrm{PPO}}G_{\mathrm{PPO}}
\underbrace{(2+2+2+4+4)}_{\substack{\text{policy forward}+\text{reference forward}\\
+\;\text{critic forward}+\text{policy backward}+\text{critic backward}}}
P\bar L_{\mathrm{PPO}}
\nonumber\\
&=
14T_{\mathrm{PPO}}B_{\mathrm{PPO}}G_{\mathrm{PPO}}P\bar L_{\mathrm{PPO}}.
\label{eq:flops-ppo}
\end{align}

Thus, for trajectories of equal length, Agentic ESOpt costs one policy forward pass per sampled trajectory, \textbf{compared with approximately four forward-pass equivalents for Agentic GRPO and seven for Agentic PPO}. Environment execution is excluded throughout, since the comparison measures model-side training FLOPs.

\paragraph{Sudoku.}
In the reported Sudoku comparison, Agentic ESOpt and GRPO both use 100 update rounds with a prompt batch size of 32, while Agentic PPO does 500 update rounds. Agentic ESOpt evaluates $G_{\mathrm{ES}}=32$ perturbation directions per prompt, whereas Agentic GRPO uses $G_{\mathrm{GRPO}}=8$ rollouts. Therefore, assuming equal trajectory lengths,
\[
\frac{\mathrm{FLOPs}_{\mathrm{ES}}}
     {\mathrm{FLOPs}_{\mathrm{GRPO}}}
\approx
\frac{2\times32}{8\times8}
=1.
\]
Hence, the fourfold larger ES population is exactly compensated by its fourfold lower model-side cost per trajectory. Agentic PPO is capped at 500 training steps to remain within a comparable compute budget. For the measured values in \Cref{tab:sudoku-efficiency}, we further replace the nominal trajectory length by the actual number of model-processed tokens, yielding 3.1/6.3/9.4 EFLOPs for Agentic ESOpt and 3.2/7.6/10.9 EFLOPs for Agentic GRPO at $H^*\in\{5,10,15\}$. The extra FLOPs of Agentic GRPO are because the algorithm's drawback in long-horizon credit assignment often results in longer trajectories.

\paragraph{Math Reasoning and DocVQA.}
For Math Reasoning, Agentic ESOpt ($G=16$) and Agentic GRPO ($G=8$) each cover the 400-example training set once. For DocVQA, both methods likewise cover 640 training examples. Under comparable trajectory lengths, their relative model-side FLOPs are therefore
\[
\frac{\mathrm{FLOPs}_{\mathrm{ES}}}
     {\mathrm{FLOPs}_{\mathrm{GRPO}}}
\approx
\frac{2\times16}{8\times8}
=
\frac{1}{2}.
\]
Agentic ESOpt therefore requires approximately half the model-side training FLOPs of the matched Agentic GRPO configuration on these two tasks.

\newpage
\section{Experimental Settings and Additional Results}
\label{app:settings}

This appendix provides the task-specific implementation details and additional diagnostics supporting the experiments in Sections~\ref{sec:sudoku-agentic}--\ref{sec:test-time-heuristic-design}. We follow the order of the main experiments: Sudoku, Math reasoning, DocVQA, WebArena-Lite, and automatic heuristic design, followed by a cross-setting summary of Agentic ESOpt hyperparameters.

\subsection{Sudoku}
\label{app:sudoku-setting}

\subsubsection{Empirical Diagnostics for Long-Horizon Scaling}
\label{app:sudoku-diagnostics}

To complement the long-horizon scaling analysis in Section~\ref{sec:sudoku-agentic} and Appendix~\ref{app:horizon-variance}, we examine three empirical diagnostics on Sudoku as the minimum successful horizon $H^*$ increases. We first visualize the local parameter-space reward landscape around the same Qwen3.5-4B checkpoint. We then relate per-turn correctness to trajectory-level success through a simple schematic calculation, and finally compare the realized horizons of Agentic GRPO and Agentic ESOpt during training.

For the landscape visualization, we select two orthogonal directions in the $d$-dimensional parameter space and evaluate a two-dimensional plane spanned by these directions around the current checkpoint. Each point in the heatmap corresponds to a perturbed model in that slice, and its value is the Gaussian-smoothed average terminal reward measured on the Sudoku batch. Since Sudoku provides only binary terminal feedback, this value can be interpreted as a smoothed local success score. The theoretical point is not that longer horizons leave the reward landscape unchanged, but that ES does not introduce an additional explicit accumulation over turns in its parameter-score factor.

\begin{figure}[H]
\centering
\subfigure[Minimum successful horizon $H^*=5$]{
  \includegraphics[width=0.32\linewidth]{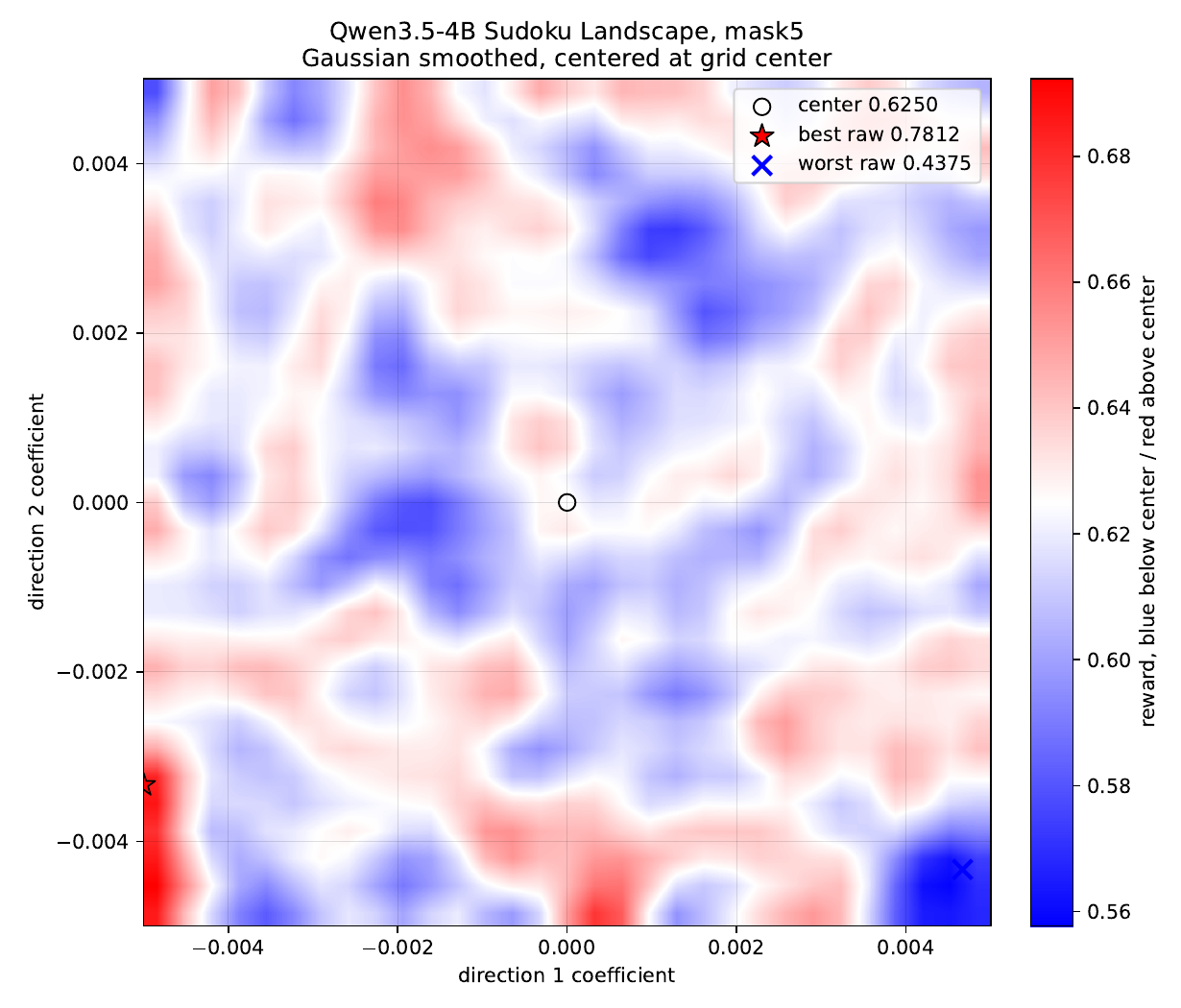}
}%
\subfigure[Minimum successful horizon $H^*=10$]{
  \includegraphics[width=0.32\linewidth]{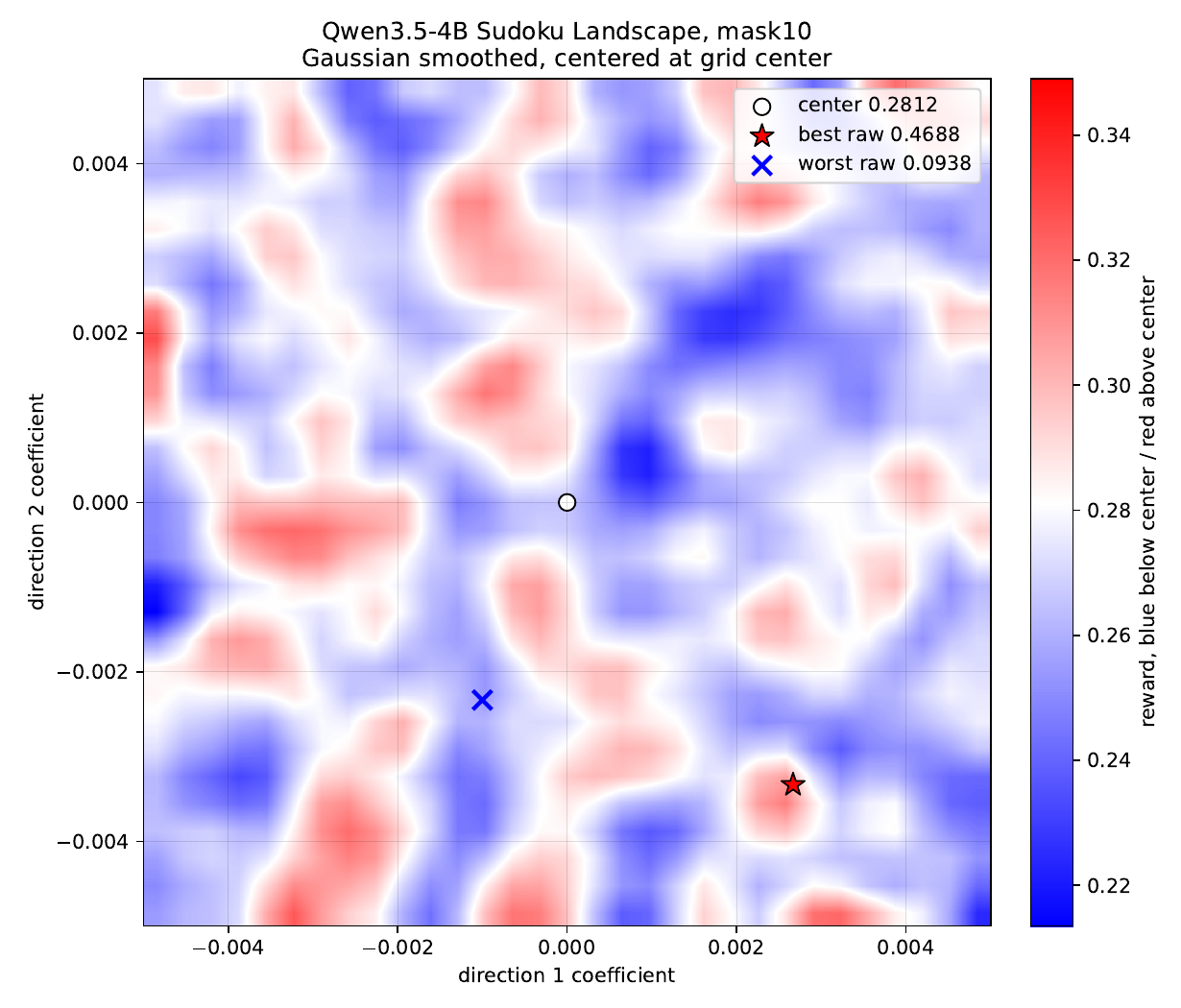}
}%
\subfigure[Minimum successful horizon $H^*=15$]{
  \includegraphics[width=0.32\linewidth]{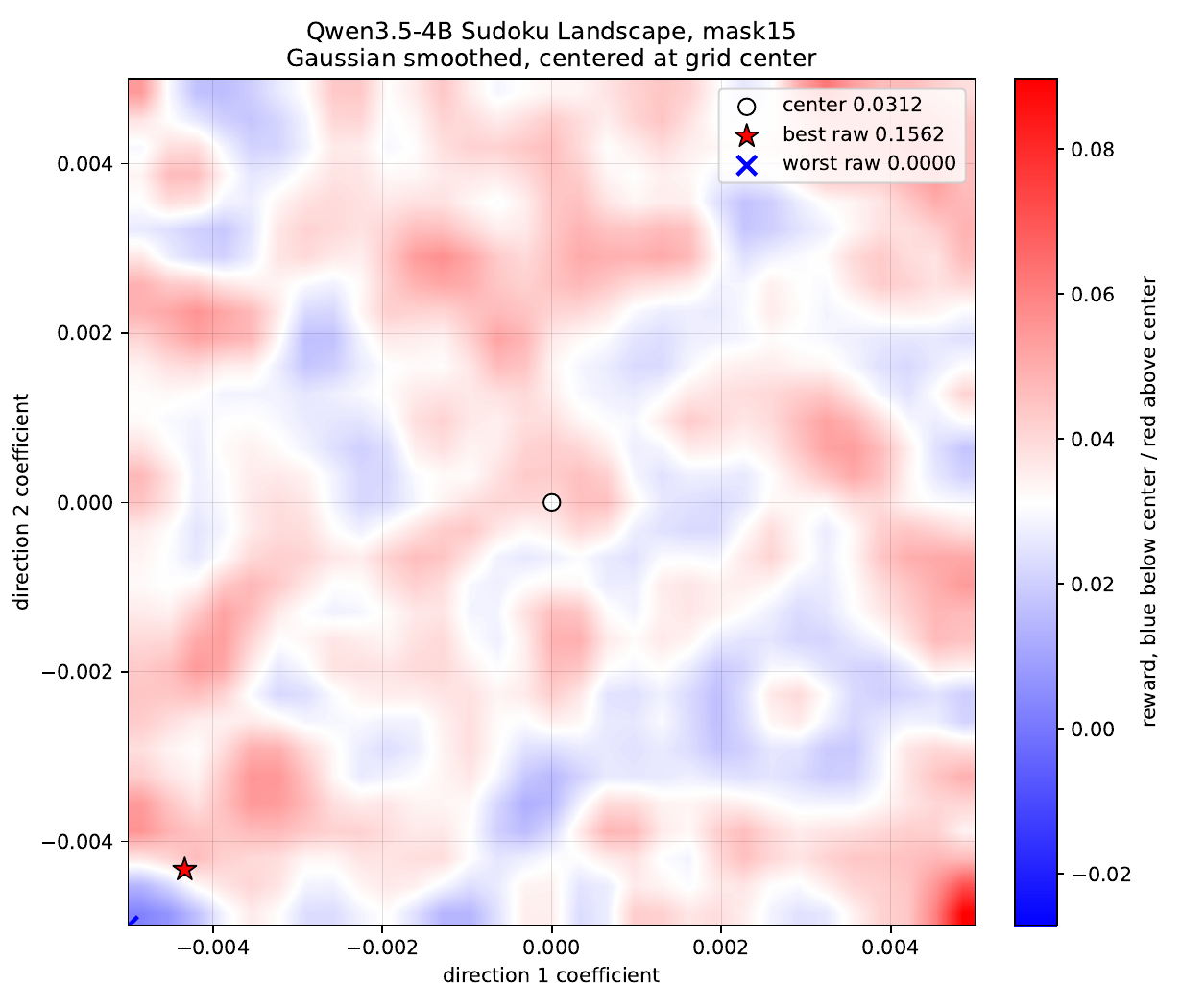}
}
\caption{
Two-dimensional Gaussian-smoothed reward landscapes around the Qwen3.5-4B checkpoint for Sudoku tasks with $H^*\in\{5,10,15\}$. Each panel marks the unperturbed checkpoint and the best and worst raw perturbations in the displayed slice. The heatmap value is the smoothed average terminal reward in the local neighborhood. As $H^*$ increases, the absolute reward level and contrast decrease, reflecting increasingly sparse trajectory feedback, while local parameter-space variation remains observable.
}
\label{fig:sudoku-landscapes}
\end{figure}

As shown in Figure~\ref{fig:sudoku-landscapes}, increasing $H^*$ substantially reduces the absolute reward level and reward contrast, making informative terminal feedback increasingly scarce. Nevertheless, useful variation remains visible in the local parameter neighborhood even at $H^*=15$. This is consistent with our theoretical characterization: longer horizons still make the return signal harder, but Agentic ESOpt does not incur an additional horizon-wise accumulation in its parameter-score factor.

Following the Sudoku step-accuracy definition of \citet{kim2026longhorizon}, a turn is correct when it fills a masked cell with its unique target value. To connect per-turn correctness with trajectory-level success, consider a minimum-length successful path with homogeneous per-turn correctness probability $p$. Its success probability is
\[
    S_{H^*} = p^{H^*}.
\]
This relation is schematic and is not fitted to the measured Sudoku trajectories; it simply illustrates how residual turn errors compound with the number of required actions.

\begin{figure}[H]
\centering
\subfigure[Schematic compounding from per-turn correctness to trajectory success.]{
  \includegraphics[width=0.45\linewidth]{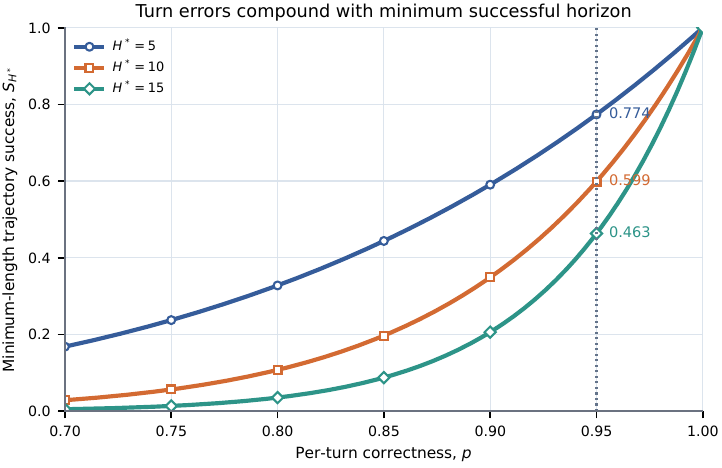}
}%
\hfill
\subfigure[Measured realized-horizon diagnostics on the $H^*=15$ setting.]{
  \includegraphics[width=0.48\linewidth]{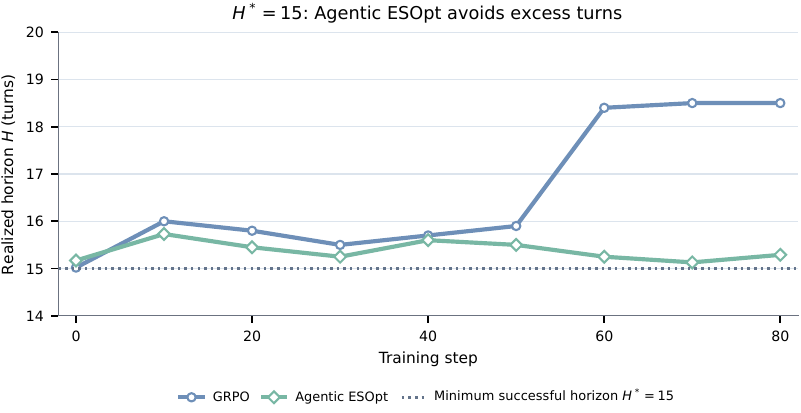}
}
\caption{
Turn-level diagnostics for long-horizon Sudoku. Left: under the schematic relation $S_{H^*}=p^{H^*}$, the same per-turn error rate produces a larger trajectory-level penalty as the minimum successful horizon grows. For example, $p=0.95$ gives $S_{H^*}=0.774$, $0.599$, and $0.463$ for $H^*=5$, $10$, and $15$, respectively. Right: realized-horizon diagnostics on the $H^*=15$ setting. Agentic GRPO training uses temperature $1$, top-$p=1$, and top-$k=-1$, while its evaluation uses temperature $0.7$, top-$p=0.8$, and top-$k=20$; Agentic ESOpt uses the latter decoder for both training and evaluation. Agentic GRPO reaches the 45-turn interaction budget by step 60, whereas Agentic ESOpt remains near the 15-turn minimum and ends at 15.41 turns.
}
\label{fig:sudoku-turn-diagnostics}
\end{figure}

Figure~\ref{fig:sudoku-turn-diagnostics} connects the theoretical argument to observable behavior. The left panel shows that even small residual turn errors become much more destructive as the required horizon grows. The right panel shows the corresponding training-time separation: Agentic GRPO progressively accumulates excess turns, whereas Agentic ESOpt stays close to the minimum-length solution path. These diagnostics are not a complete causal proof, but they are consistent with the estimator-level distinction in Section~\ref{sec:sudoku-agentic}: Agentic ESOpt evaluates one coherent parameter perturbation using the complete trajectory outcome, rather than propagating terminal feedback through an expanding sequence of action-level decisions.

\subsubsection{Training Configuration}

\paragraph{Agentic ESOpt and Vanilla ES.}
All fixed ES comparisons use Qwen3.5-4B for 100 generations. At each generation, $G=32$ full-parameter directions are evaluated on the same batch of 32 puzzles, and the resulting trajectory rewards are z-score normalized before an update with $\alpha=5\times10^{-4}$. The interaction budget is three times the mask count---15, 30, and 45 turns for $H^*=5$, $10$, and $15$---with at most 64 generated tokens per turn. Training and evaluation use temperature $0.7$, top-$p=0.8$, top-$k=20$, min-$p=0$, presence penalty $1.5$, and repetition penalty $1.0$. Evaluation is performed every 10 generations with three repeated runs.

\begin{table}[H]
\centering
\small
\setlength{\tabcolsep}{4.5pt}
\caption{Sudoku ES configurations. All profiles use $G=32$, 32 puzzles per generation, 100 generations, and $\alpha=5\times10^{-4}$.}
\label{tab:sudoku-es-hparams}
\begin{tabular}{lccc}
\toprule
Profile & Mask count & Sigma schedule & Maximum turns \\
\midrule
Vanilla ES & 5 / 10 & $1.0\times10^{-3}$ constant & 15 / 30 \\
Vanilla ES & 15 & $5.0\times10^{-4}$ constant & 45 \\
Agentic ESOpt & 5 / 10 & $1.0\times10^{-3}\!\rightarrow\!2.5\times10^{-4}$ cosine & 15 / 30 \\
Agentic ESOpt & 15 & $7.0\times10^{-4}\!\rightarrow\!5.0\times10^{-4}$ cosine & 45 \\
\bottomrule
\end{tabular}
\end{table}

\paragraph{Agentic RL.}
Both Agentic GRPO profiles run 100 rollout--update rounds with a global prompt batch of 32 and eight rollouts per prompt. They use rollout micro-batch 8, policy-training micro-batch 2, learning rate $10^{-6}$, KL coefficient $10^{-3}$, and clipping coefficient $0.2$. GRPO-A samples with temperature $0.7$, top-$p=0.8$, and top-$k=20$, while GRPO-B samples with temperature $1$, top-$p=1$, and top-$k=-1$. Both are evaluated with the Agentic ESOpt decoder before training and every 20 rounds, with three repeated evaluations. One round denotes a complete rollout batch followed by the corresponding policy update.

We implement Turn-PPO \citep{li2026turn} as the Agentic PPO baseline. Turn-PPO uses a separate Qwen3.5-4B critic initialized with a newly added value head for turn-level advantage estimation. We follow the same interaction and evaluation protocol as the other Agentic RL baselines (temperature $1$, top-$p=1$, and top-$k=-1$ in training and with temperature $0.7$, top-$p=0.8$, and top-$k=20$ for evaluation), and train at most 500 steps to align the FLOPs.

\subsubsection{Vanilla-ES Population-Sensitivity Configuration}
The 4B/9B comparison in \Cref{tab:population-scaling} uses the $H^*=15$ Sudoku environment for 100 generations and evaluates $G\in\{8,16\}$ every 10 generations. To isolate population sensitivity, all runs use Vanilla ES with a constant perturbation scale $\sigma=5\times10^{-4}$, update scale $\alpha=5\times10^{-4}$, full-parameter perturbations, population z-score normalization, and eight training puzzles per direction. Each generation therefore contains $8G$ direction--puzzle rollouts. Perturbation rollouts and evaluation use temperature $0.7$, top-$p=0.8$, and top-$k=20$, and each evaluation is repeated three times.

The environment provides a binary terminal reward: a rollout receives 1 only when the final board is a complete legal Sudoku solution preserving all given cells. Invalid formats, out-of-range values, attempts to modify given cells, and overwrites leave the board unchanged. This fixed setup makes population size the only varied ES hyperparameter in the ablation; the cosine perturbation schedule and intermediate reward shaping are not used.

\subsection{Math Reasoning}
\label{app:math-docvqa-implementation}

All method-specific configurations described below refer to training unless otherwise stated. Final evaluation is conducted under a common protocol shared by Agentic ESOpt, Agentic GRPO, Trace2Skill, their combined variants, and all other reported baselines. For Trace2Skill, we sample 16 runs for each instance and follow the instructions from \citet{ni2026trace2skill} in using at most one error trajectory per instance to distill skill.

\subsubsection{Agentic ESOpt Training Configuration}

Experiments of ReAct-style Math Reasoning are conducted on four A100 80GB GPUs. We first optimize the No Skill Qwen3.5-4B policy and record the seed of every parameter update together with the trajectories later supplied to Trace2Skill. For final evaluation, the base checkpoint is reloaded, and the update sequence is deterministically reconstructed. The No Skill and skill-conditioned evaluations therefore share an identical parameter stage and differ only in the external skill context.

\begin{table}[H]
\centering
\scriptsize
\setlength{\tabcolsep}{12pt}
\caption{Agentic ESOpt training configuration for Math reasoning.}
\label{tab:math-docvqa-es-hparams}
\begin{tabular}{lc}
\toprule
Parameter & Value \\
\midrule
Backbone & Qwen3.5-4B \\
ES generations & 25 \\
Population size $G$ & 16 \\
Cases per direction & 16 \\
Update scale $\alpha$ & $5\times10^{-4}$ \\
$\sigma$ schedule & $10^{-3}\!\rightarrow\!5\times10^{-4}$ cosine \\
Reward normalization & Population z-score \\
Maximum turns & 50 \\
Training generation limit & 4096 tokens / turn \\
\bottomrule
\end{tabular}
\end{table}

Each direction is evaluated on 16 problems, giving $16\times16=256$ direction--problem trajectories per generation. Training uses temperature $1$, top-$p=1$, top-$k=40$, min-$p=0$, presence penalty $2$, and repetition penalty $1$. The trajectory reward is exact match on the parsed final answer. Intermediate evaluation is performed every 10 generations with one sample per problem.

\subsubsection{Multi-Turn GRPO Training Configuration}

The Math Agentic GRPO baseline follows the same sampling configuration as Agentic ESOpt, with learning rate $10^{-6}$, eight rollouts per prompt, KL coefficient $10^{-3}$ using the low-variance KL form, temperature $1$, top-$p=1$, top-$k=40$, presence penalty $2$, and repetition penalty $1$.

\begin{table}[t]
\centering
\scriptsize
\setlength{\tabcolsep}{12pt}
\caption{Multi-turn Agentic GRPO training configuration for Math reasoning.}
\label{tab:math-docvqa-grpo-hparams}
\begin{tabular}{lc}
\toprule
Parameter & Value \\
\midrule
Backbone & Qwen3.5-4B \\
Training records & 400 \\
Prompt batch size & 20 \\
Rollouts per prompt & 8 \\
Training epochs & 1 \\
Training rounds & 20 \\
Total trajectories & 3{,}200 \\
Learning rate & $10^{-6}$ \\
KL coefficient & $10^{-3}$ \\
Maximum interaction & 100 user + 100 assistant turns \\
\bottomrule
\end{tabular}
\end{table}

For full-parameter Agentic GRPO training, we appropriately compress the 4096-token per-turn generation configuration to satisfy the training-memory constraint on four A100 80GB GPUs. This adjustment applies only to the training rollout configuration and does not affect the final evaluation protocol. The 400 training records form 20 prompt batches per epoch.

\subsubsection{Common Evaluation Configuration}
All reported methods and baselines use the same final evaluation protocol. For Math reasoning, we evaluate on 100 held-out DAPO problems and 30 AIME 2026 problems, with four samples per problem. Each trajectory is allowed at most 50 interaction turns and up to 4096 generated tokens per turn. We use temperature $1$, top-$p=1$, top-$k=40$, min-$p=0$, presence penalty $2$, and repetition penalty $1$ for all methods. These evaluation settings are shared by Agentic ESOpt, Agentic GRPO, Trace2Skill, their combined variants, and all other reported baselines, independent of their method-specific training configurations.

\subsection{DocVQA}
\label{app:docvqa-implementation}

As in Math reasoning, the configurations below specify the corresponding training procedures unless otherwise stated. For DocVQA, Agentic ESOpt and Agentic GRPO additionally use fully aligned interaction and generation settings during training, and all reported methods share the same final evaluation configuration. For Trace2Skill, we sample 16 runs for each instance and follow the instructions from \citet{ni2026trace2skill} in using at most one error trajectory and one correct trajectory per instance to distill skill.

\subsubsection{Agentic ESOpt Training Configuration}

Experiments of ReAct-style DocVQA are also conducted on four A100 80GB GPUs. The DocVQA experiment follows the same paired reconstruction protocol as Math: the No Skill run records update seeds and source trajectories, and final evaluation replays the identical parameter sequence before adding the post-hoc skill context.

\begin{table}[H]
\centering
\scriptsize
\setlength{\tabcolsep}{12pt}
\caption{Agentic ESOpt training configuration for DocVQA.}
\label{tab:docvqa-es-hparams}
\begin{tabular}{lc}
\toprule
Parameter & Value \\
\midrule
Backbone & Qwen3.5-4B \\
ES generations & 40 \\
Population size $G$ & 16 \\
Cases per direction & 16 \\
Update scale $\alpha$ & $5\times10^{-4}$ \\
$\sigma$ schedule & $10^{-3}\!\rightarrow\!5\times10^{-4}$ cosine \\
Reward normalization & Population z-score \\
Maximum turns & 50 \\
Generation limit & 512 tokens / turn \\
Total generation cap & 32{,}768 tokens / trajectory \\
\bottomrule
\end{tabular}
\end{table}

Each generation contains $16\times16=256$ direction--question trajectories. Training uses temperature $1$, top-$p=1$, top-$k=40$, min-$p=0$, presence penalty $2$, and repetition penalty $1$. A valid trajectory must contain at least one parsed tool action before the final answer. Continuous ANLS is used as the training reward; threshold accuracy counts an answer as correct when ANLS is strictly greater than $0.5$. Intermediate evaluation is performed every 10 generations with one sample per question.

\subsubsection{Multi-Turn GRPO Training Configuration}

DocVQA Agentic GRPO uses the same sampling configuration and interaction budget as Agentic ESOpt, together with the same optimizer and KL settings as the Math baseline.

\begin{table}[t]
\centering
\scriptsize
\setlength{\tabcolsep}{12pt}
\caption{Multi-turn Agentic GRPO training configuration for DocVQA.}
\label{tab:docvqa-grpo-hparams}
\begin{tabular}{lc}
\toprule
Parameter & Value \\
\midrule
Backbone & Qwen3.5-4B \\
Training records & 50 \\
Prompt batch size & 4 \\
Rollouts per prompt & 8 \\
Training epochs & 15 \\
Training rounds & 180 \\
Total trajectories & 5{,}760 \\
Learning rate & $10^{-6}$ \\
KL coefficient & $10^{-3}$ \\
Maximum turns & 50 \\
Generation limit & 512 tokens / assistant turn \\
Total generation cap & 32{,}768 tokens / trajectory \\
\bottomrule
\end{tabular}
\end{table}

Agentic GRPO and Agentic ESOpt therefore use the same 50-turn interaction limit, 512-token per-turn generation limit, total trajectory generation cap, and sampling configuration during training. The final two records do not fill a four-question batch and are omitted, leaving 12 training rounds per epoch. Continuous ANLS is used as the training reward as well.

\subsubsection{Common Evaluation Configuration}
All reported methods and baselines also follow the same final evaluation protocol on DocVQA. We evaluate on the same held-out set of 100 questions with four samples per question. Each trajectory is limited to 50 interaction turns, 512 generated tokens per turn, and 32{,}768 generated tokens in total. Sampling uses temperature $1$, top-$p=1$, top-$k=40$, min-$p=0$, presence penalty $2$, and repetition penalty $1$ for every method. Agentic ESOpt, Agentic GRPO, Trace2Skill, their combined variants, and all other reported baselines therefore use identical evaluation settings. In addition, the interaction and generation budgets of Agentic ESOpt and Agentic GRPO are fully aligned during DocVQA training.

\subsubsection{Repeated-Sampling Profiles for Math and DocVQA}
\label{app:math-docvqa-pass-at-k}

Figure~\ref{fig:math-docvqa-pass-at-k} extends the four-sample results in Table~\ref{tab:math-docvqa} to repeated-sampling budgets up to $k=32$, using the same evaluation decoder for each method.

\begin{figure}[H]
\centering
\includegraphics[width=0.9\linewidth]{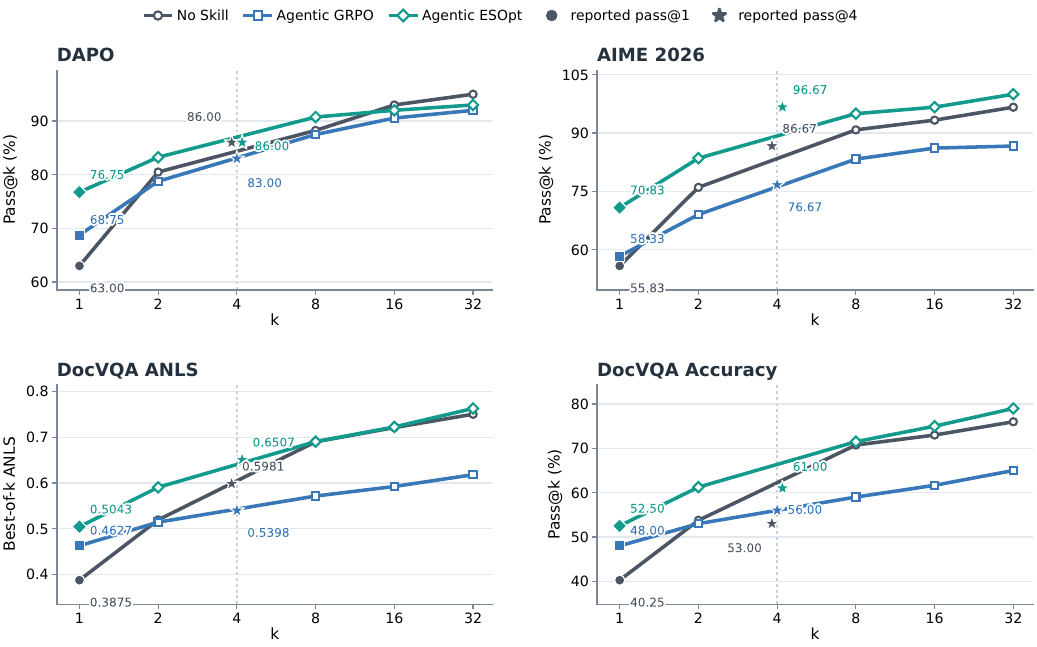}
\caption{Repeated-sampling profiles for DAPO, AIME 2026, DocVQA ANLS, and DocVQA accuracy. The $k=1$ markers use the reported Mean@4 values as estimates of single-sample performance, while stars show the originally reported Pass@4 or Max@4 values. Curves report the newly estimated best-of-$k$ results for $k\in\{2,8,16,32\}$; the newly computed $k=4$ curve points are omitted to keep the original four-sample measurements visually distinct.}\vspace{-10pt}
\label{fig:math-docvqa-pass-at-k}
\end{figure}

Across all four panels, Agentic ESOpt remains above the matched Agentic GRPO baseline as the sampling budget increases. On AIME 2026 and both DocVQA metrics, it also preserves an advantage over the base policy at larger $k$, indicating that the average-performance gains are not obtained by collapsing repeated-sampling coverage.

\subsubsection{DocVQA Training-Stage Turn Diagnostics}
\label{app:docvqa-turn-diagnostics}

\begin{figure}[H]
\centering
\includegraphics[width=0.65\linewidth]{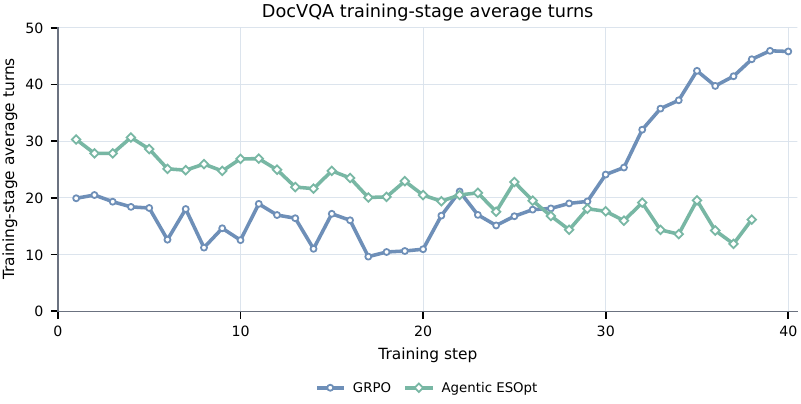}
\caption{Average realized trajectory length during DocVQA training. Agentic GRPO and Agentic ESOpt use the same rollout decoder; their training schedules are aligned to a common update axis and shown through step 40.}\vspace{-10pt}
\label{fig:docvqa-training-turns}
\end{figure}

The two methods begin with similar trajectory lengths, but their behavior diverges later in training: Agentic GRPO trajectories grow rapidly, while Agentic ESOpt remains shorter and more stable. This diagnostic mirrors the excess-turn pattern observed in Sudoku and provides an additional view of optimization under sparse trajectory-level feedback.

\subsection{WebArena-Lite}
\label{app:webarena-details}

\subsubsection{Goal-Conditioned Task Formulation}

WebArena-Lite evaluates goal-conditioned control in a partially observed browser environment. A task $g$ specifies a user goal and an initial latent web state. At time $t$, the agent receives observation $o_t$, samples an action
\[
    a_t
    \sim
    \pi_\theta(\cdot \mid \boldsymbol{o}_{\leq t}, \boldsymbol{a}_{<t}, g, c),
\]
and transitions to the next browser state. A trajectory $\boldsymbol{\tau}=(o_0,a_0,\ldots,o_T)$ is successful when it satisfies the goal predicate within the interaction budget. The benchmark objective is
\[
    J_{\mathrm{web}}(\theta;c)
    =
    \mathbb{E}_{g \sim \mathcal{G}}
    \mathbb{E}_{\boldsymbol{\tau} \sim \pi_\theta(\cdot \mid g,c)}
    \left[
        \mathbbm{1}_{\mathrm{succ}}(\boldsymbol{\tau},g)
    \right],
\]
estimated by the empirical success rate over the held-out task set. The reward is sparse: intermediate actions receive no direct task credit, and the final score depends on whether the complete trajectory reaches the target web state.

The external context $c$ represents non-parametric adaptation. In the No Skill condition, it contains the task instruction and interaction history; in the Trace2Skill condition, it additionally contains procedural knowledge distilled from previous trajectories. The No Skill pair in Table~\ref{tab:webarena} isolates the parameter update, while the skill-conditioned pair compares the complete adaptation pipelines.

For website category $\mathcal{G}_k$, the reported conditional success rate is
\[
    \widehat{J}_{\mathcal{G}_k}(\theta;c)
    =
    \frac{1}{|\mathcal{G}_k|}
    \sum_{g \in \mathcal{G}_k}
    \mathbbm{1}_{\mathrm{succ}}(\boldsymbol{\tau}_g,g).
\]
The dataset average applies the same estimator to the full set $\mathcal{G}$.

\subsubsection{Rollout Protocol and Data Split}

All frozen and Agentic ESOpt rows use the same browser protocol. The agent observes WebRL-style textual browser states, acts in the WebRL id-based action space, and is limited to 30 browser actions per task. Each response has a 2,048-token budget, the browser viewport is $1280\times720$, and the evaluator assigns binary success after termination. Sampling uses temperature $0.7$, top-$p=0.8$, top-$k=20$, min-$p=0$, presence penalty $1.5$, and repetition penalty $1$. Final results average three complete evaluations over all 165 held-out tasks.

The training split is constructed from the original 812 WebArena tasks after excluding the 165 tasks mapped to WebArena-Lite by the benchmark-provided correspondence. Site-stratified splitting of the remaining 647 tasks yields 582 training tasks and 65 validation tasks. We run each model with three seeds to distill the skill (base model, Agentic RL model, Agentic ESOpt model) and select the best one on the validation dataset as the evaluation skill. WebArena-Lite tasks never contribute parameter-update rewards or trajectory-to-skill inputs, and periodic evaluation is read-only. All six sites---Reddit, GitLab, Wikipedia, Map, Shopping, and Shopping Admin---remain enabled.

\subsubsection{Parameter and Skill Adaptation Stages}

\begin{table}[H]
\centering
\setlength{\tabcolsep}{4pt}
\caption{WebArena-Lite adaptation stages on Qwen3.5-27B. Both Agentic ESOpt conditions share the same No Skill parameter updates; the combined condition adds one post-hoc skill-distillation stage.}
\label{tab:webarena-hparams}
\resizebox{\linewidth}{!}{%
\begin{tabular}{lcccc}
\toprule
Condition & Parameter stage & Skill stage & Data per stage & Additional ES \\
\midrule
No Skill baseline & none & none & -- & none \\
Agentic ESOpt + No Skill & 70 ES generations & none & $G=8$ directions $\times$ 8 tasks/generation & -- \\
Trace2Skill baseline & none & 70 skill iterations & 8 tasks $\times$ 8 rollouts/iteration & none \\
Agentic ESOpt + Trace2Skill & same 70-generation ES run & one post-hoc distillation & all completed ES trajectories & none \\
\bottomrule
\end{tabular}
}
\end{table}

The Trace2Skill baseline starts from an empty skill and runs 70 skill-evolution iterations. Each iteration uses eight training tasks with eight sampled trajectories per task and selects at most one positive and one negative trajectory per task. Validation is performed every 10 iterations, and Table~\ref{tab:webarena} reports the final evaluated skill.

The No Skill parameter stage performs 70 full-parameter ES updates with $G=8$ on eight training tasks per generation. Each generation therefore contains $8\times8=64$ direction--task evaluations, with every perturbation scored on the same task batch. Rewards are z-score normalized, $\sigma_0=1.5\times10^{-3}$, $\sigma_t=1.5\times10^{-3}$, $\alpha=2.5\times10^{-4}$, and evaluation is performed every 10 generations.

After the parameter stage, trajectories from all completed ES generations are passed once to Trace2Skill. Final evaluation reconstructs the same No Skill update sequence and adds the distilled skill to the system prompt. This experiment implements sequential shared-trajectory composition; no additional skill-conditioned ES stage is performed.

\subsubsection{Full-Evaluation Curve}
\label{app:webarena-eval-curve}

\begin{figure}[H]
\centering
\includegraphics[width=0.5\linewidth]{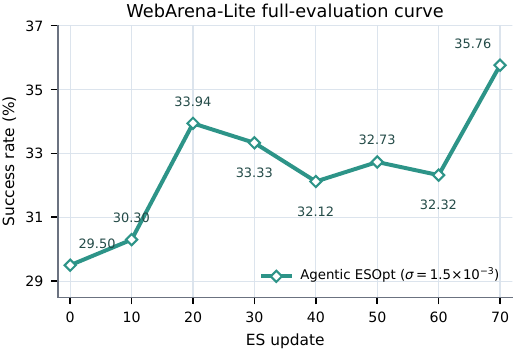}
\caption{Full WebArena-Lite evaluation success rate during the No Skill Qwen3.5-27B Agentic ESOpt run. Evaluation is performed every 10 ES updates; the curve starts from the 29.50\% base checkpoint and reaches 35.76\% after 70 updates.}
\label{fig:webarena-agentic-curve}
\end{figure}

The full-set curve is non-monotonic across intermediate checkpoints, but the final update reaches the best observed success rate. Table~\ref{tab:webarena} reports the three-run final evaluation used for the main comparison.

\subsection{Automatic Heuristic Design}
\label{app:ahd-setting}

\subsubsection{Experimental Configuration}
Experiments of AHD are done on a cluster of eight 3090 24GB GPUs. We evaluate Agentic ESOpt under two fixed outer search scaffolds: independent Sample and EoH. In each paired comparison, the proposal budget and outer search procedure are unchanged; Agentic ESOpt adds only the parameter-space update. All runs use LLaMA-3.1-8B-Instruct.

EoH maintains $N=10$ heuristics for 25 outer generations. Operators e1 and e2 each generate $N$ candidates, while m1 and m2 each generate $kN$ candidates, giving a total budget of $25N(2+2k)$. We use $k=1$ for $T=1{,}000$ and $k=3$ for $T=2{,}000$. The Sample scaffold evaluates independent i1 proposals in batches of 20.

\begin{table}[H]
\centering
\setlength{\tabcolsep}{4pt}
\caption{AHD proposal budgets and Agentic ESOpt update units. Matched baselines use the same proposal counts. In EoH, m1 and m2 each form one ES update batch per outer generation.}
\label{tab:ahd-hparams}
\resizebox{\linewidth}{!}{%
\begin{tabular}{lcccccc}
\toprule
Outer scaffold & $T$ & Outer generations & Candidates/generation & ES update batches & Directions/batch & $\sigma_0\!\rightarrow\!\sigma_T$ \\
\midrule
EoH ($k=1$) & 1{,}000 & 25 & $10(2+2k)=40$ & 50 & 10 & $10^{-3}\!\rightarrow0$ \\
EoH ($k=3$) & 2{,}000 & 25 & $10(2+2k)=80$ & 50 & 30 & $10^{-3}\!\rightarrow0$ \\
Independent Sample & 1{,}000 & 50 & 20 & 50 & 20 & $10^{-3}\!\rightarrow0$ \\
Independent Sample & 2{,}000 & 100 & 20 & 100 & 20 & $10^{-3}\!\rightarrow0$ \\
\bottomrule
\end{tabular}
}
\end{table}

Agentic ESOpt uses full-parameter, one-sided Gaussian perturbations, cosine decay $\sigma:10^{-3}\!\rightarrow0$, no warmup, $\alpha=5\times10^{-4}$, and population z-score normalization. In EoH, only mutation operators m1 and m2 trigger parameter updates, yielding two update batches per outer generation; e1 and e2 remain unchanged. In Sample, each batch of 20 proposals forms one update.

All objectives are converted internally to minimization costs. EoH uses the selected parent's cost minus the perturbed child's cost, while Sample uses the negative child score. Invalid or non-finite programs receive a finite within-batch penalty before normalization, and an all-invalid batch produces no update. Candidate generation uses temperature $1$, top-$p=0.98$, no top-$k$ cutoff, and at most 768 new tokens.

\subsubsection{Additional ACO-Style Results and Ablations}
\label{app:ahd-additional-results}

\paragraph{ACO-style results.}

\begin{table}[H]
\centering
\small
\caption{AHD ACO-style results at total evaluation budgets $T\in\{1000,2000\}$. Lower is better for all columns. The $\Delta$ rows report the ratio $\mathrm{EoH}/(\mathrm{Agentic\ ESOpt+EoH})-1$; positive ratios denote improvement. Improvements are shaded green with green deltas; regressions remain unshaded with neutral-gray deltas.}
\label{tab:ahd-aco}
\setlength{\tabcolsep}{6pt}
\resizebox{\linewidth}{!}{%
\begin{tabular}{lcccccc}
\toprule
Method & \makecell{TSP\\$N=50$} & \makecell{TSP\\$N=100$} & \makecell{CVRP\\$N=50,C=50$} & \makecell{CVRP\\$N=100,C=50$} & \makecell{BPP\\$N=500,C=150$} & \makecell{BPP\\$N=1{,}000,C=150$} \\
\midrule[0.35mm]
ACO & 5.992 & 8.948 & 11.355 & 18.778 & 208.828 & 417.938 \\
\midrule[0.35mm]
\multicolumn{7}{c}{\textit{Total Evaluations: $T=1000$}} \\
\midrule[0.35mm]
EoH & 6.030 & 8.859 & 9.328 & 15.666 & 203.162 & 405.172 \\
\improvedcell{Agentic ESOpt + EoH} &
\improvedcell{5.882} &
\improvedcell{8.391} &
\improvedcell{9.265} &
16.073 &
203.500 &
405.787 \\
{\scriptsize $\Delta$ vs EoH} &
\improvedcell{\deltagain{$\downarrow$2.53\%}} &
\improvedcell{\deltagain{$\downarrow$5.57\%}} &
\improvedcell{\deltagain{$\downarrow$0.68\%}} &
\deltaloss{$\uparrow$2.53\%} &
\deltaloss{$\uparrow$0.17\%} &
\deltaloss{$\uparrow$0.15\%} \\
\midrule[0.35mm]
\multicolumn{7}{c}{\textit{Total Evaluations: $T=2000$}} \\
\midrule[0.35mm]
EoH & 5.886 & 8.368 & 9.379 & 15.989 & 203.307 & 405.688 \\
\improvedcell{Agentic ESOpt + EoH} &
5.890 &
8.398 &
\improvedcell{9.179} &
\improvedcell{15.420} &
\improvedcell{203.052} &
\improvedcell{405.276} \\
{\scriptsize $\Delta$ vs EoH} &
\deltaloss{$\uparrow$0.08\%} &
\deltaloss{$\uparrow$0.36\%} &
\improvedcell{\deltagain{$\downarrow$2.18\%}} &
\improvedcell{\deltagain{$\downarrow$3.69\%}} &
\improvedcell{\deltagain{$\downarrow$0.13\%}} &
\improvedcell{\deltagain{$\downarrow$0.10\%}} \\
\bottomrule
\end{tabular}
}
\end{table}

At $T=1{,}000$, Agentic ESOpt + EoH improves three of the six ACO-style test sets: both TSP settings and CVRP-50. At $T=2{,}000$, it improves both CVRP and both BPP settings; the two TSP differences are small ($0.08\%$ and $0.36\%$). Together with the constructive Sample and EoH comparisons, Agentic ESOpt improves 28 of 36 matched method--budget settings across 12 test sets and six scenarios, with one tie and seven regressions.

\paragraph{Component and sampling-temperature ablations.}

\begin{table}[H]
\centering
\begin{minipage}[t]{0.57\linewidth}
\centering
\caption{Component ablation on constructive AHD at $T=1000$. TSP is minimized, and KP is maximized.}
\label{tab:ahd-component-ablation}
\setlength{\tabcolsep}{10pt}
    \small
\resizebox{\linewidth}{!}{%
\begin{tabular}{lcc}
\toprule
Method & \makecell{TSP\\$N=50$} & \makecell{KP\\$N=50,W=12.5$} \\
\midrule
EoH (baseline) & 6.545 & 19.996 \\
\improvedcell{Agentic ESOpt + EoH} & \improvedcell{6.463} & \improvedcell{20.001} \\
w/o ES (noise-only) & 6.484 & 19.996 \\
w/o cosine schedule & 6.480 & 19.997 \\
\bottomrule
\end{tabular}
}
\end{minipage}
\hfill
\begin{minipage}[t]{0.38\linewidth}
\centering
\caption{EoH sampling-temperature ablation on constructive TSP ($N=50$) at $T=1000$, with Agentic ESOpt included for comparison. Lower is better.}
\label{tab:eoh-temperature-ablation}
\setlength{\tabcolsep}{18pt}
\small
\begin{tabular}{lc}
\toprule
Setting & TSP \\
\midrule
EoH ($0.6$) & 6.51782 \\
EoH ($1.0$) & 6.545 \\
EoH ($1.5$) & 6.95927 \\
\improvedcell{Agentic ESOpt} & \improvedcell{6.463} \\
\bottomrule
\end{tabular}
\end{minipage}
\end{table}

The component ablation in \Cref{tab:ahd-component-ablation} shows that both the reward-weighted parameter update and the cosine perturbation schedule contribute to the final result: retaining noise without the ES update, or fixing the perturbation radius, weakens performance on both tasks. \Cref{tab:eoh-temperature-ablation} further shows that the gain is not reproduced by retuning the EoH sampling temperature. Temperature $0.6$ is the strongest EoH setting, yet Agentic ESOpt remains better on TSP.

\paragraph{Repeated-run significance.}
We repeat the constructive EoH comparison 20 times per method on TSP ($N=50$) and KP ($N=100,W=25$). Table~\ref{tab:ahd-significance} reports the mean and sample standard deviation across runs together with a one-sided, equal-variance $t$-test. KP values are converted back from the internally negated solver cost to the native maximization objective. The test is computed from the same 20 per-run observations summarized in the table.

\begin{table}[H]
\centering
\small
\setlength{\tabcolsep}{7pt}
\caption{Repeated-run analysis for constructive AHD. Each method is evaluated over 20 independent runs. Lower is better for TSP and higher is better for the native KP objective.}
\label{tab:ahd-significance}
\begin{tabular}{lccc}
\toprule
Task & EoH & Agentic ESOpt + EoH & $p$-value \\
\midrule
TSP ($N=50$, $\downarrow$) &
$6.5517 \pm 0.0729$ &
$6.5007 \pm 0.0868$ &
$0.0258$ \\
KP ($N=100,W=25$, $\uparrow$) &
$40.1562 \pm 0.0024$ &
$40.1578 \pm 0.0017$ &
$0.0100$ \\
\bottomrule
\end{tabular}
\end{table}

Both comparisons are significant at the $0.05$ level. The repeated-run analysis therefore supports the consistency of the AHD gains beyond a single search seed.

\paragraph{On-the-fly runtime overhead.}
We additionally measure end-to-end wall-clock time for the constructive Sample scaffold at $T=1000$.

\begin{table}[H]
\centering
\small
\setlength{\tabcolsep}{12pt}
\caption{End-to-end runtime for constructive Sample at $T=1000$. All values are in minutes.}
\label{tab:ahd-sample-runtime}
\begin{tabular}{lcc}
\toprule
Task & Sample & Agentic ESOpt + Sample \\
\midrule
Design constructive heuristics for TSP & 54.8 & 60.6 \\
Design constructive heuristics for KP  & 41.2 & 48.6 \\
Design constructive heuristics for ASP & 51.7 & 56.7 \\
\bottomrule
\end{tabular}
\end{table}

Across the three tasks, attaching Agentic ESOpt adds only 5.0--7.4 minutes, corresponding to a 9.7\%--18.0\% increase over the original Sample runtime. The parameter updates are therefore carried out nearly on the fly within the existing candidate-generation and evaluation loop, without a separate post-search training stage.

\subsection{Cross-Setting Hyperparameter Summary}
\label{app:hyperparameters}

Table~\ref{tab:es-hparam-summary} consolidates the Agentic ESOpt configurations used across all reported settings. The task-specific sections above provide the corresponding interaction limits, baseline settings, and evaluation protocols.

\begin{table}[H]
\centering
\scriptsize
\setlength{\tabcolsep}{2.5pt}
\caption{Cross-setting Agentic ESOpt configurations. ``Update batches'' counts parameter updates, and ``reward load/direction'' specifies how each perturbation is scored. All settings use full-parameter, one-sided Gaussian perturbations and population z-score normalization.}
\label{tab:es-hparam-summary}
\resizebox{\linewidth}{!}{%
\begin{tabular}{llccccc}
\toprule
Setting & Backbone & Update batches & Directions/update & Reward load/direction & $\sigma_0\!\rightarrow\!\sigma_T$ & $\alpha$ \\
\midrule
Sudoku, $H^*=5/10$ & Qwen3.5-4B & 100 & 32 & 32 puzzles & $10^{-3}\!\rightarrow2.5{\times}10^{-4}$ & $5{\times}10^{-4}$ \\
Sudoku, $H^*=15$ & Qwen3.5-4B & 100 & 32 & 32 puzzles & $7{\times}10^{-4}\!\rightarrow5{\times}10^{-4}$ & $5{\times}10^{-4}$ \\
Math & Qwen3.5-4B & 25 & 16 & 16 problems & $10^{-3}\!\rightarrow5{\times}10^{-4}$ & $5{\times}10^{-4}$ \\
DocVQA & Qwen3.5-4B & 40 & 16 & 16 examples & $10^{-3}\!\rightarrow5{\times}10^{-4}$ & $5{\times}10^{-4}$ \\
WebArena-Lite & Qwen3.5-27B & 70 & 8 & 8 web tasks & $1.5{\times}10^{-3}\!\rightarrow1.5{\times}10^{-3}$ constant & $2.5{\times}10^{-4}$ \\
AHD + EoH, $T=1{,}000$ & Llama-3.1-8B & 50 & 10 & one heuristic objective & $10^{-3}\!\rightarrow0$ & $5{\times}10^{-4}$ \\
AHD + EoH, $T=2{,}000$ & Llama-3.1-8B & 50 & 30 & one heuristic objective & $10^{-3}\!\rightarrow0$ & $5{\times}10^{-4}$ \\
AHD + Sample, $T=1{,}000/2{,}000$ & Llama-3.1-8B & 50 / 100 & 20 & one heuristic objective & $10^{-3}\!\rightarrow0$ & $5{\times}10^{-4}$ \\
\bottomrule
\end{tabular}%
}
\end{table}

For candidate $i$, the implementation applies $\theta+\sigma_t\boldsymbol{\epsilon}_i$, evaluates the complete trajectory or generated heuristic, reverts the perturbation, and reconstructs the same direction from its integer seed during the update. The implementation uses one-sided perturbations and no explicit $1/\sigma_t$ multiplier, so $\alpha$ is the effective update scale. Task-specific baseline configurations are summarized in \Cref{tab:sudoku-es-hparams,tab:math-docvqa-es-hparams,tab:math-docvqa-grpo-hparams,tab:docvqa-es-hparams,tab:docvqa-grpo-hparams,tab:webarena-hparams,tab:ahd-hparams}.

The perturbation schedule reflects the optimization objective of each setting. Train-time experiments retain a nonzero terminal radius to preserve neighborhood smoothing: Sudoku uses horizon-specific terminal radii, Math and DocVQA decay to $5\times10^{-4}$, and WebArena-Lite keeps $\sigma=1.5\times10^{-3}$. AHD decays $\sigma$ to zero because the final objective is the best task-specific heuristic found during test-time search.

Population size $G$ and the number of tasks used to score each direction are chosen jointly. Increasing either quantity can improve update stability or task coverage, while also increasing rollout cost. Across settings, binary-reward tasks use multiple tasks per direction, whereas AHD assigns one executable heuristic objective to each direction.

Population z-score normalization provides a common update interface across heterogeneous rewards. It converts binary success, continuous ANLS, and problem-dependent heuristic objectives into relative within-population scores. For AHD, invalid or non-finite programs are mapped to finite within-batch penalties before normalization, and an all-invalid batch produces no parameter update.

Hyperparameter sensitivity and adaptive schedules for $G$ and $\sigma$ are discussed in Appendix~\ref{app:limitation-futurework}.

\clearpage
\section{Prompts and Skills}
\label{app:environment-prompts}

\subsection{Sudoku}
\label{app:prompt-sudoku}

\subsubsection{Environment Prompt}

Sudoku is implemented as a multi-turn action environment. The agent does not submit a full board in one response. Instead, at each turn it observes the original puzzle, the current board, the current empty-cell list, and optional environment feedback from the previous action, then emits exactly one action in the format \texttt{set <row> <col> <value>}. Rows and columns are 1-indexed. The episode terminates when all masked cells are filled, or the action budget is exhausted, and the final reward is binary.

The following is a turn-0 prompt from the evaluation split with 15 masked cells. Empty cells are rendered as ``.''.

\begin{tcolorbox}[breakable, enhanced, colframe=promptgreen!85!black, colback=promptgreenbg, boxrule=0.65pt,
title={Prompt for Sudoku}, fonttitle=\bfseries\scriptsize,
attach boxed title to top left={xshift=5pt,yshift=-2pt},
boxed title style={colframe=promptgreen!85!black, colback=promptgreen!85!black, boxrule=0.4pt, arc=0.8mm,
left=3pt, right=3pt, top=1pt, bottom=1pt},
arc=1mm, top=8pt, bottom=4pt, left=5pt, right=5pt, boxsep=1pt]
\begin{lstlisting}[breaklines=true, breakatwhitespace=false, basicstyle=\ttfamily\fontsize{4.8pt}{5.4pt}\selectfont]
You are an agent solving Sudoku one action at a time. At each turn, fill exactly one empty cell. Rows and columns are 1-indexed. The board is split by | and horizontal lines into nine 3x3 boxes. Every row, every column, and every 3x3 box must contain digits 1 through 9 exactly once. Choose the cell only from the Current empty cells list. Keep all original givens fixed. Your entire response must be exactly one line containing one action and nothing else. Do not explain, reason aloud, wrap the answer in code fences, or output any extra words. Use exactly this format:
set <row> <col> <value>

Original puzzle, mask_count=15:
      c1 c2 c3 | c4 c5 c6 | c7 c8 c9
r1:   3 5 9 | 4 . 7 | . . 8
r2:   2 1 8 | . 9 5 | 4 . 6
r3:   4 7 6 | 2 8 1 | 3 5 9
      --------+----------+--------
r4:   6 2 1 | 8 5 3 | 9 4 7
r5:   . 3 5 | 9 . 4 | 6 2 1
r6:   . 4 7 | 6 1 2 | 8 3 5
      --------+----------+--------
r7:   7 6 . | 1 3 8 | 5 9 .
r8:   5 9 . | . 2 6 | 1 . 3
r9:   . 8 3 | 5 4 . | 7 6 2

Current board, turn=0, remaining_empty=15:
      c1 c2 c3 | c4 c5 c6 | c7 c8 c9
r1:   3 5 9 | 4 . 7 | . . 8
r2:   2 1 8 | . 9 5 | 4 . 6
r3:   4 7 6 | 2 8 1 | 3 5 9
      --------+----------+--------
r4:   6 2 1 | 8 5 3 | 9 4 7
r5:   . 3 5 | 9 . 4 | 6 2 1
r6:   . 4 7 | 6 1 2 | 8 3 5
      --------+----------+--------
r7:   7 6 . | 1 3 8 | 5 9 .
r8:   5 9 . | . 2 6 | 1 . 3
r9:   . 8 3 | 5 4 . | 7 6 2

Current empty cells: r1c5, r1c7, r1c8, r2c4, r2c8, r5c1, r5c5, r6c1, r7c3, r7c9, r8c3, r8c4, r8c8, r9c1, r9c6
\end{lstlisting}
\end{tcolorbox}

For example, a valid response to this prompt is a single line such as \texttt{set 1 5 6}. The environment then applies the action and constructs the next prompt by inserting \texttt{Last environment feedback: Filled r1c5 with 6.} before the original puzzle block, updating the current board, setting \texttt{turn=1}, decreasing \texttt{remaining\_empty} to 14, and removing \texttt{r1c5} from the current empty-cell list.

\subsection{Math Reasoning}
\label{app:prompt-math}

\subsubsection{ReAct Prompt}
Instead of a single direct answer call, the AIME setting is treated as an agentic ReAct environment. The model may call a command-line Python tool multiple times, receives each tool result as an observation, and stops when it emits the final boxed answer. The ES reward is the binary exact-match score of the final answer after the whole trajectory.

\begin{tcolorbox}[breakable, enhanced, colframe=promptteal!85!black, colback=prompttealbg, boxrule=0.65pt,
title={Prompt for AIME ReAct}, fonttitle=\bfseries\scriptsize,
attach boxed title to top left={xshift=5pt,yshift=-2pt},
boxed title style={colframe=promptteal!85!black, colback=promptteal!85!black, boxrule=0.4pt, arc=0.8mm,
left=3pt, right=3pt, top=1pt, bottom=1pt},
arc=1mm, top=8pt, bottom=4pt, left=5pt, right=5pt, boxsep=1pt]
\begin{lstlisting}[breaklines=true, breakatwhitespace=false, basicstyle=\ttfamily\scriptsize]
You are solving a math reasoning problem. You can think step-by-step and use command-line Python when useful.

Action format:
Action:
{
  "name": "bash",
  "arguments": {"command": "<shell command>"}
}
After each Action, the environment executes the command and returns an Observation. Repeat Thought/Action/Observation as needed. When finished, output the final result as \boxed{answer}.

Problem:
Patrick started walking at a constant rate along a straight road from school to the park. One hour after Patrick left, Tanya started running along the same road from school to the park. One hour after Tanya left, Jose started bicycling along the same road from school to the park. Tanya ran at a constant rate of 2 miles per hour faster than Patrick walked, Jose bicycled at a constant rate of 7 miles per hour faster than Tanya ran, and all three arrived at the park at the same time. The distance from the school to the park is m/n miles, where m and n are relatively prime positive integers. Find m+n.

Thought: Let Patrick's travel time be T and speed be p. I can solve the two equal-distance equations with Python.
Action:
{
  "name": "bash",
  "arguments": {"command": "python -c \"import sympy as sp; T,p=sp.symbols('T p'); sol=sp.solve([sp.Eq(p,2*T-2), sp.Eq(p,sp.Rational(9,2)*T-9)],[T,p]); T0,p0=sol[T],sol[p]; d=sp.factor(T0*p0); print(d, sp.numer(d)+sp.denom(d))\""}
}
Observation: 252/25 277

Thought: The distance is 252/25, so m+n=277.
Final answer: \boxed{277}
\end{lstlisting}
\end{tcolorbox}

\subsubsection{Trajectory-to-Skill Composition}
Math considers trajectories from all 25 ES generations, keeps at most one failed trajectory for each of the 400 training problems, and excludes successful traces from skill distillation. A single post-hoc Trace2Skill pass uses \llmname{GPT-5.4-nano}, and temperature $1$. We run each model with three seeds to distill the skill (base model, Agentic RL model, Agentic ESOpt model) and select the best result to report. Final evaluation reconstructs the same No Skill ES updates and injects the distilled skill only into the system context; it performs no further parameter optimization.

\subsection{DocVQA}
\label{app:prompt-docvqa}

\subsubsection{ReAct Prompt}
DocVQA is also an agentic ReAct setting. The task exposes the document image path, and the agent may use command-line or Python tools to inspect the image, crop regions, or run OCR before submitting a short answer. The evaluator scores the final submitted answer with ANLS.

\begin{tcolorbox}[breakable, enhanced, colframe=promptmint!85!black, colback=promptmintbg, boxrule=0.65pt,
title={Prompt for DocVQA ReAct}, fonttitle=\bfseries\scriptsize,
attach boxed title to top left={xshift=5pt,yshift=-2pt},
boxed title style={colframe=promptmint!85!black, colback=promptmint!85!black, boxrule=0.4pt, arc=0.8mm,
left=3pt, right=3pt, top=1pt, bottom=1pt},
arc=1mm, top=8pt, bottom=4pt, left=5pt, right=5pt, boxsep=1pt]
\begin{lstlisting}[breaklines=true, breakatwhitespace=false, basicstyle=\ttfamily\scriptsize]
Answer the document visual question. Inspect the referenced document image carefully. You can use command-line and Python tools over the image file.

[Optional skill instructions are inserted here.]

Action format:
Action:
{
  "name": "bash",
  "arguments": {"command": "<shell command>"}
}
After each Action, the environment returns an Observation. Repeat Thought/Action/Observation as needed. When finished, output only the short final answer.

Image path: data\trace2skill\docvqa\images\49153.png
Question: What is the 'actual' value per 1000, during the year 1975?

Thought: I need inspect the document chart and locate the 1975 value for the series labeled "actual".
Action:
{
  "name": "bash",
  "arguments": {"command": "python -c \"from PIL import Image; p=r'data\\trace2skill\\docvqa\\images\\49153.png'; img=Image.open(p); print(img.size)\""}
}
Observation: (1695, 2200)

Thought: I should zoom into the chart area around the 1975 tick and read the actual series label/value before answering.
Action:
{
  "name": "bash",
  "arguments": {"command": "python -c \"from PIL import Image; p=r'data\\trace2skill\\docvqa\\images\\49153.png'; img=Image.open(p); crop=img.crop((250,650,1450,1350)); crop.save('docvqa_49153_chart_crop.png'); print('saved docvqa_49153_chart_crop.png')\""}
}
Observation: saved docvqa_49153_chart_crop.png

Thought: The chart crop shows the 1975 actual value as 0.28 per 1000.
Final answer: 0.28
\end{lstlisting}
\end{tcolorbox}

\subsubsection{Trajectory-to-Skill Composition}
DocVQA selects the final 50 task occurrences and keeps at most one failed and one successful trace per occurrence; when one outcome type is unavailable, only the observed type is retained. A single post-hoc Trace2Skill pass uses \llmname{GPT-5.4-nano} and temperature $1$. We run each model with three seeds to distill the skill (base model, Agentic RL model, Agentic ESOpt model) and select the best result to report. Final evaluation reconstructs the same No Skill ES updates and injects the distilled skill only into the system context; it performs no further parameter optimization.

\subsection{WebArena-Lite}
\label{app:prompt-webarena}

\subsubsection{Environment Prompt}
WebArena-Lite uses the WebArena id-based accessibility-tree prompt. The full observation can be long, so the box shows the stable prompt structure and one real WebArena-Lite task objective.

\begin{tcolorbox}[breakable, enhanced, colframe=promptforest!85!black, colback=promptforestbg, boxrule=0.65pt,
title={Prompt for WebArena-Lite}, fonttitle=\bfseries\scriptsize,
attach boxed title to top left={xshift=5pt,yshift=-2pt},
boxed title style={colframe=promptforest!85!black, colback=promptforest!85!black, boxrule=0.4pt, arc=0.8mm,
left=3pt, right=3pt, top=1pt, bottom=1pt},
arc=1mm, top=8pt, bottom=4pt, left=5pt, right=5pt, boxsep=1pt]
\begin{lstlisting}[breaklines=true, breakatwhitespace=false, basicstyle=\ttfamily\scriptsize]
You are an autonomous intelligent agent tasked with navigating a web browser. You will be given web-based tasks. The information includes the user's objective, the current accessibility tree, the current page address, open tabs, and the previous action.

Available actions include:
click [id]
type [id] [content] [press_enter_after=0|1]
hover [id]
press [key_comb]
scroll [down|up]
goto [url]
go_back
go_forward
stop [answer]

OBSERVATION:
[accessibility-tree nodes with element IDs]
CURRENT PAGE: [administration interface]
OBJECTIVE: What is the top-1 best-selling product in 2022
PREVIOUS ACTION: None

The response should reason briefly and end with:
In summary, the next action I will perform is ```<action>```
\end{lstlisting}
\end{tcolorbox}

\subsubsection{Trace2Skill Policy Context}
\label{app:webarena-trace2skill-skill}
The following box presents the WebArena skill used by the Trace2Skill-conditioned agents.

\begin{tcolorbox}[breakable, enhanced, colframe=skillgreen!85!black, colback=skillgreenbg, boxrule=0.65pt,
title={Trace2Skill WebArena Skill}, fonttitle=\bfseries\scriptsize,
attach boxed title to top left={xshift=5pt,yshift=-2pt},
boxed title style={colframe=skillgreen!85!black, colback=skillgreen!85!black, boxrule=0.4pt, arc=0.8mm,
left=3pt, right=3pt, top=1pt, bottom=1pt},
arc=1mm, top=8pt, bottom=4pt, left=5pt, right=5pt, boxsep=1pt]
\begin{lstlisting}[breaklines=true, breakatwhitespace=false, basicstyle=\ttfamily\scriptsize, literate={’}{{'}}1]
# WebArena Skill

## WebArena workflow
- Use only visible page evidence and current WebRL ids; re-resolve elements after any navigation, refresh, rerender, or layout change.
- Start from the most direct visible path. Prefer search, filters, menus, or built-in navigation before long scrolling or guessing controls.
- Re-read the current visible page after every click, selection, text entry, scroll, or navigation step. If the page did not visibly change, do not repeat the same action; choose a different visible control or route.
- Before typing, verify the field label, current value, and input type. Clear or select all existing text when overwriting.
- For dropdowns, pickers, autocomplete fields, and multi-selects, inspect the current control state first, choose the visible option that matches the task, and confirm the selected value is shown afterward.
- Treat search results, snippets, counts, and summary links as discovery aids only. Do not conclude completion until the requested state is visibly shown on the relevant page or detail view.
- If search or filter results are empty, ambiguous, or stuck, adjust the query, broaden or narrow the filter, or switch to a different grounded view instead of repeating near-identical searches.
- Use visible filters, chips, row counts, and result scopes as the source of truth for the current view.

## Lists, pagination, and completeness
- For tasks that ask for all matching items, exhaust pagination, scrolling, and expansion until no more entries remain and the requested cardinality is confirmed.
- Use page controls and visible range indicators to move stepwise through the relevant span; do not rely on a single viewport or old row positions.
- Keep a running tally or deduplicated set while scanning across pages or scroll regions when completeness matters.
- For count, total, or range questions, verify every contributing visible row or bucket directly and compute from the source values.
- For most, least, highest, or lowest questions, confirm the global extreme across the full relevant result set or a trustworthy aggregate before selecting the target.

## Detail pages, reviews, and extraction
- Open the item detail page before answering any status, review, comment, or explanation question; confirm the title and key state from the header, badge, or metadata rather than the list row.
- For review or thread-based questions, use explicit review, comment, or reply text from the dedicated discussion view; do not infer from ratings, snippets, or summaries.
- Treat clipped snippets, ellipses, and partial text as incomplete evidence; expand or open the page that shows the full content.
- When asked to extract text, copy the source text verbatim and keep only the exact requested span.
- If a likely match appears, verify its visible details on the detail page before answering or editing.

## Forms, edits, and persistence
- Locate an explicit writable composer, dialog, or editor before typing; do not type into read-only previews or summaries.
- Fill every required field explicitly and use the page’s own save, apply, update, or submit control.
- After saving or submitting, verify the persisted record or changed state is visible on the destination page; do not rely on a generic toast alone.
- If validation, required-field, login, modal, or blank/stale-state interruptions appear, resolve the issue or back out to the last stable page and continue with a different path.
- When editing a prefilled field, clear or select all first so the new value replaces the old one cleanly.
- For validated inputs, enter a format the field accepts and confirm success through a cleared error, visible confirmation, or reread saved value.
- For create-then-configure tasks, finish and confirm the created resource before moving to follow-up settings such as members, permissions, or access controls.
- Prefer reusing an existing item when it satisfies the task; avoid creating duplicates unless creation is explicitly required.

## Reports, maps, and directions
- For report or analytics tasks, use the site’s dedicated report or analytics page first; set the requested filters before generating the report.
- Match the report subview or page type to the question, and read the rendered results directly rather than the filter form or summary widgets.
- For place or location tasks, use the site search or map view to locate the anchor place first, then ground the answer in visible map context or a detailed result panel.
- Use directions mode for travel-time or reachability tasks, keep it separate from general search, and fill both endpoints with concrete locations.
- Verify the route summary, travel time, or distance from the route output, not from surrounding map UI.
- If a destination search fails or resolves incorrectly, refine the query or switch to a broader grounded place label instead of repeating the same guess.

## Dead ends and login barriers
- If navigation lands on a login, cookie, account, or other blocking screen, treat it as a dead end: backtrack once to the last useful page and continue with a different route.
- Check for an existing signed-in session, account menu, or visible credential source before attempting login.
- Do not invent or brute-force credentials; use only credentials provided by the task or visible page evidence.
- If a login attempt fails, read the error message and switch strategy instead of repeating similar guesses.

## Completion rule
- Once the requested state, output, or persisted change is visibly confirmed, stop immediately and base the final answer only on page-visible evidence.
- Do not add extra verification after the completion condition is already satisfied.
\end{lstlisting}
\end{tcolorbox}

\subsubsection{Agentic ESOpt + Trace2Skill Policy Context}
\label{app:webarena-es-trace2skill-skill}
The following box presents the WebArena skill distilled from Agentic ESOpt trajectories and used for the combined result in \Cref{tab:webarena}. It is reported separately from the Trace2Skill-only skill above because the two conditions use different trajectory sources.

\begin{tcolorbox}[breakable, enhanced, colframe=promptteal!85!black, colback=prompttealbg, boxrule=0.65pt,
title={Agentic ESOpt + Trace2Skill WebArena Skill}, fonttitle=\bfseries\scriptsize,
attach boxed title to top left={xshift=5pt,yshift=-2pt},
boxed title style={colframe=promptteal!85!black, colback=promptteal!85!black, boxrule=0.4pt, arc=0.8mm,
left=3pt, right=3pt, top=1pt, bottom=1pt},
arc=1mm, top=8pt, bottom=4pt, left=5pt, right=5pt, boxsep=1pt]
\begin{lstlisting}[breaklines=true, breakatwhitespace=false, basicstyle=\ttfamily\scriptsize, literate={’}{{'}}1]
# WebArena Skill

## Core workflow
- Use the most specific visible page or view that can directly show the requested signal; prefer list, detail, report, or management pages over inference.
- Use only the current visible page and DOM evidence after each navigation or state change; never reuse stale element ids.
- Before acting, re-scan the visible controls and match the requested verb to the exact UI control. Do not substitute a nearby action for the one the task asks for.
- Clear stale or conflicting filters before applying a new search or constraint, and confirm the active filters and sort order before extracting anything.
- For search-first tasks, use the page’s actual search control or Enter, then confirm the page is truly showing search results before selecting anything.
- After each click, search, filter, or sort, verify that the page visibly changed. If the same action leaves the page unchanged or produces an error, stop repeating it and switch to a different control or path.
- For list-based tasks, do not stop at a preview row. Inspect enough pages or candidates to cover the full scope, and use pagination or page-jump controls when available.
- For plural or broad tasks, verify each requested item explicitly against the task criterion before exiting.
- Open the exact matching detail view before using a record’s date, time, status, identifier, or other exact field as the answer.
- For form, edit, or dialog workflows, type only into writable controls, fill all required fields, confirm dropdown selections visibly changed, and use the explicit save, submit, or confirm control.
- After any save, submit, or other critical click, look for explicit visible success evidence and inspect a stable destination or reopened view once to confirm the persisted state matches the target.
- Do not trust transient banners, toasts, or typed text without visible page evidence that the change persisted.
- For discussion, comment, reply, or notification tasks, use the visible composer/editor and the explicit submit control from the record-level path.
- If an unexpected but related page opens, use back navigation or breadcrumbs to return to the source list or item and continue from there.
- If a click, submit, or navigation action has no visible effect, do not repeat it blindly; recover a visible state, locate the correct control, and try a different path.
- If the page becomes blank, stale, or unexpected, stop using old element ids and recover a visible state first by going back, reloading, or navigating to a known page.
- For toggle-style tasks, click the actual state-changing control, not a count, auxiliary link, or profile-like detour.
- For saved-item or wishlist-style tasks, confirm the item appears in the persistent saved view, not just in a transient confirmation message.
- Track every requested target explicitly in multi-item tasks and verify each one in the final visible state before exiting.
- Stop immediately once the requested page state, saved change, or answer is visibly confirmed; do not add extra exploration or verification after completion.
\end{lstlisting}
\end{tcolorbox}

\subsection{Automatic Heuristic Design}
\label{app:prompt-ahd}

\subsubsection{TSP-Construct Prompt}
For AHD, the agent writes a heuristic function instead of a direct task solution. The following box shows the TSP constructive setting and representative EoH operator prompts. The concrete parent algorithms and code blocks are inserted at the placeholders for crossover or mutation operators.

\begin{tcolorbox}[breakable, enhanced, colframe=promptolive!85!black, colback=promptolivebg, boxrule=0.65pt,
title={Prompt for AHD TSP-Construct}, fonttitle=\bfseries\scriptsize,
attach boxed title to top left={xshift=5pt,yshift=-2pt},
boxed title style={colframe=promptolive!85!black, colback=promptolive!85!black, boxrule=0.4pt, arc=0.8mm,
left=3pt, right=3pt, top=1pt, bottom=1pt},
arc=1mm, top=8pt, bottom=4pt, left=5pt, right=5pt, boxsep=1pt]
\begin{lstlisting}[breaklines=true, breakatwhitespace=false, basicstyle=\ttfamily\scriptsize]
Task:
Given a set of nodes with their coordinates, find the shortest route that visits each node once and returns to the starting node. The task can be solved step-by-step by starting from the current node and iteratively choosing the next node. Help me design a novel algorithm that is different from algorithms in literature to select the next node in each step.

Function:
def select_next_node(current_node, destination_node, unvisited_nodes, distance_matrix):
    return next_node

Input/output information:
'current_node', 'destination_node', 'next_node', and 'unvisited_nodes' are node IDs. 'distance_matrix' is the distance matrix of nodes. All are Numpy arrays.

EoH operator prompts:
[i1] First, describe your new algorithm and main steps in one sentence. The description must be inside a brace. Next, implement it in Python as select_next_node(...). Do not give additional explanations.
[e1] Given several existing algorithms and their code, create a new algorithm that has a totally different form from the given ones.
[e2] Given several existing algorithms and their code, identify their common backbone idea, then create a new algorithm with a different form that is motivated by that backbone.
[m1] Given one existing algorithm and its code, create a modified version with a different form.
[m2] Given one existing algorithm and its code, identify the main algorithm parameters and create a new algorithm with different parameter settings of the score function.
\end{lstlisting}
\end{tcolorbox}

\section{Licensing and External Components}
\label{app:license}

\subsection{Licensing Scope}
\label{app:license-scope}

The code introduced by this study is released under the MIT License. External code, benchmark assets, datasets, model checkpoints, and API services remain governed by their respective upstream licenses or terms. The table below is intentionally limited to repositories whose code or task data are directly invoked, vendored, or copied by the released workflow. Repositories consulted only as references, and baseline repositories whose implementations were not executed, are excluded.

The experiments additionally use Llama and Qwen checkpoints, API-hosted models, DAPO/AIME Math data, and DocVQA data under the terms of their respective authors and providers. Nothing in this release alters or supersedes those terms.

\subsection{Upstream Repositories Used by the Released Workflow}
\label{app:upstream-sources}

\begin{table}[H]
\centering
\setlength{\tabcolsep}{1mm}
\renewcommand\arraystretch{1.2}
\caption{A summary of licenses and external resources used in this work.}
\resizebox{\textwidth}{!}{
\begin{tabular}{llll}
\bottomrule[0.8mm]
Resources & Type & License & URL \\
\midrule[0.3mm]

Qwen3.5-4B & Base LLM & Apache-2.0 License &
\url{https://huggingface.co/Qwen/Qwen3.5-4B} \\

Qwen3.5-9B & Base LLM & Apache-2.0 License &
\url{https://huggingface.co/Qwen/Qwen3.5-9B} \\

Qwen3.5-27B & Base LLM & Apache-2.0 License &
\url{https://huggingface.co/Qwen/Qwen3.5-27B} \\

LLaMA-3.1-8B-Instruct & Base LLM & \small{Llama 3.1 Community License} &
\scriptsize{\url{https://huggingface.co/meta-llama/Llama-3.1-8B-Instruct}} \\

\midrule[0.3mm]

verl & RL framework & Apache-2.0 License &
\url{https://github.com/volcengine/verl} \\

EoH & AHD framework & MIT License &
\url{https://github.com/FeiLiu36/EoH} \\

Trace2Skill & framework & Apache-2.0 License &
\url{https://github.com/Qwen-Applications/Trace2Skill} \\

VisualAgentBench & Web-agent runtime & Apache-2.0 License &
\url{https://github.com/THUDM/VisualAgentBench} \\

VisualWebArena & Web-agent environment & MIT License &
\url{https://github.com/web-arena-x/visualwebarena} \\

\midrule[0.3mm]

AHD Problems & Dataset / Benchmark & MIT License &
\url{https://github.com/zz1358m/MCTS-AHD-master} \\

DAPO-Math-17k & Dataset & Apache-2.0 License &
\scriptsize{\url{https://huggingface.co/datasets/BytedTsinghua-SIA/DAPO-Math-17k}} \\

AIME 2026 & Dataset & Apache-2.0 License &
\url{https://huggingface.co/datasets/math-ai/aime26} \\

DocVQA & Dataset & Challenge Terms &
\url{https://site.docvqa.org/datasets} \\

WebArena(-Lite) & Dataset / Benchmark & Research use only &
\url{https://github.com/liushiliushi/JitRL} \\

\toprule[0.8mm]
\end{tabular}}
\label{tab:license}
\end{table}

\end{document}